\documentclass[runningheads,a4paper]{llncs}

\usepackage{geometry}
\usepackage{pgfplots}
\usepackage[all]{nowidow}
\usepackage[utf8]{inputenc}
\usepackage{tikz}
\usetikzlibrary{er,positioning, bayesnet}
\usepackage{multicol}
\usepackage{algpseudocode,algorithm,algorithmicx}

\usepackage{microtype}
\usepackage{graphicx}
\usepackage{subfigure}
\usepackage{booktabs}

\newtheorem{assumption}{Assumption}{\bfseries}{\itshape}

\usepackage[textsize=tiny]{todonotes}
\usepackage{dsfont}
\usepackage[T1]{fontenc}

\usepackage{amsmath}
\usepackage{amssymb}
\usepackage{mathtools}

\usepackage{natbib}

\usepackage{amsfonts}
\usepackage{graphicx}
\usepackage{xr-hyper}
\usepackage{dsfont}
\usepackage{empheq}
\usepackage{enumitem}
\usepackage{bm}
\usepackage{bbm}
\usepackage{hhline}
\usepackage{stmaryrd}
\usepackage{tikz}      % Diagrams
\usetikzlibrary{calc, shapes, backgrounds}

\usepackage[utf8]{inputenc} % allow utf-8 input
\usepackage[T1]{fontenc}    % use 8-bit T1 fonts
\usepackage{url}            % simple URL typesetting
\usepackage{booktabs}       % professional-quality tables
\usepackage{amsfonts}       % blackboard math symbols
\usepackage{nicefrac}       % compact symbols for 1/2, etc.
\usepackage{epsfig}
\usepackage{color}
\usepackage{titletoc, setspace,
comment, relsize, float}
\usepackage{caption}
\usepackage{multirow}
\def\bs{\mathbf{s}}
\def\bc{\mathbf{c}}
\def\bX{\mathbf{X}}
\def\bx{\mathbf{x}}
\def\bZ{\mathbf{Z}}
\def\bY{\mathbf{Y}}
\def\bI{\mathbf{I}}

\DeclareMathOperator*{\argmin}{arg\,min}

\begin{document}

\title{A Statistical Learning View of Simple Kriging}


\author{Emilia Siviero\inst{1} \and
Emilie Chautru\inst{2} \and
Stephan Clémençon\inst{1}
}
\authorrunning{Emilia Siviero}

\institute{LTCI, Télécom Paris, Institut Polytechnique de Paris, France\\
\email{emilia.siviero@telecom-paris.fr}\\ \and
Centre de Geosciences, Mines ParisTech, France
}
\maketitle              % typeset the header of the contribution
\thispagestyle{plain}\pagestyle{plain}


\begin{abstract}
    In the Big Data era, with the ubiquity of geolocation sensors in particular, massive datasets exhibiting a possibly complex spatial dependence structure are becoming increasingly available. In this context, the standard probabilistic theory of statistical learning does not apply directly and guarantees of the generalization capacity of predictive rules learned from such data are left to establish. We analyze here the \textit{simple Kriging} task, the flagship problem in Geostatistics, from a statistical learning perspective, \textit{i.e.} by carrying out a nonparametric finite-sample predictive analysis. Given $d\geq 1$ values taken by a realization of a square integrable random field $X=\{X_s\}_{s\in S}$, $S\subset \mathbb{R}^2$, with unknown covariance structure, at sites $s_1,\; \ldots,\; s_d$ in $S$, the goal is to predict the unknown values it takes at any other location $s\in S$ with minimum quadratic risk. The prediction rule being derived from a training spatial dataset: a single realization $X'$ of $X$, independent from those to be predicted,  observed at $n\geq 1$ locations $\sigma_1,\; \ldots,\; \sigma_n$ in $S$. Despite the connection of this minimization problem with kernel ridge regression, establishing the generalization capacity of empirical risk minimizers is far from straightforward, due to the non independent and identically distributed nature of the training data $X'_{\sigma_1},\; \ldots,\; X'_{\sigma_n}$ involved in the learning procedure. In this article, non-asymptotic bounds of order $O_{\mathbb{P}}(1/\sqrt{n})$ are proved for the excess risk of a \textit{plug-in} predictive rule mimicking the true minimizer in the case of isotropic stationary Gaussian processes, observed at locations forming a regular grid in the learning stage. These theoretical results, as well as the role played by the technical conditions required to establish them, are illustrated by various numerical experiments, on simulated data and on real-world datasets, and hopefully pave the way for further developments in statistical learning based on spatial data.
\end{abstract}

% \keywords{Geostatistics, Spatial Analysis, Kriging, Nonparametric Covariance Estimation, Prediction, Random Fields}

% \maketitle

\section{Introduction}\label{sec:intro}
In recent years, a variety of statistical learning techniques --- including boosting methods, support vector machines, neural networks among others --- have been successfully developed for performing various tasks such as classification, regression or clustering. Machine learning techniques are now supported by a very sound probabilistic theory, see \textit{e.g.} \citep{probabilistictheory, boucheron2013concentration}, guaranteeing the generalization capacity of empirically learned predictive rules under mild assumptions. However, the validity framework of statistical learning remains mainly confined to the case of independent and identically distributed (i.i.d.) training data.
At the same time, spectacular progress has been made in the collection, management and warehousing of massive datasets for scientific, engineering, medical or commercial purposes, relying on modern technologies, such as satellite imagery or geophysical tomography. These data tend to exhibit complex dependence structures, resulting in the violation of the i.i.d. assumption, under which the generalization ability of predictive rules --- constructed from training examples by means of statistical learning techniques --- is established in general.
Whereas the case of time series, which can rely on concentration results for ergodic processes, is receiving increasing attention (see \textit{e.g.} \citep{Steinwart09, ChristStein2009, KM14, Hanneke2017, CCB19}), that of spatial data is in contrast poorly documented  in the statistical learning literature. The analysis of such structured data finds applications in various fields, ranging from mining exploration to agronomy through epidemiology.
Generally viewed as a realization of a second-order random field $X=(X_s)_{s\in S}$ taking its values in $\mathbb{R}$ and indexed by a spatial set $S\subset \mathbb{R}^2$, they are traditionally analyzed in Geostatistics by stipulating a rigid and parsimonious parametric form for the underlying law, adapted to the physical phenomenon of interest. Once the parameters are assessed by means of standard $M$-estimation techniques such as MLE, the estimated model is used for various (predictive) tasks, see \textit{e.g.} \citep{gaetan2009spatial}.

\par The massive character of spatial datasets now available suggests resorting to more flexible, nonparametric, approaches to analyze spatial observations. However, classic tools are simply not designed to cope with data that exhibit such a strong dependence structure. New (non-asymptotic) results must be developed, in order to establish generalization guarantees. This article aims at contributing to the design and the study of statistical learning methods applied to spatial data, by investigating \textit{Kriging}, the flagship problem in Geostatistics, introduced by \citep{krige1951statistical} and later in the seminal work of \citep{matheron62}. In the standard Kriging setup, the spatial process $X$ is observed at $d\geq 1$ sites $s_1,\; \ldots,\; s_d$ in the domain $S$. Based on a (generally non i.i.d.) training dataset, a single realization $X'$ of $X$ observed at $n\geq 1$ locations $\sigma_1,\; \ldots,\; \sigma_n$ in $S$, the goal is to build a map $f:S\times \mathbb{R}^d\rightarrow \mathbb{R}$ in order to predict $X$ at all unobserved sites $s\in S$ with minimum mean squared error (MSE). In \textit{simple Kriging}, the mean of $X$ is supposed to be known and the goal is to search for a predictive map $f(s)=f(s,\; (X_{s_1},\; \ldots,\; X_{s_d}) )$ that is linear in $\mathbf{X}(\mathbf{s}_d):=(X_{s_1},\; \ldots,\; X_{s_d})$. In this article, we start with recalling that the optimal predictor of this type (which can be derived by ordinary least squares) has the same form as a kernel ridge regressor (KRR), once the Gram matrix is replaced with the true covariance matrix of the random vector $\mathbf{X}(\mathbf{s}_d)$. From a nonparametric statistical perspective, a predictive map can be built based on the spatial observations $X'_{\sigma_1},\; \ldots,\; X'_{\sigma_n}$ by means of a plug-in strategy, that can also be viewed as empirical risk minimization (ERM),  under the hypothesis of stationarity (combined with appropriate conditions on the decorrelation rate). Then, the covariance function can be estimated in a frequentist manner under additional mild smoothness assumptions and an empirical version of the optimal linear predictor is obtained by replacing the (unknown) covariance with the estimator in the KRR type formula.
Assuming classically that the observed sites $\sigma_1,\; \ldots,\; \sigma_n$ form a (regular) grid of $S$, denser and denser as $n$ grows (\textit{i.e.} placing ourselves in the \textit{in-fill} asymptotic setting), and that $X$ is a Gaussian spatial process, we establish non-asymptotic bounds for the excess risk of the predictive function thus constructed. These generalization bounds rely in particular on recent concentration results for sums of Gamma random variables, see \textit{e.g.} \citep{bercu2015concentration, Wang_Ma}. To the best of our knowledge, this is the only theoretical analysis of this nature documented in the statistical and machine learning literature. In order to illustrate the generalization capacity of the nonparametric predictive approach analyzed in this article and the impact of the conditions stipulated to guarantee it, a number of numerical experiments are carried out.

\par The article is structured as follows. In Section \ref{sec:background}, the main notations are introduced and the basic concepts pertaining to the theory of Kriging in Geostatistics --- involved in the subsequent analysis --- are briefly recalled, together with some key results related to kernel ridge regression. The main results of the paper are stated in Section \ref{sec:main}, where the connection between KRR and simple Kriging is preliminarily highlighted and excess risk bounds for the empirical version of the simple Kriging rule are established under appropriate assumptions, discussed at length. Numerical results, on both simulated and real data, are displayed in Section \ref{sec:num} for illustration purposes, while some concluding remarks, concerning the possible extension of the present study, are collected in Section \ref{sec:concl}. The proofs of the main results are displayed in the Appendix, together with additional comments and additional experimental results.

\section{Theoretical Background}\label{sec:background}

Here and throughout, the indicator function of any event $\mathcal{E}$ is denoted by $\mathbb{I}\{\mathcal{E}\}$, the Dirac mass at any point $x$ by $\delta_x$, the transpose of any matrix $M$ by $M^{\top}$, the trace of any square matrix $A$ by $Tr(A)$. Vectors of finite dimension $d \in \mathbb{N}^*$ are classically viewed as column vectors, the Euclidean norm in $\mathbb{R}^d$ is denoted by $\|.\|$ and the corresponding inner product by $\langle.,.\rangle$.
%The Hilbert-Schmidt norm of any rectangular matrix $A$ is denoted by $\vert\vert A\vert\vert_{HS}=\sqrt{Tr(A\; ^tA)}$. 
By $Cov(Z_1,\; Z_2)$ is meant the covariance of any pair $(Z_1, Z_2)$ of square integrable real-valued random variables defined on a same probability space, by $Var(Z)$ the covariance matrix of any square integrable random vector $Z$ and by $\vert E \vert$ the cardinality of any finite set. The operator norm of any $d\times d$ matrix $A$ is denoted by $\vert\vert\vert A\vert\vert\vert=\sup\{\vert\vert Av\vert\vert:\; v\in\mathbb{R}^d,\;  \vert\vert v\vert\vert=1\}$.

\subsection{Statistical Kriging}\label{subsec:kriging}
Let $S\subset \mathbb{R}^2$ be a (Borel measurable) spatial set and $X$ be a second-order random field on $S$ with $\mathbb{R}$ as state space, \textit{i.e.} a collection $X=\{X_s:\; s\in S\}$ of real valued square integrable random variables (r.v.) defined on the same probability space, $(\Omega, \mathcal{F}, \mathbb{P})$ say, indexed by $s\in S$. We denote by $\mu:s\in S\mapsto \mathbb{E}[X_s]$ its mean and by $C:(s,t)\in S^2\mapsto C(s,t)=Cov(X_s,\; X_t)$ its covariance functions. As formulated in the seminal contribution \citep{matheron62}, the problem of ordinary \textit{Kriging} can be described as follows. Based on the observation of the spatial process at a finite number of points $s_1,\; \ldots,\; s_d$ in the spatial set $S$, the goal pursued is to build a predictor $\widehat{X}_s$ of $X_s$ at a given unobserved site $s\in S$. The accuracy of the prediction can be measured by the mean squared error
\begin{equation}\label{eq:MSE}
L\left(s,\widehat{X}_s\right)=\mathbb{E}_X\left[\left(\widehat{X}_s-X_s\right)^2\right].
\end{equation}
 
 \noindent{\bf Simple Kriging.} In Kriging in its simplest form, the mean $\mu(\cdot)$ is supposed to be known and, rather than recentering it, it is assumed that the random field $X$ is centered. 
\begin{assumption}\label{hyp:simple}
The random field $X$ is centered: $\mu\equiv 0$.
\end{assumption} 
 One also assumes that the predictor $\widehat{X}_s$ is of the form of a \textit{linear} combination of the $X_{s_i}$'s
 \begin{equation}\label{eq:pred_linear}
 \widehat{X}_{s,\Lambda_d(s)}=\lambda_{1}(s)X_{s_1}+\ldots+\lambda_{d}(s)X_{s_d},
 \end{equation}
  where $\Lambda_d(s)=(\lambda_1(s),\; \ldots,\; \lambda_d(s))\in \mathbb{R}^d$. The optimal predictor of this form regarding the expected prediction error can be deduced from a basic variance computation, it is described below.
  \begin{lemma}\label{lem:opt_kriging}
  For $d\geq 1$, let $\bs_d=(s_1,\; \ldots,\; s_d)$, $\bX(\bs_d)=(X_{s_1},\; \ldots,\; X_{s_d})$, $\Sigma(\bs_d)=Var(\bX(\bs_d))$ and define $\bc_d(s)=(Cov(X_s,X_{s_1}),\; \ldots,\; Cov(X_s, X_{s_d}))$. Suppose that the matrix $\Sigma(\bs_d)$ is positive definite. Then, the solution of the minimization problem
    \begin{equation*}
  \min_{\Lambda_d(s) \in \mathbb{R}^d}L\left(s, \widehat{X}_{s, \Lambda_d(s)}\right)
  \end{equation*}
  is unique and given by
  \begin{equation}\label{eq:theo_krig_def}
    \Lambda_d^*(s)=\; \Sigma(\bs_d)^{-1}\bc_d(s).
  \end{equation}
  In addition, the minimum is equal to
  \begin{equation*}
 L\left(s, \widehat{X}_{s, \Lambda_d^*(s)}\right)=Var(X_s)-\; \bc_d(s)^{\top}\Sigma(\bs_d)^{-1}\bc_d(s).
  \end{equation*}
  \end{lemma}
  
  \begin{proof}
  The optimality derives from the basic properties of orthogonal projection in the $L_2$ space and the closed analytical form for the minimizer of
\begin{multline}
    L\left(s, \widehat{X}_{s, \Lambda_d(s)}\right) = Var\left( \Lambda_d(s)^\top \bX(\bs_d) - X_s\right)\\ = Var(X_s) + \Lambda_d(s) ^\top \Sigma(\bs_d) \Lambda_d(s) - 2 \bc_d(s)^\top \Lambda_d(s),
    \label{eq:lem1_L}
\end{multline}
is obtained by solving a linear system: $\Lambda_d^*(s)=\; \Sigma(\bs_d)^{-1}\bc_d(s)$. The minimal mean squared error is then:
\begin{multline*}
L\left(s, \widehat{X}_{s, \Lambda_d^*(s)}\right)  = Var(X_s) + \bc_d(s)^\top\Sigma(\bs_d)^{-1} \Sigma(\bs_d) \left(\Sigma(\bs_d)^{-1}\right)^\top \bc_d(s)\\  \qquad - 2 \bc_d(s)^\top\Sigma(\bs_d)^{-1} \bc_d(s) = Var(X_s) - \bc_d(s)^\top\Sigma(\bs_d)^{-1}\bc_d(s).
\end{multline*}
\end{proof}
  
\begin{remark}{\sc (Alternative to ERM)} We point out that, in order to avoid the restrictions stemmimg from the sub-Gaussianity assumptions, an alternative to ERM in regression, relying on a tournament method combined with the median-of-means (MoM) estimation procedure, has recently received much attention in the statistical learning literature, see \citep{lugosi2016risk} for further details. Extension of the MoM approach to Kriging will be the subject of a future work.
\end{remark}
  
  The proof above is well-known folklore in Statistics: it is exactly the same as that of the result giving the optimal solution  in linear regression.  The major difference between the two frameworks is of statistical nature. If one tries to predict $X_s$ by a linear combination of the components of the observed random vector $\bX(\bs_d)$ in simple Kriging, statistical fitting does not rely on the observation of $n\geq 1$ independent copies of the pair output-input $(X_s, \bX(\bs_d))$ but on the observation of a single realization $X'$ of the random field $X$ at certain sites $\sigma_1,\; \ldots,\; \sigma_n$ solely and some appropriate structural (possibly parametric) assumptions regarding the second-order structure of the random field $X$, see \textit{e.g.} \citep{Chiles.etal1999}. Before analyzing simple Kriging from a statistical learning perspective, a few remarks are in order.
  
  \begin{remark} {\sc (Gaussian random fields)} Notice that, in the case where the random field $X$ is Gaussian, we have:
  $  \widehat{X}_{s,\Lambda_d^*(s)}=\mathbb{E}[X_s \mid \bX(\bs_d)]$.
  Hence $ \widehat{X}_{s,\Lambda_d^*(s)}$ is the minimizer of the quadratic error $L(s,\widehat{X}_s)$ over the set of \textit{all} predictors $\widehat{X}_s$.
  \end{remark}
  \begin{remark}\label{rem:exact_interpolator} {\sc (Exact interpolation)} We point out that, in contrast to the least squares solution in a linear regression problem in general, the solution of the simple Kriging problem is an exact interpolator, in the sense that $\widehat{X}_{s_i,\Lambda_d^*(s_i)}=X_{s_i}$ for all $i\in\{1,\; \ldots,\; d\}$, see \textit{e.g.} \citep[p. 129]{Cressie_1993}.
  \end{remark}
  
  \begin{remark}{\sc (Alternative framework)}\label{rem:alt_framework}
  The spatial prediction problem has been investigated in a different statistical framework, much more restrictive regarding practical applications, assuming the observation of $N\geq 1$ independent copies of $\bX(\bs_d)$, in \citep{8645258}, where a non-asymptotic analysis is carried out as $N$ increases.
  We also point out that, instead of the classic \textit{in-fill} setting considered in subsection \ref{subsec:nonpar_cov_est} (stipulating that the grid $\sigma_1,\; \ldots,\; \sigma_n$ formed by the observed sites in $S$ in the learning/estimation stage is denser and denser, while $S$ is fixed), the \textit{out-fill} framework can be considered alternatively (prediction accuracy is then analyzed as the spatial domain $S$ becomes wider and wider) or combined with the \textit{in-fill} model in a hybrid fashion, see \textit{e.g.} \citep{Hall} and the references mentioned in the Appendix.
  \end{remark}

 \begin{remark}{\sc (Unknown constant mean)}
 In this paper, focus is on \textit{simple Kriging} (\textit{cf} Assumption \ref{hyp:simple}, see Chapter 3 in \cite{Cressie_1993}). However, when the mean is unknown but supposedly constant equal to $\mu \in \mathbb{R}$, one may naturally recenter the spatial process by a consistent estimate $\hat{\mu}$ of $\mu$, without damaging the learning rates established in Section \ref{sec:main} provided its accuracy is high enough.
\end{remark}
  
  As recalled in the subsection below, in spite of the major difference regarding the statistical setup, the Kriging and regression problems share similarities in the optimization approach considered to solve them.

\subsection{Kernel Ridge Regression}
In the most basic and usual regression setup, one has a system consisting of a square integrable random output $Y$ taking its values in $\mathbb{R}$ and a random vector $\bZ=(Z_1,\; \ldots,\; Z_d)$ concatenating a set of random input/explanatory random variables. Based on a training sample $\{(Y_1,\bZ_1),\; \ldots,\; (Y_N, \bZ_N)\}$ composed of independent copies of the random pair $(Y,\bZ)$, the goal pursued is to build a predictor $\widehat{F}_N(z)$ mimicking the mapping $F^*(z)=\mathbb{E}[Y \mid \bZ ]$ that minimizes the MSE
\begin{equation}\label{eq:MSE_reg}
L(F)=\mathbb{E}\left[ \left( Y-F(\bZ) \right)^2 \right]
\end{equation}
over the ensemble of all possible measurable mappings $F:\mathbb{R}^d\rightarrow \mathbb{R}$ that are square integrable with respect to the distribution of $\bZ$. Many practical approaches are based on the Empirical Risk Minimization (ERM) principle, which consists in replacing in the optimization problem the criterion \eqref{eq:MSE_reg}, unknown just like $(Y,\bZ)$'s distribution, by a statistical version computed from the training sample available, in general
$$
\widehat{L}_N(F)=\frac{1}{N}\sum_{i=1}^N\left( Y_i -F(\bZ_i)\right)^2,
$$
and restricting minimization to a class $\mathcal{F}$ of predictors of controlled complexity, see \textit{e.g.} \citep{GKKW02}. Under the assumption that the random variables $Y$ and $\{F(\bZ):\; F\in \mathcal{F}\}$ have sub-Gaussian tails, the order of magnitude of the fluctuations of the maximal deviations $\sup_{F\in \mathcal{F}}\vert \widehat{L}_N(F)-L(F)\vert $ can be estimated and generalization bounds for the MSE of empirical risk minimizers can be established, see \textit{e.g.} \citep{LecMend16}. In \textit{linear regression}, the class considered is that composed of all linear functionals on $\mathbb{R}^d$, namely $\mathcal{F}=\{\langle \zeta, \cdot \rangle:\; \zeta\in \mathbb{R}^d \}$, where $\langle \cdot, \cdot \rangle$ denotes the usual inner product on $\mathbb{R}^d$. In linear ridge regression (LRR), one thus minimizes the empirical error plus a quadratic penalty term to avoid overfitting
$$
\frac{1}{2}\sum_{i=1}^N\left( Y_i-\; \zeta^\top\bZ_i \right)^2+\frac{\eta}{2}\vert\vert \zeta\vert\vert^2,
$$
where $\eta\geq 0$ is a tuning parameter that rules the trade-off between complexity penalization and goodness-of-fit (generally selected via cross-validation in practice). It yields 
\begin{equation}\label{eq:LRR}
\widehat{\zeta}_{\eta}=\; \bY_N^\top \mathcal{Z}_N \left( \eta \bI_d + \mathcal{Z}_N\; \mathcal{Z}_N^\top  \right)^{-1},
\end{equation}
where $\bI_d$ is the $d\times d$ identity matrix,  $\bY_N=(Y_1,\; \ldots,\; Y_N)$ and $\mathcal{Z}_N$ is the $d\times N$ matrix with $\bZ_1,\; \ldots,\; \bZ_N$ as column vectors, as well as the predictive mapping 
$
\widehat{F}_N(z)=\; \widehat{\zeta}_{\eta}^\top z$, $z\in \mathbb{R}^d$. The regularization term ensures that the matrix inversion involved in \eqref{eq:LRR} is always well-defined by bounding the smallest eigenvalues away from zero. In kernel ridge regression (KRR), one applies LRR in the feature space, \textit{i.e.} to the data $(Y_1,\Phi(\bZ_1)),\; \ldots,\; 
(Y_N, \Phi(\bZ_N))$, where $\Phi:\mathbb{R}^d\rightarrow \mathcal{H}$ is a \textit{feature map} taking its values in a reproducing kernel Hilbert space (RKHS) $(\mathcal{H},\; \langle \cdot,\cdot\rangle_{\mathcal{H}})$ associated to a (positive definite) kernel $K$ such that $\forall (z,z')\in \mathbb{R}^d\times \mathbb{R}^d$, $K(z,z')=\langle \Phi(z),\Phi(z')\rangle_{\mathcal{H}}$. By means of the kernel trick, the predictive mapping, linear in the feature space, can be written as 
\begin{equation}\label{eq:opt0}
\widetilde{F}_N(z)= \; \bY_N^\top(\eta \bI_N +\mathcal{K}_N)^{-1}\kappa_N(z),
\end{equation}
where $\bI_N$ is the $N\times N$ identity matrix, $\mathcal{K}_N$ is the Gram matrix with entries $K(\bZ_i,\bZ_j)=\langle \Phi(\bZ_i),\Phi(\bZ_j)\rangle_{\mathcal{H}}$ for $1\leq i,\; j\leq N$ and $\kappa_N(z)=(K(\bZ_1,z),\; \ldots,\; K(\bZ_N,z))\in \mathbb{R}^N$. One may refer to \citep[Chapter 9]{steinwart} for more details on support vector machines for regression.

\section{Main Results}\label{sec:main}

This section states the main results of the paper. As a first go, we highlight the similarity between KRR and the dual Kriging problem, the latter being viewed as a multitask learning problem with an infinite number of tasks. See also \citep{kanagawa2018gaussian} for a detailed discussion about the connections between Gaussian processes and kernel methods. We next explain how a solution to Kriging with statistical guarantees in the form of non-asymptotic learning rate bounds can be derived from an accurate estimator of the covariance under appropriate conditions.

\subsection{Viewing Dual Kriging as a KRR Problem}\label{subsec:connect}
The purpose of \textit{simple Kriging} can be more ambitious than the construction of a pointwise prediction for the random field $X$ at an unobserved site $s\in S$ through linear interpolation of the discrete observations available as formulated in subsection \ref{subsec:kriging}. The goal pursued may consist in building a decision function $f:S\times \mathbb{R}^d\to \mathbb{R}$ in order to predict $X$ over all $S$ based on the observation of the spatial process at a finite number of points $s_1,\; \ldots,\; s_d$ in the spatial set $S$.
The global accuracy of the predictive map $f$ over the entire set $S$ can be measured by the
 integrated mean squared error
\begin{equation}\label{eq:IMSE}
L_S(f)=\int_{s\in S}L(s,f(s))\,ds=
\mathbb{E}_X\left[ \int_{s\in S}\left( f(s,\; \bX(\bs_d))-X_s \right)^2\, ds \right].
\end{equation}
Hence, the objective is to predict an infinite number of output variables $X_s$, with $s\in S$, based on the input variables $X_{s_1},\; \ldots,\; X_{s_d}$. It can be viewed as a multitask predictive problem with an infinite number of tasks.
In the ordinary formulation, the prediction $f(s)$ at any point $s\in S$ is assumed to be a linear combination of the $X_{s_i}$'s
 \begin{equation}\label{eq:type}
 f_{\Lambda_d}(\cdot,\; \bX(\bs_d))=\lambda_{1}(\cdot)X_{s_1}+\ldots+\lambda_{d}(\cdot)X_{s_d},
 \end{equation}
  where $\Lambda_d:s\in S\mapsto (\lambda_1(s),\; \ldots,\; \lambda_d(s))$ is a measurable function valued in $\mathbb{R}^d$. In this case, we have:
  \begin{equation}\label{eq:error}
L_S(f_{\Lambda_d})=\int_{s\in S}\left(Var(X_s)+\; \Lambda_d(s)^\top\Sigma(\bs_d)\Lambda_d(s)-2\; \bc_d(s)^\top\Lambda_d(s)  \right)ds.
  \end{equation}
 The optimal predictive rule of this form regarding the expected prediction error can be straightforwardly deduced from Lemma~\ref{lem:opt_kriging}. It is described in the result stated below, the proof of which is omitted.
  \begin{lemma}\label{lem:opt_min_error} Suppose that the hypotheses of Lemma \ref{lem:opt_kriging} are fulfilled. Define the predictive mapping: $\forall s\in S$,
\begin{equation}\label{eq:opt2}  
  f_{\Lambda_d^*}(s, \bX(\bs_d))=\langle\Lambda_d^*(s), \bX(\bs_d) \rangle=\; \bX(\bs_d)^\top \Sigma(\bs_d)^{-1}\bc_d(s).
  \end{equation}
  We have
    \begin{equation}
  f_{\Lambda_d^*}=\argmin_{f}L_S(f).
  \end{equation}
  In addition, the minimum global error is
  $$
  L_S(f_{\Lambda_d^*})=\int_{s\in S}\left(Var(X_s)-\; \bc_d(s)^\top\Sigma(\bs_d)^{-1}\bc_d(s)  \right)\, ds.
  $$
  \end{lemma}
  
  Observe that the mapping \eqref{eq:opt2} at $s\in S$ has the same form as that of the kernel ridge regressor at $z\in \mathbb{R}^d$ in \eqref{eq:opt0}, except that the regularized Gram matrix $\eta\bI_N+\mathcal{K}_N$ is replaced by $\Sigma(\bs_d)$, the vector $\bY_N$ by $\bX(\bs_d)$ and $\kappa_N(z)$ by $\bc_d(s)$.\\
  
  \medskip
  
  \noindent {\bf Plug-in predictive rules.} The quantities $\Sigma(\bs_d)$ and $\bc_d(s)$ are unknown in practice just like the risk \eqref{eq:error} and must be replaced by estimators in order to form an estimator $\widehat{\Lambda}_d$ of $\Lambda_d^*$ (or an empirical version of \eqref{eq:error}). For this reason, establishing rate bounds that assess the generalization capacity of the resulting predictive map $f_{\widehat{\Lambda}_d}$ is far from straightforward. It is the aim of the subsequent analysis to develop a non-asymptotic and non-parametric framework for simple Kriging with statistical guarantees, based on a preliminary finite-sample study of the performance of a covariance estimator $\widehat{C}(s,t)$, requiring no prior parametric modelling. The angle embraced here is thus different from that usually adopted in the traditional Kriging literature, often calling forth the use of MLE methods (see \textit{e.g.} section 5.3.3 in \cite{gaetan2009spatial}). In contrast, it is akin to that of statistical learning, particularly relevant when the availability of large training datasets permit to consider flexible techniques, avoiding the specification of a parametric class of probability laws.
    \begin{remark} {\sc (Gaussian random fields, bis)} We point out that, in the case where the random field $X$ is Gaussian, the mapping $ f_{\Lambda_d^*}$ is a minimizer of the global error $L_S$ over the set of all predictive rules $f(s,\; \bX(\bs_d))$ such that the integrated MSE \eqref{eq:error} is well-defined.
  \end{remark}

\begin{remark} {\sc (Worst case error \textit{vs} integrated error)}
Rather than integrating the pointwise MSE over the spatial domain $S$ to define global accuracy of a predictive map $f$ (see \eqref{eq:IMSE}), one may naturally consider the supremum of the MSE over $S$, namely $\sup_{s\in S} L(s,\; f(s))$.  Notice that, under the assumptions stipulated, the rate bound results obtained in the subsequent analysis obviously remain valid when substituting the IMSE with it.
\end{remark}

\subsection{Nonparametric Covariance Estimation}\label{subsec:nonpar_cov_est}
As mentioned above, the estimation of the covariance function, which the plug-in predictive approach considered here fully relies on, is based on a `large` number $n\geq 1$ of observations $\bX(\sigma_n) := (X_{\sigma_1},\; \ldots,\; X_{\sigma_n})$ exhibiting a certain dependence structure and cannot rely on independent realizations of the random field $X$, in contrast to the usual statistical learning setup. The \textit{in-fill} asymptotic stipulates that the number of observations inside the fixed and bounded domain $S$ increases, the latter forming a denser and denser grid as $n\rightarrow +\infty$. 
\medskip

\noindent {\bf Regular grids.} The observed sites are supposed to be equispaced, forming a regular grid. Since the goal of this paper is to explain the main ideas rather than dealing with the problem in full generality, we assume for simplicity that the spatial set $S$ is equal to the unit square $[0,1]^2$ and the observed sites are the points of the dyadic grid at scale $J\geq 1$ (see Figure \ref{fig:dyadic_grid}):
\begin{equation}\label{eq:grid}
\mathcal{G}_J=\left\{(k2^{-J},l2^{-J}):\;\;0\leq k,\; l\leq 2^{J}  \right\}.
\end{equation}
In this case, we have $n=(1+2^J)^2$. The gridpoints are indexed using the lexicographic order on $\mathbb{R}^2$ by assigning index $i=k(1+2^J)+(l+1)$ to point $(k2^{-J},l2^{-J})$, which is denoted by $\sigma_i$.
We point out that the regularity of the grid formed by the observed sites is key to the present analysis, to control the spectrum of the covariance matrix of the sampled points in a non-asymptotic fashion namely (see also Assumption \ref{hyp:bounded} below), so as to define an unbiased semi-variogram estimator at the observed lags with provable accuracy (\textit{cf} Proposition \ref{prop:CI_var_lag}). Proving that such a control still holds true for irregular grids (under specific assumptions unavoidably, remaining to be formulated precisely) is a great mathematical challenge (even in the 1-d time series case, see e.g. \cite{BDbook}) and is the subject of an ongoing research, as explained at length in the Appendix section. Extending the present theoretical study to observations on irregular grids is of importance undeniably, insofar that it may cover various situations in practice, but will be the subject of further research. However, beyond technical barriers, one should pay attention to the fact that measurements at equispaced spatial sites are extremely common and of great interest in practice, since they are precisely those produced by numerous image-producing systems, in Geophysics for instance.  For this reason, the Geostatistical literature (see \textit{e.g.} \citep[Section 2.4]{Cressie_1993}) has mainly focused on the regular case, while the irregular case is in contrast poorly documented and no dedicated statistical study of the (even non-asymptotic) behavior of the variogram/covariance estimator has been carried out yet.

\begin{figure}
\centering
    \includegraphics[scale = 0.17]{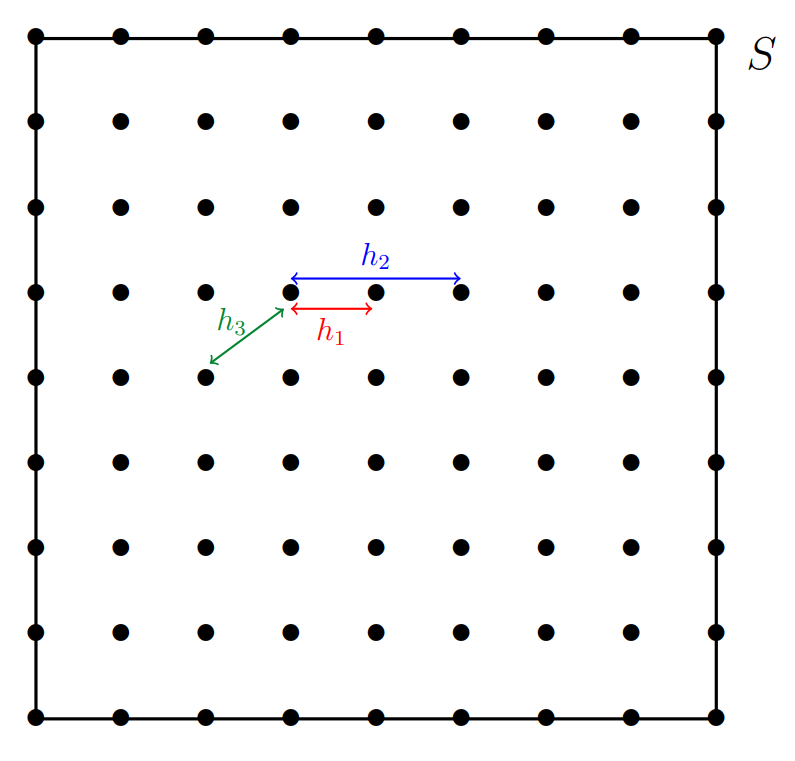}
    \caption{Dyadic grid at scale $J = 3$ ($n = 81$). Depicted lags: $h_1 = 2^{-J}$ (in red), $h_2 = 2/2^J$ (in blue) and $h_3 = \sqrt{2}/2^J$ (in green).}
    \label{fig:dyadic_grid}
\end{figure}
\medskip

\noindent {\bf Gaussianity, stationarity and isotropy.} The assumption of (second-order) stationarity is required so that a frequentist approach can be successful in such a nonparametric framework. In the subsequent analysis, we suppose that the following hypotheses are fulfilled.

\begin{assumption}\label{hyp:stat_iso}
The centered random field $X$ is stationary in the second-order sense and isotropic w.r.t. the Euclidean norm, \textit{i.e.}  there exists $c:\mathbb{R}_+\rightarrow \mathbb{R}$ s.t. $c(\vert\vert t-s\vert\vert )=C(s,t)$ for all $(s,t)\in S^2$.
\end{assumption}
\begin{assumption}\label{hyp:decorr}
There exists a known integer $j_1\geq 1$ s.t. $c(h)=0$ as soon as $h\geq \sqrt{2}-2^{-j_1}$.
\end{assumption}

\begin{remark}{\sc (On Assumption~\ref{hyp:decorr})}
    We point out that many popular covariance models fulfil Assumption~\ref{hyp:decorr}. This is the case of the truncated power law, the cubic and the spherical covariance models, used to generate the datasets analyzed in the experiments presented in Section \ref{sec:num} (see also the Appendix section). On the contrary, the Gaussian, the exponential and the Matern covariance models do not satisfy this hypothesis. However, as shown in the Appendix, the prediction methodology still performs satisfactorily in these situations, provided that the decorrelation rate is fast enough.
\end{remark}

\begin{assumption}\label{hyp:gaussian}
The random field $X$ is Gaussian with positive definite covariance function.
\end{assumption}
Under Assumption \ref{hyp:gaussian}, the weak stationarity guaranteed by Assumption \ref{hyp:stat_iso} is of course equivalent to strong stationarity (\textit{i.e.} shift-invariance of the finite dimensional marginals) insofar as Gaussian laws are fully characterized by the mean and covariance functions. Notice also that Assumption \ref{hyp:gaussian} ensures in particular that $\Sigma(\bs_d)$ is invertible for any pairwise distinct sites $s_1,\; \ldots,\; s_d$, so that the Kriging predictor $f_{\Lambda_d^*}$ is well-defined. Under Assumptions \ref{hyp:simple}-\ref{hyp:stat_iso}, we have $C(s,t)=\mathbb{E}[X_tX_s]=c(h)$ for all $(s,t)\in S^2$ s.t. $h=\vert\vert s-t\vert\vert$. Supposing in addition that Assumption \ref{hyp:decorr} is fulfilled (the function $c$ can be then extended to $\mathbb{R}^2$ by setting $c(h)=0$ for all $h>\sqrt{2}$), a natural estimator $\widehat{c}$ of the covariance function is then defined by $\widehat{c}(h)=0$ if $h\geq \sqrt{2}-2^{-j_1}$ and otherwise by
\begin{equation}\label{eq:emp_cov}
    \widehat{c}(h)=\frac{1}{n_h}\sum\limits_{(\sigma_i, \sigma_j) \in N(h)}X_{\sigma_i}X_{\sigma_j},
\end{equation}
where $N(h) = \left\{(\sigma_i, \sigma_j), \, \|\sigma_i - \sigma_j\| = h, \, (i, j) \in \llbracket 1, n \rrbracket^2\right\}$ is the set of pairs of sites that are at distance $h$ from one another and $n_h = \vert N(h)\vert$ denotes its cardinality (notice that $n_0=n$). Equipped with these notations, notice that a pair of sites $(\sigma_i, \sigma_j) \in N(h)$ at distance $h$ from one another is taken into account twice in the set $N(h)$, since $\| \sigma_i - \sigma_j \| = \| \sigma_j - \sigma_i \| = h$ and so $(\sigma_j, \sigma_i) \in N(h)$. Define $\mathcal{H}_n = \{\vert\vert \sigma_i-\sigma_j\vert\vert:\; (i,j)\in\llbracket 1, n \rrbracket^2\}$ the set of observed lags, which are all less than $\sqrt{2}$. The lemma stated below shows that $n_h$ is of order $n$ for observed lags $h<\sqrt{2}-2^{-j_1}$ as soon as $n>(\sqrt{2}-2^{-j_1})^2$, thus ensuring that the number of terms averaged in \eqref{eq:emp_cov} is large enough, of order $n$ namely. Refer to the Appendix for the technical proof.
\begin{lemma}\label{lem:nh_n}
Suppose that Assumption \ref{hyp:decorr} is fulfilled and let $n>(\sqrt{2}-2^{-j_1})^2$. Then, there exists a constant $0 < \nu \leq 1$ depending on $j_1$ only such that: $\forall h \in \mathcal{H}_n$, s.t. $h < \sqrt{2} - 2^{-j_1}$, $n_h >\nu n$.
\end{lemma}

\begin{figure}
    \centering
    \subfigure[Grid at scale $J = 2$.]{\includegraphics[width=45mm]{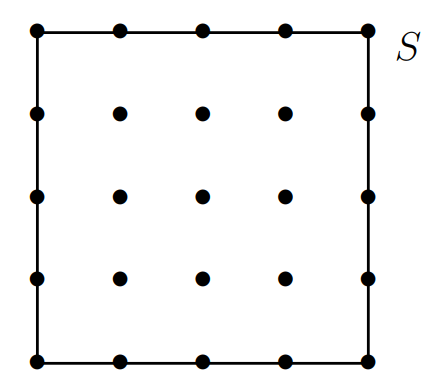}
        \label{fig:gridh_sub1}}
    \subfigure[$h = 2^{-J}$ and $n_h = 80$.]{\includegraphics[width=45mm]{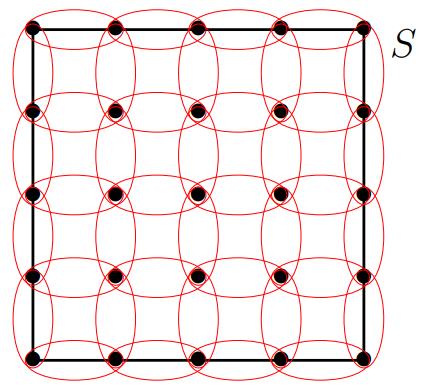}
        \label{fig:gridh_sub2}}\\
    \subfigure[$h = 1$ and $n_h = 20$.]{\includegraphics[width=45mm]{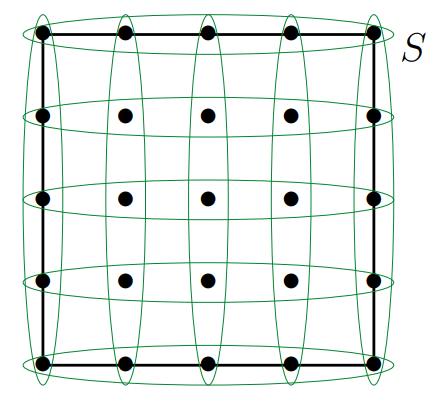}
        \label{fig:gridh_sub3}}
    \subfigure[$h = \sqrt{2}$ and $n_h = 4$.]{\includegraphics[width=45mm]{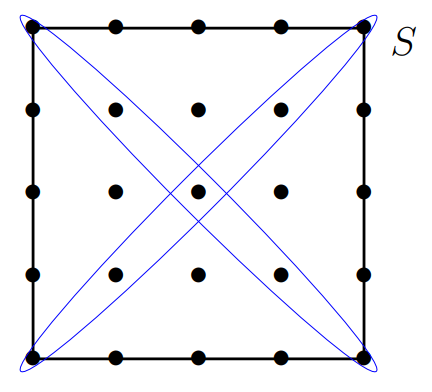}
        \label{fig:gridh_sub4}}
    \caption{Changes in the value of the number $n_h$ of pairs of sites that are at distance $h$ from one another, for different values of the lag $h$, on a dyadic grid at scale $J = 2$ ($n = 25$).}
    \label{fig:gridh}
\end{figure}

Figure \ref{fig:gridh_sub1} depicts the case of a dyadic grid at scale $J = 2$ for illustration purpose. For small values of $h$ (see Figure \ref{fig:gridh_sub2}), the number $n_h$ is large, so that a sufficient number of pairs of sites can be involved in the computation of $\widehat{c}(h)$, \textit{cf} \eqref{eq:emp_cov}. As $h$ grows, the number $n_h$ decreases (Figure \ref{fig:gridh_sub3}). In the extreme situation, \textit{i.e.} when $h = \sqrt{2}$, the number of pairs is reduced to $4$, see Figure \ref{fig:gridh_sub4}.
Hence, Assumption \ref{hyp:decorr}, stipulating that the covariance function vanishes for lags exceeding the threshold $\sqrt{2}-2^{-j_1}$, guarantees that a sufficient number of pairs of observations can be used to estimate the non zero values taken by $c(\cdot)$ on the set of lags formed by the gridpoints. Of course, as explained in the Appendix section, this hypothesis can be relaxed, by assuming a specific decay rate for $c(h)$ as $h$ tends to $\sqrt{2}$ (or equivalently for $c(\sqrt{2}-2^{-j_1})$ as $j_1$ tends to $\infty$). For the sake of simplicity, in order to avoid an excessive number of parameters involved in the problem statement, the inference and statistical learning results are established under Assumption \ref{hyp:decorr} and extensions are left to the reader. As will be seen below, beyond the number of pairs over which one averages to compute the statistic \eqref{eq:emp_cov}, the Gaussian hypothesis, Assumption \ref{hyp:gaussian}, plays a crucial role to describe (the concentration properties of) its distribution.

\medskip

\noindent {\bf Covariance \textit{vs} semi-variogram.} In Geostatistics, rather than the covariance, one uses the semi-variogram to characterize the second-order dependence structure of the observations, namely
\begin{equation}\label{eq:var}
    \gamma(h) = \frac{1}{2} Var(X_{s+h} - X_s) = \frac{1}{2} \mathbb{E}[(X_{s+h} - X_s)^2].
\end{equation}
Its main advantage lies in the fact that its computation does not require the knowledge of the (supposedly constant) mean. It is linked to the covariance by the equation $\gamma(h) = c(0) - c(h)$. Observe that, for any lag $h\in \mathcal{H}_n$, unbiased estimators of $\gamma(h)$ based on the values of the random field $X$ at the observed sites are:
\begin{equation}\label{eq:estvar1}
    \forall h < \sqrt{2} - 2^{-j_1}, \quad \widehat{\gamma}(h) = \frac{1}{2 n_h} \sum\limits_{(\sigma_i, \sigma_j) \in N(h)} (X_{\sigma_i} - X_{\sigma_j})^2,
\end{equation}
and $\widehat{\gamma}(h) = 0$ otherwise. Observe also that
\begin{equation}\label{eq:estvar2}
\forall h < \sqrt{2} - 2^{-j_1}, \quad \widehat{c}_h(0):=\widehat{\gamma}(h)+\widehat{c}(h)=\frac{1}{n_h}\sum_{i=1}^nn_h(i)X_{\sigma_i}^2,
\end{equation}
and $\widehat{c}_h(0) = 0$ otherwise, where $n_h(i)=\vert\{j\in \{1,\; \ldots,\; n\}:\; (\sigma_i,\sigma_j)\in N(h)\}  \vert$ for $i\in\{1,\; \ldots,\; n\}$.
Indeed, one may write: $\forall h \geq 0$,
\begin{multline*}
     \widehat{\gamma}(h) = \frac{1}{2 n_h} \sum\limits_{(\sigma_i, \sigma_j) \in N(h)} (X_{\sigma_i}^2 + X_{\sigma_j}^2) - \frac{1}{n_h} \sum\limits_{(\sigma_i, \sigma_j) \in N(h)} X_{\sigma_i} X_{\sigma_j}\\  = \frac{1}{2 n_h} \sum\limits_{(\sigma_i, \sigma_j) \in N(h)} (X_{\sigma_i}^2 + X_{\sigma_j}^2) - \widehat{c}(h),
\end{multline*}
and
\begin{multline*}
    \frac{1}{2 n_h} \sum\limits_{(\sigma_i, \sigma_j) \in N(h)} (X_{\sigma_i}^2 + X_{\sigma_j}^2) = \frac{1}{2 n_h} \sum\limits_{i=1}^n X_{\sigma_i}^2 \sum\limits_{j=1}^n \mathbb{I}\{(\sigma_i, \sigma_j) \in N(h)\}\\  + \frac{1}{2 n_h} \sum\limits_{j=1}^n X_{\sigma_j}^2 \sum\limits_{i=1}^n \mathbb{I}\{(\sigma_i, \sigma_j) \in N(h)\}  = \frac{1}{n_h} \sum\limits_{i=1}^n X_{\sigma_i}^2 n_h(i).
\end{multline*}

Under Assumption \ref{hyp:gaussian}, the distributions of the estimators \eqref{eq:estvar1}-\eqref{eq:estvar2} can be classically made explicit, as revealed by the result stated below, see \textit{e.g.} \citep{gaetan2009spatial, Cressie_1993}.

\begin{proposition}\label{prop:Gaetan_Guyon}
    Suppose that Assumption \ref{hyp:gaussian} is fulfilled. Let $h\in \mathcal{H}_n$.
    Denote by $L(n, h)$ the symmetric positive semi-definite (Laplacian) matrix with entries $L_{i, j}(n, h) = - \mathbb{I}\{(\sigma_i, \sigma_j) \in N(h)\}$ if $i \ne j$ and $L_{i, i}(n,h) =n_h(i)$ and by $\ell_i(h)$'s the $n_h$ eigenvalues of the symmetric positive semi-definite matrix $L(n, h) \Sigma_n$, where $\Sigma_n=\Sigma(\sigma_1,\; \ldots,\; \sigma_n)$. Denote also by $D(n, h)$ the diagonal matrix with entries $D_{i, i}(n, h) = n_h(i)$ and by $\rho_i(h)$'s the $n_h$ eigenvalues of the symmetric positive semi-definite matrix $D(n, h) \Sigma_n$. The following assertions hold true.
    \begin{itemize}
   \item[(i)] The estimators \eqref{eq:estvar1}-\eqref{eq:estvar2} are distributed as follows:
        \begin{equation}\label{eq:chi}
            \widehat{\gamma}(h) \sim \frac{1}{n_h} \sum\limits_{i=1}^{n_h} \ell_i(h) \chi_{i}^2 \text{ and } \widehat{c}_h(0) \sim \frac{1}{n_h} \sum\limits_{i=1}^{n_h} \rho_i(h) \chi_{i}^2,
        \end{equation}
        where the $\chi_i^2$'s are independent $\chi^2$ random variables with one degree of freedom.
\item[(ii)] The $\ell_i(h)$'s and the $\rho_i(h)$'s are strictly positive, \textit{i.e.} the matrices $L(n, h) \Sigma_n$ and $D(n, h) \Sigma_n$ are positive definite.
        \end{itemize}
        
\end{proposition}
Refer to the Appendix section for the technical argument.
Recall also that, under Assumptions \ref{hyp:stat_iso}-\ref{hyp:decorr}, the spatial process $X$ has an isotropic spectral density $\Phi(u)=(2\pi)^{-2}\int_{s\in \mathbb{R}^2}\exp(-i\; s^\top u)c(\vert\vert s\vert\vert)ds=\phi(\vert\vert u\vert\vert)$, see \textit{e.g.} \citep[Chapter 2]{stein1999interpolation}. The additional hypothesis below is required in the subsequent analysis. It classically guarantees that the eigenvalues of the covariance matrix of the spatial process sampled on the regular grid $\mathcal{G}_J$ are bounded and bounded away from $0$, see \textit{e.g.} \citep{BDbook}.
\begin{assumption}\label{hyp:bounded}
There exist $0<m\leq M<+\infty$ such that: $\forall u\in \mathbb{R}^2$, 
$$
m\leq \sum_{k\in \mathbb{Z}^2}\Phi(u+2\pi k)\leq M.
$$
\end{assumption}
Combining then Lemma \ref{lem:nh_n} with Proposition \ref{prop:Gaetan_Guyon} and classic bounds for the largest eigenvalues of Laplacian matrices (see the auxiliary results in the Appendix), one may deduce Poisson tail bounds for the deviations between the unbiased estimators and their expectations from the exponential inequalities established for Gamma r.v.'s, see \textit{e.g.} \citep{bercu2015concentration, Wang_Ma}. The proof is detailed in the Appendix section, together with intermediary results involved in its argument.

\begin{proposition}\label{prop:CI_var_lag}
Suppose that Assumptions \ref{hyp:simple}--\ref{hyp:bounded} are fulfilled. Let $h\in \mathcal{H}_n$. For all $t > 0$, we have:
\begin{eqnarray*}
\mathbb{P}\left( \left\vert \widehat{\gamma}(h)-\gamma(h) \right\vert \geq t \right) &\leq&  e^{-C_1 n t}+ e^{-C'_1 n t^2},\\
\mathbb{P}\left(\left\vert\widehat{c}_h(0)-c(0)\right\vert \geq t\right) &\leq& e^{-C_2 n t}+e^{-C'_2 n t^2},
\end{eqnarray*}
where $C_i$ and $C'_i$, $i\in\{1,\; 2\}$, are positive constants depending on $j_1$, $m$ and $M$ solely.
\end{proposition}

\medskip

\noindent {\bf Estimation of the covariance function.} The empirical covariance function $\widehat{c}(h)$ can be extrapolated at unobserved lags $h\in[0,\sqrt{2}-2^{-j_1}]\setminus \mathcal{H}_n$ by means of various nonparametric procedures, such as local averaging methods. For simplicity, one may consider a piecewise constant estimator, for instance the $1$-NN estimator $\widehat{c}(h)=\widehat{c}(l_h)$, where $l_h=\argmin_{l\in \mathcal{H}_n}\vert\vert h-l \vert\vert$ (breaking ties in an arbitrary fashion). As $\vert\vert h-l_h\vert\vert \leq 2^{-J}=1/(\sqrt{n}-1)$, the (weak) smoothness hypothesis below then permits to control the covariance estimation error at unobserved lags.

\begin{assumption}\label{hyp:smooth}
The mapping $h\in [0,\; \sqrt{2}-2^{-j_1}]\mapsto c(h)$ is of class $\mathcal{C}^{1}$ with gradient bounded by $D<+\infty$ and there exists $0 < B < + \infty$, such that $\sup\limits_{h \geq 0} \vert c(h) \vert \leq B$.
\end{assumption}
Of course, under more restrictive regularity assumptions, the accuracy of other classic nonparametric estimation techniques (\textit{e.g.} splines of degree larger than $2$) can be established, refer to the discussion in the Appendix section. 
 The result below is proved at length in the Appendix.
 Under Assumption \ref{hyp:smooth}, it simply follows from Proposition \ref{prop:CI_var_lag} combined with the union bound and the finite increment inequality.
\begin{corollary}\label{cor} Suppose that Assumptions \ref{hyp:simple}--\ref{hyp:smooth} are satisfied. Then, for any $\delta\in (0,1)$, we have with probability at least $1-\delta$:
$$
\sup_{h\geq 0}\left\vert \widehat{c}(h)- c(h) \right\vert \leq C_3\sqrt{\log(4 n/\delta)/n}+ D/(\sqrt{n}-1),
$$
as soon as $n\geq C'_3\log(4 n/\delta)$, where $C_3$ and $C'_3$ are positive constants depending on $j_1$, $m$ and $M$ solely.
\end{corollary}

To the best of our knowledge, this non-asymptotic bound for a nonparametric estimator of the covariance function in the in-fill setup is the first result of this kind, the vast majority of the results documented in the literature being either of asymptotic nature and/or related to parametric inference. One may refer to \textit{e.g. }\cite{hall1994nonparametric} for limit results related to kernel smoothing methods applied to covariance estimation.
Based on the estimator $\widehat{c}$ of the covariance function, one may derive statistical counterparts of the quantities involved in \eqref{eq:opt2}, namely $\Sigma(\mathbf{s}_d)$ and $\mathbf{c}_d(s)$ and form an empirical version of the optimal Kriging rule. As will be shown in the next subsection, generalization bounds for the performance of the predicting function thus constructed can be established from the non-asymptotic guarantees stated in the corollary above.

\subsection{Excess Risk Bounds in Simple Kriging}\label{subsec:bounds_ER}
Equipped with the non-asymptotic results established in the subsection above, we now address the simple Kriging problem from a predictive learning perspective, as formulated in subsections \ref{subsec:kriging} and \ref{subsec:connect}. Let $d\geq 1$ and consider arbitrary pairwise distinct sites $s_1,\; \ldots,\; s_d$ in $S=[0,1]^2$. The goal is to predict the value $X_s$ taken by $X$ at any site $s\in S$ based on $(X_{s_1},\; \ldots,\; X_{s_d})$, nearly as accurately as the optimal Kriging rule \eqref{eq:opt2} would do it. For this purpose, one uses a training dataset, composed of observations $X'_{\sigma_1},\; \ldots,\; X'_{\sigma_n}$ of $X'$, an independent copy of $X$, at sites $\sigma_1,\; \ldots,\; \sigma_n$ forming a regular dyadic grid. Consider $\widehat{c}(\cdot)$, the estimator of the covariance function studied in subsection \ref{subsec:nonpar_cov_est}, based on the $X'_{\sigma_i}$'s. From $\widehat{c}(\cdot)$, the covariance matrix $\Sigma(\bs_d)$ and the covariance vector $\bc_d(s)$ can be naturally estimated as follows:
\begin{eqnarray}
\widehat{\bc}_d(s)&=&\left( \widehat{c}(\vert\vert s-s_1\vert\vert,\; \ldots,\; \widehat{c}(\vert\vert s-s_d\vert\vert \right)\text{ for } s\in S,\label{eq:c_hat}\\
\widehat{\Sigma}(\bs_d)&=& \left( \widehat{c}(\vert\vert s_i-s_j\vert\vert )  \right)_{1\leq i,\; j\leq d}.\label{eq:sigma_hat}
\end{eqnarray}

 \begin{assumption}\label{hyp:cov_bound}
Let $0 < \underline{m} \leq \overline{M} < +\infty$ and assume that the eigenvalues of the covariance matrix $\Sigma(\bs_d)$ are upper bounded by $\overline{M}$, lower bounded by $\underline{m}$.
\end{assumption}

\noindent {\bf Estimation of the precision matrix.} Under Assumption \ref{hyp:gaussian}, the matrix $\Sigma(\bs_d)$ is always invertible, which permits to define the Kriging rule \eqref{eq:opt2}, involving the precision matrix $\Sigma(\bs_d)^{-1}$. The simplest way of building an estimate of the precision matrix is to invert \eqref{eq:sigma_hat}, when it is positive definite. This theoretically happens with overwhelming probability, as shown by the result stated below, and turned out to be true in all the numerical experiments presented in Section \ref{sec:num} and in the Appendix section. Hence, the estimator of the precision
matrix we consider here in order to build an empirical version of the predicting function \eqref{eq:opt2} is the inverse of \eqref{eq:sigma_hat}, when
the latter is definite positive,
and that of any definite positive regularized version (e.g. \cite{Tikhonov}) of the
latter otherwise. It is (possibly abusively) denoted by $\widehat{\Sigma}(\bs_d)^{-1}$ in both situations. Its accuracy is described in a non-asymptotic fashion by the bound stated in the result below.
\begin{proposition}\label{prop:CI_Sigma_vec} Suppose that Assumptions \ref{hyp:simple}--\ref{hyp:cov_bound} are satisfied. The following assertions hold true.
\begin{enumerate}
    \item[(i)] For any $\delta\in (0,1)$, we have with probability at least $1-\delta$:
    $$\vert\vert\vert\widehat{\Sigma}(\bs_d)- \Sigma(\bs_d)\vert\vert\vert \leq C_3 \, d \, \sqrt{\log(4 n/\delta)/n} + d \, D/(\sqrt{n} - 1),$$
    as soon as $n \geq C'_3 \log(4 n/\delta)$, where $C_3$ and $C'_3$ are positive constants depending on $j_1$, $m$ and $M$ solely (see Corollary \ref{cor}).
    \item[(ii)] For any $\delta\in (0,1)$, we have with probability at least $1-\delta$:
    $$\vert\vert\vert\widehat{\Sigma}(\bs_d)^{-1}- \Sigma(\bs_d)^{-1}\vert\vert\vert \leq C_4 \, d \, \sqrt{\log(4 n/\delta)/n} + C'_4 \, d \, D/(\sqrt{n} - 1),$$
    as soon as $n \geq C''_4 \log(4 n/\delta)$, where $C_4$, $C'_4$ and $C''_4$ are positive constants depending on $j_1$, $m$, $M$, $\underline{m}$ and $\overline{M}$ solely.
\end{enumerate}
\end{proposition}

Assertion \textit{(i)} simply follows from Corollary \ref{cor}, the operator norm and the max norm being equivalent in finite dimension. The proof of the second assertion uses a classic inequality for inverse matrices as in \cite{wedin1973perturbation} combined with Assertion \textit{(i)}. Refer to the Appendix section for further technical details.

\medskip

\noindent {\bf Empirical Risk Minimization.} Now, replacing $\Sigma(\bs_d)^{-1}$ and $\bc_d(s)$ by the estimators introduced above, a natural statistical counterpart of $\Lambda^*_d$ is built by means of the \textit{plug-in} method:
\begin{equation}\label{eq:emp_krig_def}
    \widehat{\Lambda}_d(s)=\; \widehat{\Sigma}(\bs_d)^{-1}\widehat{\bc}_d(s).
\end{equation}
We point out that it actually corresponds to an empirical risk minimizer. Indeed, \eqref{eq:emp_krig_def} is the minimizer of 
$$
\Lambda_d(s)^\top\widehat{\Sigma}(\bs_d)\Lambda_d(s)-2\widehat{\bc}_d(s)^\top\Lambda_d(s)
$$
over $\Lambda_d(s)$ in $\mathbb{R}^d$, which functional can be viewed as an empirical version of $L(s,\; f_{\Lambda_d}(s))-c(0)=\; \Lambda_d(s)^\top\Sigma(\bs_d)\Lambda_d(s)-2\bc_d(s)^\top\Lambda_d(s)$, the pointwise risk at $s$ up to an additive term independent from $\Lambda_d(s)$ under the assumptions introduced in subsection \ref{subsec:nonpar_cov_est}.
Define thus the empirical predictive mapping $f_{\widehat{\Lambda}_d}$ by:
\begin{equation}\label{eq:f_emp}
f_{\widehat{\Lambda}_d}(s, \bx(\bs_d))=\langle\widehat{\Lambda}_d(s), \bx(\bs_d) \rangle=\; \bx(\bs_d)^\top \widehat{\Sigma}(\bs_d)^{-1}\widehat{\bc}_d(s),
  \end{equation}
  for all $s\in S$ and any $\bx(\bs_d)=(x_{s_1},\; \ldots,\; x_{s_d})\in \mathbb{R}^d$.
The (random) predictive function \eqref{eq:f_emp} can be used to predict the values taken by $X$,  any independent copy of the random field $X'$ partially observed in the learning/estimation phase, over the whole spatial domain $S$ based on the input observations $\bX(\bs_d)=(X_{s_1},\; \ldots,\; X_{\bs_d})$. Conditioned upon the $X'_{\sigma_i}$'s, it is of course a linear prediction rule which minimizes the statistical counterpart of $L_S(f_{\Lambda_d})$ based on the $X'_{\sigma_i}$'s, namely
\begin{equation}\label{eq:emp_risk}
\widehat{L}_S(f_{\Lambda_d}):=\int_{s\in S}\left(\widehat{c}(0)+\Lambda_d(s)^\top\widehat{\Sigma}(\bs_d)\Lambda_d(s)-2\widehat{\bc}_d(s)^\top\Lambda_d(s)  \right)ds,
\end{equation}
over all Borel measurable functions $\Lambda_d:S\to \mathbb{R}^d$ such that \eqref{eq:emp_risk} is well-defined (as previously noticed, the quantity integrated over $S$ then reaches its minimum at all $s$ in $S$). Hence, the plug-in predictive rule \eqref{eq:f_emp} can also be derived from Empirical Risk Minimization (ERM), the main paradigm of statistical learning, see \textit{e.g.} \cite{probabilistictheory}.
\medskip

\noindent {\bf Generalization capacity.} The predictive performance of the function $f_{\widehat{\Lambda}_d}$ constructed on the basis of the $X'_{\sigma_i}$'s is then measured by the conditional expectation, obtained by replacing $\Lambda_d(s)$ by its empirical counterpart $\widehat{\Lambda}_d(s)$ in \eqref{eq:IMSE} :
\begin{eqnarray*}
  L_S(f_{\widehat{\Lambda}_d})&=&  \mathbb{E}_{X}\left[\int_{s\in S}\left( f_{\widehat{\Lambda}_d}(s, \bX(\bs_d)) -X_s\right)^2\, ds \mid X'_{\sigma_1},\ldots, X'_{\sigma_n} \right]\\
  &=&\int_{s\in S}\left(c(0)+\; \widehat{\Lambda}_d(s)^\top\Sigma(\bs_d)\widehat{\Lambda}_d(s)-2\; \bc_d(s)^\top\widehat{\Lambda}_d(s)   \right)\, ds.
\end{eqnarray*}
It is a random quantity since it depends upon the training data, that is larger than $L^*_S:=L_S(f_{\Lambda_d^*})$ with probability one, see Lemma \ref{lem:opt_min_error}.
The theorem below shows that, with large probability, the prediction error of the empirical simple Kriging rule $f_{\widehat{\Lambda}_d}$ is close to the minimal prediction error $L^*_S$, assessing its generalization capacity at unobserved sites.
\begin{theorem}\label{thm:EGRisk}
Suppose that Assumptions \ref{hyp:simple}--\ref{hyp:cov_bound} are satisfied. The following assertions hold true.
\begin{enumerate}
    \item[(i)] For any $\delta\in (0,1)$, we have with probability at least $1-\delta$:
    $$\sup_{s\in S}\|\widehat{\Lambda}_d(s) - \Lambda^*_d(s)\| \leq C_5 \, d \, \sqrt{d \, \log(4n/\delta)/n}+ C'_5 \, d \, \sqrt{d} \, D/(\sqrt{n}-1),
$$
as soon as $n\geq C''_5\log(4 n/\delta)$, where $C_5$, $C'_5$ and $C''_5$ are positive constants depending on $j_1$, $m$, $M$, $\underline{m}$, $\overline{M}$, and $B$ solely.
    
    \item[(ii)] For any $\delta\in (0,1)$, we have with probability at least $1-\delta$:
    $$L_S(f_{\widehat{\Lambda}_d})-L^*_S\leq C_6 \, d^2 \, \sqrt{\log(4n/\delta)/n}+ C'_6 \, d^2 \, D/(\sqrt{n}-1),
$$
as soon as $n\geq C''_6\log(4 n/\delta)$, where $C_6$, $C'_6$ and $C''_6$ are positive constants depending on $j_1$, $m$, $M$, $\underline{m}$, $\overline{M}$, and $B$ solely.
    
\end{enumerate}
\end{theorem}

Assertion $(i)$ can be proved by exploiting the bounds obtained in subsection \ref{subsec:nonpar_cov_est} combined with Proposition \ref{prop:CI_Sigma_vec}, while the upper confidence bound for the excess of risk stated in Assertion $(ii)$ can be deduced from the latter by noticing that, with probability one, the excess of integrated quadratic risk can be written as follows:
\begin{multline}
     L_S(f_{\widehat{\Lambda}_d})-L^*_S=\\
     \int_{s\in S} \left( \widehat{\Lambda}_d(s)^\top\Sigma(\bs_d)\widehat{\Lambda}_d(s)- \Lambda_d^*(s)^\top\Sigma(\bs_d)\Lambda_d^*(s)  -2\; \bc_d(s)^\top\left(\widehat{\Lambda}_d(s) -\Lambda_d^*(s)\right)  \right) ds .
     \label{eq:excess_IQR}
\end{multline}
 The detailed proof is given in the Appendix, together with a discussion about possible ways of extending these results to a more general framework. These theoretical guarantees are illustrated by numerical results based on simulated/real spatial data in the next section (and in the Appendix). They clearly show that the prediction errors of the nonparametric empirical kriging method analyzed above get closer and closer to those of the theoretical one (based on the true covariance function) for a variety of spatial models, as the training size $n$ increases.
 %tend to increase with the correlation length (\textit{i.e.} with $j_1$), just like the related theoretical upper confidence bounds as may be shown by a careful examination of the proofs in the Supplementary Material.} 

\section{Illustrative Experiments}\label{sec:num}
We now illustrate the theoretical analysis carried out above by numerical results, illuminating the impact of the assumptions made to get finite sample guarantees. The (fully reproducible) experiments, available at \url{https://github.com/EmiliaSiv/Simple-Kriging-Code}, were implemented in Python~3.6, using the library \texttt{gstools} \citep{gstools}. Below, two covariance models for Kriging of a Gaussian random field $X$ are considered. First, an isotropic \textit{truncated power law} (TPL) covariance function is considered:
\begin{equation}\label{eq:truncpw}
c : h \in [0,+\infty) \mapsto \left(1 - h/\theta \right)^{\frac{3}{2}}\,\mathbb{I}\left\{h \leq \theta\right\},
\end{equation}
where $\theta \in \mathbb{R}_+^*$ is the correlation length. It satisfies all the assumptions involved in Theorem~\ref{thm:EGRisk}, see \citep{Golubov1981}. Next, empirical Kriging is applied to a Gaussian field $X$ with \textit{Gaussian} covariance function:
\begin{equation}\label{eq:gaussian_cov}
c : h \in [0,+\infty) \mapsto \exp\left(-h^2/\theta^2\right),
\end{equation}
which does not satisfy Assumption~\ref{hyp:decorr} but vanishes very quickly. Numerical experiments based on additional covariance models are displayed in the Appendix.

\subsection{Covariance Estimation Illustration}

First, given the spatial domain $S = [0,1]^2$, an independent realization of $X$ was simulated, serving for the nonparametric covariance estimation and referred to as the training spatial dataset $\bX'$. In accordance with the statistical framework considered in subsection \ref{subsec:nonpar_cov_est}, this realization is observed at $n$ sites $(\sigma_1, \cdots, \sigma_n)$ supposed to form a dyadic grid at scale $J \geq 1$, with $n = (1 + 2^J)^2$. The estimation of the covariance function of the random field is then computed using Equation~\eqref{eq:emp_cov} from the $n$ observations. In order to illustrate its accuracy, the estimator is computed for $100$ independent simulations of $X$ observed at the same fixed sites.

\begin{figure}[h]
    \centering
    \subfigure[TPL]{\includegraphics[width=60mm]{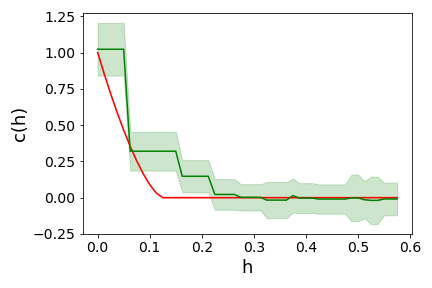}
        \label{fig:TPL_cov_est}}
    \subfigure[Gaussian]{\includegraphics[width=60mm]{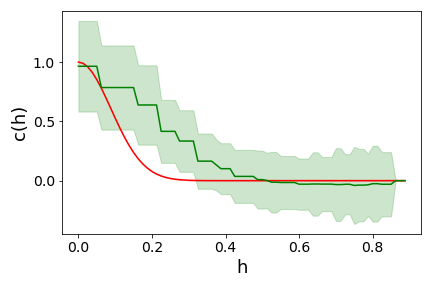}
        \label{fig:gauss_cov_est}}
    \caption{Estimation of the truncated power law (left) and the Gaussian (right) covariance functions, on a dyadic grid at scale $J = 3$ ($n = 81$), with $\theta = 5$. For each model, the red line corresponds to the true covariance function and the green line to the mean of the estimated one, together with the corresponding mean standard deviation (in green shaded bands), over $100$ replications.}
    \label{fig:cov_est}
\end{figure}

The corresponding results for the two covariance functions are depicted in Figure \ref{fig:cov_est} (for a dyadic grid at scale $J = 3$ and with correlation length fixed at $\theta = 5$), where, for each model, the true covariance function appears in red and the mean of the estimated one in green, together with the corresponding mean standard deviation over the $100$ replications.
Observe in Figure \ref{fig:TPL_cov_est}, for the truncated power law function, that the estimation is close to the true value and equal to zero beyond a certain threshold. For the Gaussian model, Figure \ref{fig:gauss_cov_est} shows that the estimation method is less accurate than for the previous covariance model. In particular, it unsuccessfully detects the true correlation parameter $\theta$.

\subsection{Simple Kriging Prediction Results for Simulated Data}\label{sec:exp_num_SK}

Based on the covariance estimates previously obtained from observations of a regular grid of size $n=(1+2^J)^2$, we now consider the simple Kriging problem from a predictive point of view, and simulate a new independent realization of $X$. The latter is observed at $d$ sites $(s_1, \cdots, s_d)$, randomly selected over the spatial domain $S$. As formulated in subsection \ref{subsec:connect}, the goal is to predict the value $X_s$ taken by $X$ at any site $s \in S$ based on the $d$ observations. Regarding the empirical evaluation of the predictive accuracy, the spatial domain $S$ is (regularly) discretized: the goal is to predict the value taken by $X$ at the corresponding $N\geq 1$ sites $s'_1,\; \ldots,\; s'_N$ in $S$. In compliance with the methodology analyzed in subsection \ref{subsec:bounds_ER}, the predictive mapping is constructed by means of the plug-in technique from the covariance vector and the covariance matrix estimators defined in \eqref{eq:c_hat} and \eqref{eq:sigma_hat}, see \eqref{eq:emp_krig_def}. The prediction error 
%were deduced by taking the squared difference for both the theoretical and the empirical Kriging rules. The quadratic difference is equal 
being the expected squared difference between the predicted random field and the true random field integrated (respectively, averaged) over (respectively, the discretized version of) the spatial domain $S$, we performed $100$ replications of the experiment, each one involving one simulation $\bX'$ for the training step and one simulation $X$ for the prediction test, the locations $(\sigma_1, \cdots, \sigma_n)$, $(s_1, \cdots, s_d)$ and $(s'_1,\; \ldots,\; s'_N)$ remaining fixed. In order to compare empirically the empirical Kriging method analyzed in the previous section to the 'Oracle' method based on the true covariance function (theoretical Kriging), the prediction techniques have been thus applied $100$ times, so that $2 \times 100$ prediction maps have been obtained. For each replication $(\bX',X)$ of the experiment, the (spatial) average over the discretized version of $S$ of the Mean Squared Error \eqref{eq:MSE} has been evaluated,
\begin{equation}\label{eq:emp_AMSE}
    AMSE = \frac{1}{N} \sum\limits_{t = 1}^N (f(s'_t)- X_{s'_t})^2,
\end{equation}
for $f=f_{\Lambda^*_d}$ (theoretical Kriging) and $f=f_{\widehat{\Lambda}_d}$ (empirical Kriging). The mean and standard deviation of \eqref{eq:emp_AMSE} have been computed over the $100$ replications.
To observe the effects of several parameters on the performance of the Kriging method, the experience was carried out for different sizes of the dyadic grid and different values of the correlation length $\theta$ ($\{2.5, 5, 7.5, 10\}$), the parameter used in the definition of the instrumental covariance functions. Note that, for the truncated power law covariance function, the parameter $\theta$ is linked to the parameter $j_1$ in Assumption~\ref{hyp:decorr}:
$j_1 = - \log(\sqrt{2} - \theta/\sqrt{n})/\log(2)$. The corresponding bound $h \geq \sqrt{2} - 2^{-j_1}$ of Assumption~\ref{hyp:decorr} for the different values of $\theta$ are $\{0.061, 0.122, 0.183, 0.244\}$ respectively.
The training dataset was drawn on a dyadic grid at scale $J = 3$ (with $n = 81$ observations), whereas the number of input observations for the prediction is equal to $d = 10$.
The results for the two spatial models are displayed in Table \ref{tab:TPL_Gauss_J3_theta}: the mean and the standard deviation (std) of the AMSE, over the $100$ replications, are given for different values of $\theta$.

\begin{table}[h]
  \centering
  \caption{Mean and standard deviation of the AMSE over $100$ independent simulations of a Gaussian process with truncated power law (left) and Gaussian (right) covariance functions for theoretical and empirical Kriging with different values of $\theta$ (with $J = 3$, $N = 1681$ and $d = 10$).}
  \vskip 0.15in
    \centering
    \begin{tabular}{ | c || c | c || c | c | }
        \hline
        \small{\texttt{TPL}}
        &
        \multicolumn{2}{| c ||}{ \small{\texttt{Theoretical}} } & \multicolumn{2}{| c |}{ \small{\texttt{Empirical}} }
        \\ \hline \hline
        $\theta$ & mean & std & mean & std
        \\ \hline
        $2.5$ & 0.961 & 0.086 & 0.971 & 0.088
        \\ \hline
        $5$ & 0.911 & 0.145 & 0.930 & 0.159
        \\ \hline
        $7.5$ & 0.850 & 0.218 & 0.864 & 0.215
        \\ \hline
        $10$ & 0.800 & 0.249 & 0.839 & 0.257
        \\ \hline
    \end{tabular}
  \hspace{0.18cm}
    \centering
    \begin{tabular}{ | c || c | c || c | c | }
        \hline
        \small{\texttt{GAUSS}}
        &
        \multicolumn{2}{| c ||}{ \small{\texttt{Theoretical}} } & \multicolumn{2}{| c |}{ \small{\texttt{Empirical}} }
        \\ \hline \hline
        $\theta$ & mean & std & mean & std
        \\ \hline
        $2.5$ & 0.891 & 0.196 & 0.899 & 0.208
        \\ \hline
        $5$ & 0.635 & 0.269 & 0.686 & 0.340
        \\ \hline
        $7.5$ & 0.421 & 0.257 & 0.703 & 1.498
        \\ \hline
        $10$ & 0.247 & 0.202 & 0.536 & 1.048
        \\ \hline
    \end{tabular}
  \label{tab:TPL_Gauss_J3_theta}
\end{table}

For the truncated power law covariance function, observe that the mean of the AMSE decreases slowly with the correlation length $\theta$ (and, by definition, with $j_1$), whereas the standard deviation increases slowly, for both the theoretical and empirical Kriging. Furthermore, in keeping with our theoretical results, the difference between the two  AMSE (\textit{i.e.} the excess of pointwise risk \eqref{eq:excess_IQR}) is small for all values of $\theta$. The same observations can be made for the Gaussian covariance function, where the standard deviation is larger than for the other covariance model, the mean is slightly smaller however.

\begin{figure}[h]
    \centering
    \subfigure[$\theta = 2.5$]{\includegraphics[width=60mm, trim=0cm 0cm 1.6cm 0cm]{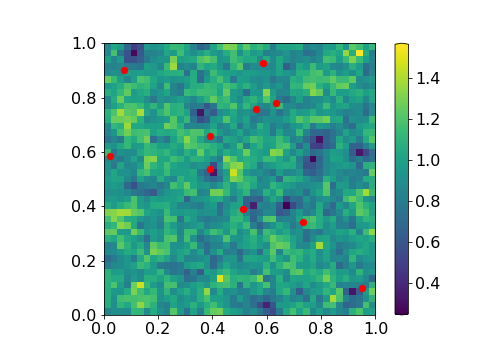}
        \label{fig:TPL_J3_Emp_theta_2.5}}
    \subfigure[$\theta = 5$]{\includegraphics[width=60mm, trim=0.8cm 0cm 0.8cm 0cm]{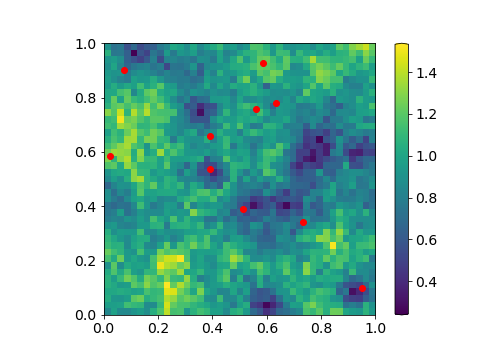}
        \label{fig:TPL_J3_Emp_theta_5.0}}
    \\
    \subfigure[$\theta = 7.5$]{\includegraphics[width=60mm, trim=0cm 0cm 1.6cm 0.8cm]{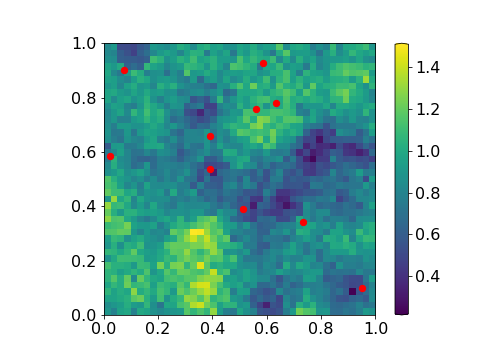}
        \label{fig:TPL_J3_Emp_theta_7.5}}
    \subfigure[$\theta = 10$]{\includegraphics[width=60mm, trim=0.8cm 0cm 0.8cm 0.8cm]{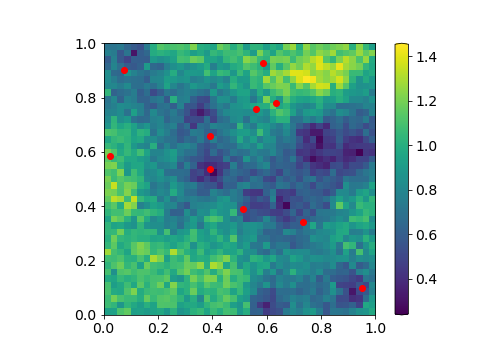}
        \label{fig:TPL_J3_Emp_theta_10.0}}
    \caption{MSE maps over $100$ realizations of a Gaussian process with truncated power law covariance function for the empirical Kriging predictor with different values of $\theta$ ($J = 3$, $N = 1681$, and $d = 10$).}
    \label{fig:TPL_J3_Emp}
\end{figure}

\begin{figure}[h]
    \centering
    \subfigure[$\theta = 2.5$]{\includegraphics[width=60mm, trim=0cm 0cm 1.6cm 0cm]{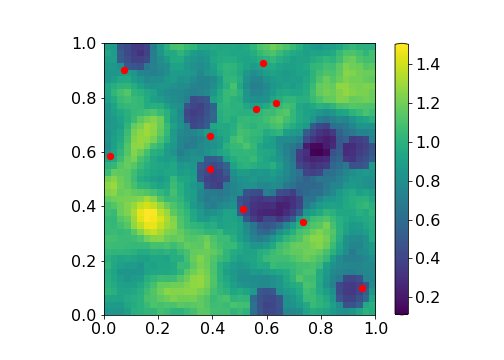}
        \label{fig:Gauss_J3_Emp_theta_2.5}}
    \subfigure[$\theta = 5$]{\includegraphics[width=60mm, trim=0.8cm 0cm 0.8cm 0cm]{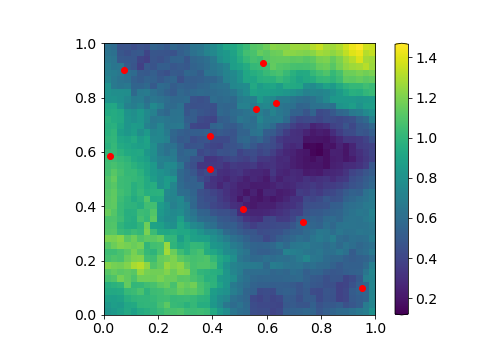}
        \label{fig:Gauss_J3_Emp_theta_5.0}}
    \\
    \subfigure[$\theta = 7.5$]{\includegraphics[width=60mm, trim=0cm 0cm 1.6cm 0.8cm]{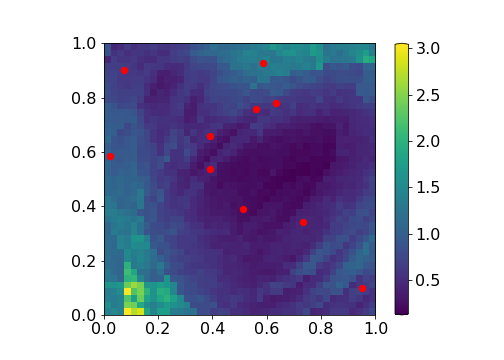}
        \label{fig:Gauss_J3_Emp_theta_7.5}}
    \subfigure[$\theta = 10$]{\includegraphics[width=60mm, trim=0.8cm 0cm 0.8cm 0.8cm]{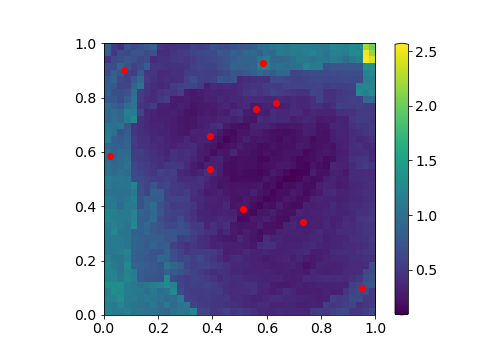}
        \label{fig:Gauss_J3_Emp_theta_10.0}}
    \caption{MSE maps of over $100$ realizations of a Gaussian process with Gaussian covariance function for the empirical Kriging predictor with different values of $\theta$ ($J = 3$, $N = 1681$, and $d = 10$).}
    \label{fig:Gauss_J3_Emp}
\end{figure}

To better understand the spatial structure of these errors, the maps of the mean squared errors for the empirical Kriging predictors are depicted in Figures \ref{fig:TPL_J3_Emp} and \ref{fig:Gauss_J3_Emp}. The observed sites $(s_1, \cdots, s_d)$ are represented in red. Observe first that, as underlined in Remark~\ref{rem:exact_interpolator}, Kriging is an exact interpolator (the error is null at the observation sites $(s_1, \cdots, s_d)$). For the truncated power law model, the complete maps of the MSE for different values of $\theta$ are comparable, with the same order of magnitude for the errors and similar location of the smallest errors. In the case of the Gaussian model, the results in Fig. \ref{fig:Gauss_J3_Emp_theta_7.5} and \ref{fig:Gauss_J3_Emp_theta_10.0} exhibit some border effects, where a higher AMSE can be observed near the boundaries of the spatial domain. This difference can be easily explained. For the truncated power law model, there is an absence of correlation between locations that are distant enough and, as can be seen in Figure \ref{fig:TPL_cov_est}, the threshold is reached quickly. In contrast, for the Gaussian model, the covariance function vanishes only for large distances, especially for the empirical version, which fails to capture the correlation length value as noticed in the previous subsection.
The predictive performance of the empirical Kriging method has been also evaluated for a larger number $n$ of training observations, \textit{i.e.} for a denser dyadic grid of scale $J = 4$ ($n = 289$). Table \ref{tab:TPL_Gauss_J4_theta} (left) shows that the results for the truncated power law model are comparable to those in the case $J = 3$. This is also the case for the Gaussian model except that, when the covariance function is unknown, the mean and standard deviation of the AMSE become larger for $\theta = 10$ (see Table \ref{tab:TPL_Gauss_J4_theta}, right).
%, as the standard deviation, whereas for small values of $\theta$, the results are quite similar to the case $J = 3$.

\begin{table}[h]
  \centering
  \caption{Mean and standard deviation of the AMSE over $100$ independent simulations of a Gaussian process with truncated power law (left) and Gaussian (right) covariance functions for theoretical and empirical Kriging with different values of $\theta$ (with $J = 4$, $N = 2401$ and $d = 10$).}
  \vskip 0.15in
    \centering
    \begin{tabular}{ | c || c | c || c | c | }
        \hline
        \small{\texttt{TPL}}
        &
        \multicolumn{2}{| c ||}{ \small{\texttt{Theoretical}} } & \multicolumn{2}{| c |}{ \small{\texttt{Empirical}} }
        \\ \hline \hline
        $\theta$ & mean & std & mean & std
        \\ \hline
        $2.5$ & 0.975 & 0.079 & 0.976 & 0.079
        \\ \hline
        $5$ & 0.927 & 0.131 & 0.928 & 0.131
        \\ \hline
        $7.5$ & 0.861 & 0.156 & 0.874 & 0.150
        \\ \hline
        $10$ & 0.815 & 0.210 & 0.841 & 0.220
        \\ \hline
    \end{tabular}
  \hspace{0.18cm}
    \centering
    \begin{tabular}{ | c || c | c || c | c | }
        \hline
        \small{\texttt{GAUSS}}
        &
        \multicolumn{2}{| c ||}{ \small{\texttt{Theoretical}} } & \multicolumn{2}{| c |}{ \small{\texttt{Empirical}} }
        \\ \hline \hline
        $\theta$ & mean & std & mean & std
        \\ \hline
        $2.5$ & 0.913 & 0.167 & 0.890 & 0.167
        \\ \hline
        $5$ & 0.708 & 0.272 & 0.745 & 0.306
        \\ \hline
        $7.5$ & 0.529 & 0.263 & 0.582 & 0.288
        \\ \hline
        $10$ & 0.312 & 0.188 & 1.675 & 12.085
        \\ \hline
    \end{tabular}
  \label{tab:TPL_Gauss_J4_theta}
    \vskip 0.1in
\end{table}

The numerical results for the Gaussian covariance function, which does not satisfy Assumption~\ref{hyp:decorr} but quickly vanishes, suggest that the validity framework of the empirical Kriging method can be extended, as discussed at length in the Appendix.

\subsection{Numerical Results for Real Data}\label{subsec:exp_num_RD}

For the sake of completeness, prediction via simple empirical Kriging has also been examined on real spatial observations, on datasets available on the web portal DRIAS (see \url{https://drias-prod.meteo.fr/okapi/accueil/okapiWebDrias/index.jsp}), which provides the mean daily temperature in France (in Kelvin), observed on a regular grid, from 1951 to 2005. The position (latitude and longitude) of the grid points are in decimal degrees (WGS84). The datasets that are used in this study are square grids of a total of $2401$ point locations (referred to as the spatial domain $S$), during the three months of summer (June, July, and August, for a total of $92$ days) of the years 2004 and 2005 (refer to Figure \ref{fig:RD_data} for the sampled square grid). The square grid is obtained directly on the web portal by selecting the desired points on the grid, in such a way that the complete grid is of the wanted dimension $49 \times 49$. Notice that the temporal dimension of the data is ignored here, it is assumed that the daily observations are independent from one year to the next. Under this simplifying hypothesis, we consider that a number of realizations of the phenomenon are available, large enough for computing significant AMSE's. The dataset of the year 2004 is used as training samples: for each day, $n = 289$ sites are observed, forming a dyadic grid at scale $J = 4$. These observations are used in order to estimate the (supposedly isotropic) covariance function of the random field by means of the nonparametric statistics \eqref{eq:emp_cov} studied in subsection \ref{subsec:nonpar_cov_est}. Then, respectively on the same days of the year 2005, $d = 10$ sites are randomly selected over the spatial domain. The goal is to predict the value $X_s$ taken by $X$ at any unobserved site $s \in S$ (\textit{i.e.} predict the mean temperature on the same day for all unobserved locations) based on the $d$ input observations and the estimated covariance function. The experiment has been performed $92$ times (training samples from 2004 and data from 2005 for the prediction step).

We point out that, since the mean is unknown, we opted for the Ordinary Kriging variant: using the semi-variogram function rather than the covariance (see \citep[Section 3]{Chiles.etal1999} and \citep[subsection 3.2]{Cressie_1993} for a presentation of the estimator and of the method), Ordinary Kriging allows us to perform the prediction without any information about the mean of the random process (see subsection \ref{subsec:nonpar_cov_est}). The theoretical results of Section \ref{sec:main} easily extends to the case of Ordinary Kriging: indeed, Propositions~\ref{prop:Gaetan_Guyon} and \ref{prop:CI_var_lag} are based on the semi-variogram estimation and the following results up to Theorem~\ref{thm:EGRisk} can be straightforwardly extended to the Ordinary Kriging predictor.

\begin{figure}[h]
\centering
    \includegraphics[scale = 0.35]{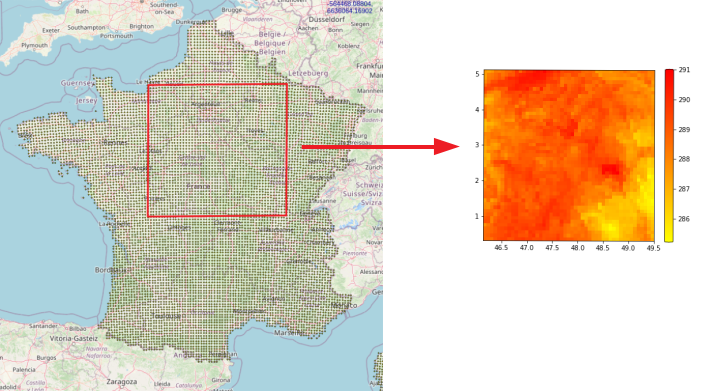}
    \caption{Real Dataset DRIAS: France map with the sampled square grid (left); color map of the sampled square grid, on June 2nd, 2005 (right).}
    \label{fig:RD_data}
\end{figure}

For the parametric Kriging method, the truncated power law model has been selected, among several covariance models. Here we set $\theta = j_1$ (the parameter from Assumption \ref{hyp:decorr}): we fixed the value of $\theta$ high enough, so that a large number of correlations are taken into account. This means that the covariances for almost all lags $h \in \mathcal{H}_n$ are involved in the computation of the parametric Kriging predictor. Yet, it is not surprising, in the case of temperature data, to obtain a better accuracy using almost all covariances, since the correlation is strong between all pairs of locations. The results are displayed in Table \ref{tab:RD_J3_J4}. For the nonparametric Kriging method, the parameter $\nu$ introduced in Lemma \ref{lem:nh_n} is set to $0.35$ in order to use most of the 
observed distances $h \in \mathcal{H}_n$ for the covariance function estimation. Notice that the mean error, and the standard deviation as well, are low (see Table \ref{tab:RD_J3_J4}). The results are encouraging and corroborate the theoretical guarantees established.

\begin{table}[h]
  \centering
  \caption{Mean and standard deviation (std) of all AMSE for parametric and nonparametric Kriging on Real Data (with $J = 4$, $N = 2401$, and $d = 10$).}
  \vskip 0.15in
  \hspace{-0.45cm}
    \begin{tabular}{ | c || c | c || c | c | }
            \hline
            &
            \multicolumn{2}{| c ||}{ \small{\texttt{Parametric}} } & \multicolumn{2}{| c |}{ \small{\texttt{Nonparametric}} }
            \\ \hline \hline
             & mean & std & mean & std
            \\ \hline
            $J = 4$ & 2.581 & 0.564 & 2.944 & 1.931
            \\ \hline
        \end{tabular}
  \label{tab:RD_J3_J4}
    \vskip 0.1in
    \end{table}

Though it is beyond the scope of this article, the statistical modelling and predictive analysis of such real data could be naturally refined in many ways, taking into account anisotropy and/or the temporal structure in particular. However, the only goal pursued here is to show that a simplistic application of the nonparametric empirical Kriging prediction method may perform well and can be competitive compared to a more rigid method based on a preliminary parametric modelling of the covariance structure.

\section{Conclusion}\label{sec:concl}
In this article, we have proposed a statistical learning view on simple Kriging, the flagship problem in Geostatistics, usually addressed from a parametric and asymptotic perspective in the literature. In spite of the similarities shared by simple Kriging with kernel ridge regression, the major difficulty in analyzing this predictive problem lies in the complex dependence structure generally exhibited by spatial data. As explained at length, an empirical version of the optimal simple Kriging rule (minimizing the MSE integrated over the spatial domain) can be constructed by means of a nonparametric estimator of the covariance function in a plug-in fashion. It is also shown that the predictive rule thus built can be viewed as a minimizer of the empirical counterpart of the risk, based on the covariance function estimator.  We have developed a novel theoretical framework offering non-asymptotic guarantees for empirical simple Kriging rules in the form of non-asymptotic bounds for the integrated MSE in a classic \textit{in-fill} setup, stipulating that the $n\geq 1$ sites at which the stationary isotropic Gaussian field under study is observed form a denser and denser regular grid. The learning rate bounds are of order $O_{\mathbb{P}}(1/\sqrt{n})$, these are the first results of this nature to the best of our knowledge. Alternative frameworks, relaxing in particular certain assumptions involved in the present analysis, may be investigated (refer to the discussion in the Appendix section for further details). Possible avenues for future work include the consideration of spatial data observed on irregular grids for the training step and that of the out-fill setup for investigating the generalization capacity of Kriging predictors over a wider and wider spatial domain.

\section*{Statement and Declarations}

\noindent {\bf Conflict of Interest Statement.} The authors declare no conflict of interest.\\

\noindent {\bf Code and Data Availability.} The codes and datasets used for this study are publicly available at \url{https://github.com/EmiliaSiv/Simple-Kriging-Code}.

\section*{Acknowledgements}

This work was supported by the Télécom Paris research chair on Data Science and Artificial Intelligence for Digitalized Industry and Services (DSAIDIS). The authors would like to thank Jean-Rémy Conti (LTCI, Télécom Paris, Institut Polytechnique de Paris), who laid the foundations of this work.


\bibliography{sn-bibliography}
\bibliographystyle{style}


\newpage

\appendix

The Appendices are organized as follows:
\vspace{0.2cm}
\begin{itemize}
    \item in Appendix A, the auxiliary results used in the technical proofs of Appendix B are established.
    \vspace{0.2cm}
    \item in Appendix B, the proofs of the main results of the paper are detailed at length.
    \vspace{0.2cm}
    \item in Appendix C, additional numerical experiments are presented.
    \vspace{0.2cm}
    \item in Appendix D, the role of the technical assumptions involved in the theoretical analysis is extensively discussed, as well as possible avenues in order to extend the results proved in this article to a more general framework.
\end{itemize}

\section{Auxiliary Results}\label{proofs:aux_res}
In this section, we present the auxiliary results used in the technical proofs of the main results of the paper. Here and throughout, given a Hermitian matrix $A$ of size $n \times n$, denote $\xi_n(A) \leq \cdots \leq \xi_1(A)$ its eigenvalues (arranged in decreasing order) and $Rank(A)$ its rank.

\subsection{Auxiliary Result for the Proof of Proposition~\ref{prop:Gaetan_Guyon}}

The following lemma is used in the proof of Proposition~\ref{prop:Gaetan_Guyon} and allows us to generalize the proof for the distributions of both the semi-variogram and the variance estimators.

\begin{lemma}\label{lem:prop1}
Let $X \sim \mathcal{N}(0, \Sigma)$ be a centered Gaussian random field with positive definite covariance function and $R$ a symmetric and positive semi-definite matrix of size $n \times n$, such that $Rank(R) \leq r$ (where $r$ is a strictly positive integer). Then, we have:
$$X^{\top} R X \sim \sum\limits_{i=1}^r \xi_i(R\Sigma) \chi_i^2,$$
where the $\chi_{i}^2$'s are independent $\chi^2$ random variables with one degree of freedom and the $\xi_i(R\Sigma)$'s are the $r$ (strictly positive) eigenvalues of $R\Sigma$.
\end{lemma}

\begin{proof}
Thanks to the assumptions, the covariance matrix $\Sigma$ is symmetric and positive definite. Thus, using a well-known result of matrix algebra (see \textit{e.g.} \citep[Chapter 21]{harville1998matrix}), define the square root of $\Sigma$, a symmetric and positive definite matrix by $\Sigma^{1/2}$ such that $\Sigma = \Sigma^{1/2} \, \Sigma^{1/2}$. The square root matrix is invertible and its inverse $\Sigma^{-1/2}$ is symmetric and positive definite. Let $Y = \Sigma^{-1/2} X \sim \mathcal{N}(0, \bI_n)$, where $\bI_n$ is the $n \times n$ identity matrix. Then
\begin{align*}
    X^{\top} R X & = X^{\top} (\Sigma^{1/2} \Sigma^{-1/2})^{\top} R (\Sigma^{1/2} \Sigma^{-1/2}) X\\ & = (\Sigma^{-1/2} X)^{\top} (\Sigma^{1/2})^{\top} R \Sigma^{1/2} (\Sigma^{-1/2} X) = Y^{\top} T Y,
\end{align*}
where $T = (\Sigma^{1/2})^{\top} R \Sigma^{1/2}$. Since $R$ is symmetric and positive semi-definite, $T$ is also symmetric and positive semi-definite. Furthermore,
\begin{equation*}
    R \Sigma = \Sigma^{-1/2} \Sigma^{1/2} R \Sigma^{1/2} \Sigma^{1/2} = \Sigma^{-1/2} T \Sigma^{1/2},
\end{equation*}
which implies that the matrices $T$ and $R \Sigma$ are similar and have the same eigenvalues (see \textit{e.g.} \citep[Chapter 21]{harville1998matrix}). Thanks to the spectral decomposition, there exists an orthogonal matrix $P$ and a diagonal matrix $D = Diag((t_i)_{i \in \{1, \cdots, n\}})$ of the eigenvalues of $T$, such that $T = P^{\top} D P$. The eigenvalues $t_i$'s are positive as $T$ is positive semi-definite. Recall that the rank of the matrix $R$ is upper bounded by a positive integer $r$ and this implies that the rank of $R \Sigma$ is upper bounded by $r$ too. Thus, denote $\xi_i(R \Sigma)$ the $r$ positive eigenvalues of $R \Sigma$. Furthermore
\begin{equation*}
    X^{\top} R X = Y^{\top} P^{\top} D P Y = Z^{\top} D Z = \sum\limits_{i=1}^n \sum\limits_{j=1}^n Z_i Z_j D_{i, j} = \sum\limits_{i=1}^r Z_i^2 \xi_i(R \Sigma),
\end{equation*}
where $Z_i \sim \mathcal{N}(0, \bI_n)$ are independent Gaussian random variables. Finally, notice that $Z_i^2$ are $r$ $\chi^2$ random variables with one degree of freedom, which concludes the proof.
\end{proof}

\subsection{Bounds on Largest Eigenvalues}\label{proofs_aux_res_prop2}

In this subsection, we present several results useful for the Proposition~\ref{prop:CI_var_lag}. These results will be used for bounding the eigenvalues that appear in the weighted sums of $\chi^2$ random variables in the distributions of both the semi-variogram and the variance estimators.

\vspace{0.3cm}

\noindent \textbf{Eigenvalues of the Product of Positive Semi-definite Hermitian Matrices.} We first present inequalities for the eigenvalues of the product of positive semi-definite Hermitian matrices (the proof can be found in \textit{e.g.} \citep{wang1992some, xi2019inequalities}).

\begin{proposition}\label{aux_res:eig_val}
    Let $A$ and $B$ two positive semi-definite $n \times n$ Hermitian matrices. Denote $\xi_n(A) \leq \cdots \leq \xi_1(A)$ and $\xi_n(B) \leq \cdots \leq \xi_1(B)$ the eigenvalues of $A$ and $B$ respectively. Let $k > 1$. Then, for $1 \leq i_1 < \cdots < i_k \leq n$,
    \begin{equation}
        \sum_{t=1}^k \xi_{i_t}(A) \xi_{n-t+1}(B) \leq \sum_{t=1}^k \xi_{i_t}(AB) \leq \sum_{t=1}^k \xi_{i_t}(A) \xi_t(B).
    \end{equation}
\end{proposition}

\vspace{0.3cm}

\noindent \textbf{Bounds on the Largest Eigenvalue of Laplacian Matrices.} We give a bound on the largest eigenvalue $\xi_1(L)$ of the Laplacian matrix $L$ of a graph in terms of the maximum degree of its vertices.

\begin{proposition}\label{aux_res:prop_eig_Lap}
Let $G = (V, E)$ a graph with maximum degree $d_{\max} = \max\limits_{v \in V} deg(v)$ and $L$ the Laplacian matrix of $G$. Then $\xi_1(L) \leq 2 d_{\max}.$
\end{proposition}

The proof essentially relies on the following propositions.

\begin{proposition}{\textbf{(\citep[Lemma 3.4.1]{spielmanLaplacian})}}\label{prop:aux_res_d_max}
Let $G = (V, E)$ a graph with maximum degree $d_{\max} = \max\limits_{v \in V} deg(v)$. Let $D$ the degree matrix of $G$ (defined as the diagonal matrix with entries $D_{ii} = deg(v_i), \, \forall i \in \{1, \dots, \vert V \vert\}$) and $A$ the adjacency matrix of $G$ (with entries $A_{ij} = \mathbb{I}\{(v_i, v_j) \in E\}, \, \forall i, \, j \in \{1, \dots, \vert V \vert\}$). Then $\xi_1(D) = d_{\max}$ and
$\xi_1(A) \leq d_{\max}$.
\end{proposition}

The next proposition is a well-known result, often called Weyl's inequality (see \textit{e.g.} \citep[Theorem 4.3.1]{horn2012matrix} for a proof of the result).

\begin{proposition}\label{res_aux:propAminusB}
Let $A$ and $B$ be two Hermitian matrices. Then 
\begin{equation*}
    \xi_1(A - B) \leq \xi_1(A) + \xi_1(B).
\end{equation*}
\end{proposition}

Combining these results with the definition of the Laplacian matrix as $L = D - A$, where $D$ is the degree matrix and $A$ the adjacency matrix of a graph, we have the wanted result in Proposition \ref{aux_res:prop_eig_Lap}.

\vspace{0.3cm}

\noindent \textbf{Bounds on the Largest Eigenvalue of the Covariance Matrix of a Stationary Random Field.} Now, we present a result on the bounded eigenvalues of the covariance matrix for a stationary random field. This result derives from the application of Bochner's Theorem (see \textit{e.g.} \citep[Chapter 2]{stein1999interpolation}), combined with the assumed bounds on the spectral density.

\begin{lemma}\label{aux_res:lem_eig_cov}
    Let $(X_s)_{s \in \mathbb{R}^2}$ be a stationary (in the second-order sense) process with spectral density $\Phi$ and covariance matrix $\Sigma$. Suppose that Assumption \ref{hyp:bounded} is fulfilled. Thus, the eigenvalues $\xi(\Sigma)$ of the covariance matrix are bounded as follows:
    $$\exists c > 0, c' > 0, \, c m \leq \xi(\Sigma) \leq c' M,$$
    where $m$ and $M$ are given  in Assumption~\ref{hyp:bounded}.
\end{lemma}

\subsection{Upper Bound on the Variances of the Semi-Variogram and Variance Estimators}\label{subsec:aux_res_cressie}

We present a well-known result from \citep{Cressie_1993} for the variance of the semi-variogram estimator.

\begin{proposition}{\textbf{(\citep[Section 2.4]{Cressie_1993})}}\label{aux_res:prop_cressie}
    Variances of the semi-variogram estimator $\widehat{\gamma}(h)$ for a fixed $h$ are $O(1/n)$.
\end{proposition}

\begin{proof}
Notice that, under the Gaussian and the intrinsic assumptions, we have $Var((X_{s+h} - X_s)^2) = 2 (2 \gamma(h))^2$. Then
\begin{align*}
    Var(\widehat{\gamma}(h)) & = \left(\frac{1}{2 n_h}\right)^2 \sum\limits_{(s_i, s_j) \in N(h)} Var((X_{s_i} - X_{s_j})^2)\\ & = \left(\frac{1}{2 n_h}\right)^2 \sum\limits_{(s_i, s_j) \in N(h)} 2 (2 \gamma(h))^2 = \frac{2 \gamma(h)^2}{n_h},
\end{align*}
which gives the wanted result.
\end{proof}

Furthermore, it's easy to see that the variance of the covariance estimator $\widehat{c}(h)$ for a fixed $h$ is $O(1/n)$. This implies, thanks to the link between the semi-variogram estimator and the covariance estimator in Equation \eqref{eq:estvar2}, that the variance of $\widehat{c}_h(0)$ is also $O(1/n)$.
Thus, we obtain
\begin{equation}\label{eq:sigma1_n}
    \exists c > 0, \, Var\left(\widehat{\gamma}(h)\right) \leq \frac{c}{n}
\end{equation}
and
\begin{equation}\label{eq:sigma2_n}
    \exists c' > 0, \, Var\left(\widehat{c}_h(0)\right) \leq \frac{c'}{n}.
\end{equation}

\subsection{Gamma Random Variables}

To avoid any ambiguity, we give the definition of a Gamma random variable and a proposition on the link between Gamma and Chi-Square random variables.

\begin{definition}
The density function of $Z \sim \Gamma(\alpha, \beta)$ a Gamma random variable with shape parameter $\alpha \in \mathbb{R}_+$ and rate parameter $\beta \in \mathbb{R}_+$ is
\begin{equation*}
    f_Z(z) = \frac{\beta^\alpha}{\Gamma(\alpha)}z^{\alpha-1}e^{-z\beta}, \, \forall z > 0,
\end{equation*}
where $\Gamma$ is the Gamma function. The mean of a Gamma random variable is: $\mathbb{E}[Z] = \frac{\alpha}{\beta}.$
\end{definition}

\begin{proposition}\label{aux_res:chi_gamma}
    If $Z \sim \chi^2_k$ and $c > 0$ then $cZ \sim \Gamma\left(\frac{k}{2}, \frac{1}{2c}\right)$.
\end{proposition}

\subsection{Extension of Tail Bound Inequalities for the Semi-Variogram and Variance Estimators}\label{subsec:aux_cor}

As a first preliminary result, we give an upper bound on the total number of distinct observable distances $h \in \mathcal{H}_n$ on the regular grid of size $n$. The idea is the following. Let $n_x$ be the number of columns/rows, such that $n = n_x \times n_x$. Then, we have $n_x^2$ possible combinations between all pairs of points location, at which we may withdraw $n_x$ values (the locations that are on the diagonal of the grid). Since the process is assumed to be isotropic, we may divide this value by $2$. Finally, we add $n_x$ for the diagonal and have the following result: $\vert\mathcal{H}_n\vert = \frac{n_x(n_x+1)}{2}$. Then, we obtain the following upper bound, since $n = n_x^2$:
\begin{equation}
    \vert\mathcal{H}_n\vert \leq n.
    \label{eq:Hn_bound}
\end{equation}

Then, we present a corollary to Proposition \ref{prop:CI_var_lag}, that extends the results on tail bound inequalities for the semi-variogram estimator and the variance estimator.

\begin{corollary}\label{cor:aux}
    Suppose that Assumptions \ref{hyp:simple}--\ref{hyp:bounded} are fulfilled. Let $k > 0$,
    \begin{equation*}
        \mathbb{P}\left(\max\limits_{h \in \mathcal{H}_n} \left\vert\widehat{\gamma}(h)-\gamma(h)\right\vert \geq k\right) \leq  2 n e^{-C'_1 n k^2}, \, \text{ whenever } k \leq \frac{C_1}{C'_1},
    \end{equation*}
    and
    \begin{equation*}
        \mathbb{P}\left(\max\limits_{h \in \mathcal{H}_n} \left\vert\widehat{c}_h(0)-c(0)\right\vert \geq k\right) \leq 2 n e^{-C'_2 n k^2}, \, \text{ whenever } k \leq \frac{C_2}{C'_2},
    \end{equation*}
    where $C_i$ and $C'_i$, $i\in\{1,\; 2\}$, are positive constants depending on $j_1$, $m$ and $M$ solely (given in Proposition~\ref{prop:CI_var_lag}).
\end{corollary}

\begin{proof}
We give the proof for the semi-variogram estimator. The proof for the variance estimator follows the same steps, replacing the constants $C_1$ and $C'_1$ by $C_2$ and $C'_2$.\\
Notice that
$$\forall k > 0, \, \, \mathbb{P}\left(\max\limits_{h \in \mathcal{H}_n} \left\vert\widehat{\gamma}(h)-\gamma(h)\right\vert \geq k\right) = \mathbb{P}\left(\bigcup\limits_{h \in \mathcal{H}_n} \left\{\left\vert\widehat{\gamma}(h)-\gamma(h)\right\vert \geq k\right\}\right).$$
Thanks to the Union Bound (or Boole's Inequality) and then applying the result in Proposition \ref{prop:CI_var_lag}, we have
\begin{align*}
    \mathbb{P}\left(\bigcup\limits_{h \in \mathcal{H}_n} \left\{\left\vert\widehat{\gamma}(h)-\gamma(h)\right\vert \geq k\right\}\right) & \leq \sum\limits_{h \in \mathcal{H}_n} \mathbb{P}\left(\left\vert\widehat{\gamma}(h)-\gamma(h)\right\vert \geq k\right)\\ & \leq \sum\limits_{h \in \mathcal{H}_n} \left(e^{-C_1 n k} + e^{-C'_1 n k^2}\right)\\ & \leq \vert\mathcal{H}_n\vert 2 \max\left\{e^{-C_1 n k}, e^{-C'_1 n k^2}\right\}.
\end{align*}
Then, if we take $k \leq \frac{C_1}{C'_1}$, the maximum is obtained for $e^{-C'_1 n k^2}$ and, combining this with the result on the cardinality of $\mathcal{H}_n$ in \eqref{eq:Hn_bound}, one gets
$$\vert\mathcal{H}_n\vert 2 \max\left\{e^{-C_1 n k}, e^{-C'_1 n k^2}\right\} \leq 2 n e^{-C'_1 n k^2},$$
which concludes the proof.
\end{proof}

\subsection{Auxiliary Results for the Proof of Proposition \ref{prop:CI_Sigma_vec}}

We present an auxiliary result from \citep[Theorem 4.1]{wedin1973perturbation} (see also \citep[Lemma 5]{staerman2021affine}) used in the proof of Proposition \ref{prop:CI_Sigma_vec}.

\begin{theorem}\label{thm_aux:wedin}
    Let $A$ and $B$ be two invertible matrices of size $d \times d$. Then it holds:
    \begin{equation}
        \vert\vert\vert A^{-1} - B^{-1} \vert\vert\vert \leq \vert\vert\vert A^{-1} \vert\vert\vert \, \vert\vert\vert B^{-1} \vert\vert\vert \, \vert\vert\vert A - B \vert\vert\vert.
    \end{equation}
\end{theorem}

\section{Technical Proofs}\label{sec:technical_proofs}

\subsection{Proof of Lemma \ref{lem:nh_n}}
For any strictly positive $h \in \mathcal{H}_n$, $\exists (d,q) \in \mathbb{N}^* \times \mathbb{N}^*$ s.t. $h = \left(\sqrt{d^2 + q^2}\right) 2^{-J}$. Since the random field $X$ is isotropic (Assumption \ref{hyp:stat_iso}), we have the following bound
$n_h \geq 4 (n_x - d) (n_x - q),$
where $n = n_x^2$ ($n_x$ represents the number of columns/rows of the square grid). Then, let $h' = h 2^J = \sqrt{d^2 + q^2} > 0$. Since $d \leq h'$ and $q \leq h'$, we have $n_h \geq 4 (n_x - h')^2 \geq 4 (n_x^2 - 2 n_x h') = 4 (n - 2 n_x h')$. Furthermore, under Assumption \ref{hyp:decorr}, we are only interested in distances s.t. $h < \sqrt{2} - 2^{-j_1}$. This implies that, for $n$ large enough (condition given by $n>(\sqrt{2}-2^{-j_1})^2$), we finally obtain:
\begin{equation}
    \forall h \in \mathcal{H}_n, \quad n_h >\nu n,
\end{equation}
where $\nu$ is a positive constant depending on $j_1$ only.

\subsection{Proof of Proposition \ref{prop:Gaetan_Guyon}}
The proof of Proposition \ref{prop:Gaetan_Guyon} essentially relies on Lemma~\ref{lem:prop1}, simultaneously applied to the estimators $\widehat{\gamma}(h)$ and $\widehat{c}_h(0)$. Refer to Appendix \ref{proofs:aux_res} for its presentation and proof. We first study the semi-variogram estimator $\widehat{\gamma}(h)$ and then the variance estimator $\widehat{c}_h(0)$.
The goal is to define a matrix $L(n, h)$ such that $$\widehat{\gamma}(h) = \bX(\sigma_n)^{\top} \frac{1}{n_h} L(n, h) \bX(\sigma_n).$$
Notice that: $\forall h > 0$,
\begin{align*}
     \widehat{\gamma}(h) & = \frac{1}{2 n_h} \sum\limits_{i=1}^n \sum\limits_{j=1}^n \left(X_{\sigma_i}^2 + X_{\sigma_j}^2 - 2 X_{\sigma_i} X_{\sigma_j}\right) \mathbb{I}\{(\sigma_i, \sigma_j) \in N(h)\}\\ & = \frac{1}{n_h} \sum\limits_{i=1}^n X_{\sigma_i}^2 n_h(i) - \frac{1}{n_h} \sum\limits_{i=1}^n \sum\limits_{j=1}^n X_{\sigma_i} X_{\sigma_j} \mathbb{I}\{(\sigma_i, \sigma_j) \in N(h)\}\\ & = \frac{1}{n_h} \sum\limits_{i=1}^n X_{\sigma_i}^2 n_h(i) - \frac{1}{n_h} \sum\limits_{i=1}^n X_{\sigma_i}^2 \mathbb{I}\{(\sigma_i, \sigma_i) \in N(h)\}\\ & \qquad - \frac{1}{n_h} \sum\limits_{i=1}^n \sum\limits_{j=1, j \ne i}^n X_{\sigma_i} X_{\sigma_j} \mathbb{I}\{(\sigma_i, \sigma_j) \in N(h)\}\\ & = \frac{1}{n_h} \left(\sum\limits_{i=1}^n X_{\sigma_i}^2 n_h(i) - \sum\limits_{i=1}^n \sum\limits_{j=1, j \ne i}^n X_{\sigma_i} X_{\sigma_j} \mathbb{I}\{(\sigma_i, \sigma_j) \in N(h)\}\right),
\end{align*}
where $\mathbb{I}\{(\sigma_i, \sigma_i) \in N(h)\} = 0$ (since $h > 0$). Then, let $L(n, h)$ the matrix with entries $L_{i, j}(n, h) = - \mathbb{I}\{(\sigma_i, \sigma_j) \in N(h)\}$ if $i \ne j$ and $L_{i, i}(n,h) = n_h(i)$.

\begin{remark}\label{rem:graph_Gh}
    For a fixed $h \in \mathcal{H}_n$, define $G_h = (V_h, E_h)$ the graph described by the regular grid, where the set of vertices $V_h$ is the set of the $n$ observations' locations and $E_h$ is the set of the $n_h$ edges that are defined by the pairs of locations that are at distance $h$. Then, $L(n, h)$ is the Laplacian matrix of $G_h$, equal to $D(n, h) - A(n, h)$ where $D(n, h)$ is the diagonal matrix of the degrees of the vertices of the graph and $A(n, h)$ is the adjacency matrix. Thanks to the Gershgorin Circle Theorem (see \textit{e.g.} \citep{shi2007bounds}), the Laplacian matrix is positive semi-definite, which implies that all its eigenvalues are nonnegative.
\end{remark}

Thanks to the remark above and the isotropy assumption (Assumption \ref{hyp:stat_iso}), the matrix $L(n, h)$ is symmetric and positive semi-definite. Based on \citep{Cressie_1993}, it is possible to rewrite the matrix as $L(n, h) = \frac{1}{2} Q(n, h)^{\top} Q(n, h)$, where $Q(n, h) \in \mathbb{R}^{n_h \times n}$ is a matrix whose entries are only $-1, 0$ and $1$. The idea is to let $(u_l)_{l \leq n_h}$ the elements of $N(h)$, such that $\forall l \in \llbracket 1, n_h \rrbracket, \, \exists (i, j) \in \llbracket 1, n_h \rrbracket^2, \, u_l = (u_l^1, u_l^2) = (\sigma_i, \sigma_j)$, where $\|\sigma_i - \sigma_j\| = \|h\|$. Thus, $u_l^1 = \sigma_i$ is equivalent to the fact that there exists $j \in \llbracket 1, n_h \rrbracket$ such that $(\sigma_i, \sigma_j) \in N(h)$. Furthermore, $\mathbb{I}\{(\sigma_i, \sigma_j) \in N(h)\} = 1$ is equivalent to the fact that there exists $l \in \llbracket 1, n_h \rrbracket$ such that $u_l = (\sigma_i, \sigma_j)$. Then, let
\begin{equation*}
    \forall l \in \llbracket 1, n_h \rrbracket, \, \forall i \in \{1, \dots, n\}, \, q_{li} = Q_{, i}(n, h) = \begin{cases}
                & 1 \, \text{, if } \, u_l^1 = \sigma_i \text{ and } u_l^2 \ne \sigma_i\\
                & -1 \, \text{, if } \, u_l^2 = \sigma_i \text{ and } u_l^1 \ne \sigma_i\\
                & 0 \, \text{, otherwise}
            \end{cases}
\end{equation*}

Then, $Rank(Q(n, h)) \leq n_h$, which implies $Rank(L(n, h)) \leq n_h$. Hence, it is possible to apply Lemma \ref{lem:prop1} to $\widehat{\gamma}(h)$ with the matrix $L(n, h)$ and the random field $X$ with positive definite covariance matrix $\Sigma_n$
$$\widehat{\gamma}(h) \sim \frac{1}{n_h} \sum\limits_{i=1}^{n_h} \ell_i(h) \chi_{i}^2,$$
where the $\chi_{i}^2$'s are independent $\chi^2$ random variables with one degree of freedom and the $\ell_i(h)$'s are the $n_h$ (strictly positive) eigenvalues of $L(n, h) \Sigma_n$.
For the variance estimator $\widehat{c}_h(0)$, the proof follows the same idea. Let $D(n, h)$ the diagonal matrix with entries $D_{i, i}(n, h) = n_h(i)$. We notice that $D(n, h)$ is the degree matrix of the graph $G_h$ described by the regular grid for a fixed $h \in \mathcal{H}_n$ (see Remark~\ref{rem:graph_Gh}). Then, it's clear that $$\widehat{c}_h(0) = \bX(\sigma_n)^{\top} \frac{1}{n_h} D(n, h) \bX(\sigma_n),$$
and that the matrix $D(n, h)$ is symmetric and positive semi-definite. Furthermore, one may see that $Rank(D(n, h)) \leq n_h$. Indeed, the diagonal elements are $n_h(i)$, that is, for a fixed location point $\sigma_i$, the number of points that are at distance $h$ from $\sigma_i$. The total sum of these $n_h(i)$ over all the grid locations $\sigma_i$ is equal to $n_h$. Thus, the extreme case is when all the $n_h(i)$'s are equal to $1$, which implies that exactly $n_h$ elements on the diagonal are non zero and in this case the rank of $D(n, h)$ is equal to $n_h$. It is possible to apply Lemma \ref{lem:prop1} to $\widehat{c}_h(0)$ with the matrix $D(n, h)$ and the random field $X$ with positive definite covariance matrix $\Sigma_n$
$$\widehat{c}_h(0) \sim \frac{1}{n_h} \sum\limits_{i=1}^{n_h} \rho_i(h) \chi_{i}^2,$$
where the $\chi_{i}^2$'s are independent $\chi^2$ random variables with one degree of freedom and the $\rho_i(h)$'s are the $n_h$ (strictly positive) eigenvalues of $D(n, h) \Sigma_n$.

\subsection{Proof of Proposition \ref{prop:CI_var_lag}}
Since we are interested only on which variables the constants in the final results depend on, we let, in the proof and in the corresponding preliminary results, $c$ and $c'$ as positive constants that are not always the same, but that depend on variables such as $j_1$, $m$ and $M$. The proofs of the Poisson tail bounds for the deviations for both the semi-variogram and the variance estimators are structured as follows: firstly, thanks to Proposition \ref{prop:Gaetan_Guyon}, the distributions of both estimators are known and these can be seen as the sum of independent Gamma random variables; secondly, we deduce exponential inequalities for these tail bounds (from \citep{bercu2015concentration} and \citep{Wang_Ma}); then, using the previous preliminary results, we can bound the largest eigenvalues involved in the distributions of the estimators, and finally, using the lower bound on $n_h$ given in Lemma \ref{lem:nh_n}, we conclude the proof. In the first part of the proof, we will deal with the semi-variogram estimator $\widehat{\gamma}(h)$. Let $t > 0$ and $h \in \mathcal{H}_n$,
$\mathbb{P}\left(\left\vert\widehat{\gamma}(h)-\gamma(h)\right\vert \geq t\right) = \mathbb{P}\left(\widehat{\gamma}(h) \geq \gamma(h) + t\right) + \mathbb{P}\left(\widehat{\gamma}(h) \leq \gamma(h) - t\right).$
Let $\mu_1 = \mathbb{E}[\widehat{\gamma}(h)] = \gamma(h)$ (since $\widehat{\gamma}(h)$ is unbiased). Recall that, thanks to Proposition \ref{prop:Gaetan_Guyon}:
$\widehat{\gamma}(h) \sim \frac{1}{n_h} \sum\limits_{i=1}^{n_h} \ell_i(h) \chi_{i}^2.$ Since the eigenvalues $\ell_i(h), \, \forall i \in \{1, \dots, n_h\}$ of $L(n, h) \Sigma_n$ are non negatives, from the link between Gamma and Chi-Square random variables (see Proposition~\ref{aux_res:chi_gamma}), $\frac{1}{n_h} \ell_i(h) \chi_{i}^2 \sim \Gamma\left(\frac{1}{2}, \frac{n_h}{2\ell_i(h)}\right)$. This implies that $\widehat{\gamma}(h)$ can be seen as the sum of $n_h$ independent Gamma variables with parameters $\alpha_i = \frac{1}{2}$ and $\beta_i(h) = \frac{n_h}{2\ell_i(h)}, \, \forall i \in \{1, \dots, n\}$. Let $\beta_*(h) = \min\limits_{i \leq n_h} \left\{\beta_i(h)\right\} = \frac{n_h}{2\ell_{\max}(h)}$, where  $\ell_{\max}(h) = \max\limits_{i \leq n_h} \ell_i(h)$. We first study the term $\mathbb{P}\left(\widehat{\gamma}(h) \leq \mu_1 - t\right)$. Using the result from \citep[Theorem 2.57]{bercu2015concentration}, (with $x = \frac{t}{\mu_1} \in ]0, 1[$, since $\mu_1 - t \geq 0$ as Gamma variables are positive random variables):
$$\mathbb{P}\left(\widehat{\gamma}(h) \leq \mu_1 - t\right) \leq \exp\left(- \frac{t^2}{2 V_1^2} \, \right),$$
where $V_1^2 = Var(\widehat{\gamma}(h))$ is the variance of the semi-variogram estimator. Furthermore, from the upper bound on the variance of the semi-variogram estimator given above (see Equation \eqref{eq:sigma1_n}), we have
\begin{equation}
    \forall t > 0, \quad \mathbb{P}\left(\widehat{\gamma}(h) \leq \gamma(h) - t\right) \leq \exp\left(-C'_1 n t^2\right),
\end{equation}
where $C'_1$ is a positive constant depending on $j_1$ only.

 Now, we study the term $\mathbb{P}\left(\widehat{\gamma}(h) \geq \mu_1 + t\right)$. Thanks to a slight modification of the result in \citep[Theorem 4.1]{Wang_Ma} for $k$ independent variables $Z_i \sim \Gamma(u_i, v_i)$, let $\mu_Z = \sum\limits_{i=1}^k \mathbb{E}[Z_i]$ and $v_* = \min v_i$
$$\forall z \geq 1, \quad \mathbb{P}\left(\frac{1}{k}\sum\limits_{i=1}^k \left(Z_i - \mathbb{E}[Z_i]\right) \geq z \mu_Z\right) \, \leq \, \exp\left(- \, v_* \mu_Z \, \, (k z - \ln(1 + k z)) \, \right).$$
Thus,
$$\mathbb{P}\left(\widehat{\gamma}(h) \geq \mu_1 + t\right) \leq \exp\left(\, - \beta_*(h) \mu_1 \left(\frac{t}{\mu_1} - \ln\left(1 + \frac{t}{\mu_1}\right)\right) \, \right).$$

We need an upper bound on the largest eigenvalue of the matrix $L(n, h) \Sigma_n$. First, we use the result presented in Proposition~\ref{aux_res:eig_val}, which is derived from some previous works on inequalities for the eigenvalues of the product of positive semi-definite Hermitian matrices (see \textit{e.g.} \citep{wang1992some, xi2019inequalities}). This result allows us to split the study of the upper bound in two: on one hand the largest eigenvalue of the matrix $L(n, h)$, on the other hand the largest eigenvalue of the covariance matrix $\Sigma_n$. Using Proposition~\ref{aux_res:eig_val}
$$\xi_1(L(n, h)) \xi_n(\Sigma_n) \leq \ell_{\max}(h) \leq \xi_1(L(n, h)) \xi_1(\Sigma_n).$$

As defined, $L(n, h)$ is the Laplacian matrix of the graph described by the regular grid, for a fixed $h \in \mathcal{H}_n$ (see Remark \ref{rem:graph_Gh}). Thus, we refer to Proposition~\ref{aux_res:prop_eig_Lap} above for the result on the bound of the largest eigenvalue of a Laplacian matrix. Furthermore, since the number $n$ of observations is finite and for $h \in \mathcal{H}_n$ the number $n_h$ of pairs at distance $h$ in the regular grid is also finite, the maximum degree $d_{\max}$ of the corresponding graph $G_h$ is also always finite and we have
\begin{equation}\label{eq:upp_bound_Lap}
    \exists c > 0, \, \xi_1(L(n, h)) \leq c.
\end{equation}

Now, we deal with the largest eigenvalue of the covariance matrix. Lemma~\ref{aux_res:lem_eig_cov}, under Assumption \ref{hyp:bounded}, gives the following bound: \begin{equation}\label{eq:bound_eig_cov_mat}
    \exists \, c' > 0, \, \, \xi_1(\Sigma_n) \leq c' M
\end{equation}

Thus, let $c_1$ a positive constant, such that we have the upper bound on the largest eigenvalue
\begin{equation}\label{eq:lmax}
    \ell_{\max}(h) \leq c_1 M.
\end{equation}

This implies the bound on the minimum value of the parameters $\beta_i(h)$
\begin{equation*}
    \exists \, C > 0, \, \, \beta_*(h) \geq C n_h,
\end{equation*}
where $C$ depends on $M$. Combining this result with the lower bound on $n_h$ in Lemma \ref{lem:nh_n}, we have
\begin{equation}
    \forall t > 0, \quad \mathbb{P}\left(\widehat{\gamma}(h) \geq \gamma(h) + t\right) \leq \exp\left(-C_1 n t\right),
\end{equation}
where $C_1$ is a positive constant depending on $j_1$ and $M$ only, which concludes the proof for the semi-variogram estimator tail bounds. In the second part of the proof, we study the variance estimator $\widehat{c}_h(0)$, with similar steps as for the previous result on the semi-variogram estimator. Let $t > 0$ and $h \in \mathcal{H}_n$,
$\mathbb{P}\left(\left\vert\widehat{c}_h(0)-c(0)\right\vert \geq t\right) = \mathbb{P}\left(\widehat{c}_h(0) \geq c(0) + t\right) + \mathbb{P}\left(\widehat{c}_h(0) \leq c(0) - t\right).$
Let $\mu_2 = \mathbb{E}[\widehat{c}_h(0)] = c(0)$ (since $\widehat{c}_h(0)$ is unbiased). Recall that, thanks to Proposition \ref{prop:Gaetan_Guyon}:
$\widehat{c}_h(0) \sim \frac{1}{n_h} \sum\limits_{i=1}^{n_h} \rho_i(h) \chi_{i}^2.$ Since the eigenvalues $\rho_i(h), \, \forall i \in \{1, \dots, n_h\}$ of $D(n, h) \Sigma_n$ are non negatives, from the link between Gamma and Chi-Square random variables, $\frac{1}{n_h} \rho_i(h) \chi_{i}^2 \sim \Gamma\left(\frac{1}{2}, \frac{n_h}{2\rho_i(h)}\right)$. This implies that $\widehat{c}_h(0)$ can be seen as the sum of $n_h$ independent Gamma random variables with parameters $a_i = \frac{1}{2}$ and $b_i(h) = \frac{n_h}{2\rho_i(h)}, \, \forall i \in \{1, \dots, n\}$. Let $b_*(h) = \min\limits_{i \leq n_h} \left\{b_i(h)\right\} = \frac{n_h}{2\rho_{\max}(h)}$, where $\rho_{\max}(h) = \max\limits_{i \leq n_h} \rho_i(h)$. We first study the term $\mathbb{P}\left(\widehat{c}_h(0) \leq \mu_2 - t\right)$. Using the result from \citep[Theorem 2.57]{bercu2015concentration}, (with $x = \frac{t}{\mu_2} \in ]0, 1[$, since $\mu_2 - t \geq 0$ as Gamma variables are positive random variables):
$$\mathbb{P}\left(\widehat{c}_h(0) \leq \mu_2 - t\right) \leq \exp\left(- \frac{t^2}{2 V_2^2} \, \right),$$
where $V_2^2 = Var(\widehat{c}_h(0))$. From the result given above (see Equation \eqref{eq:sigma2_n}), we have
\begin{equation}
    \forall t > 0, \quad \mathbb{P}\left(\widehat{c}_h(0) \leq c(0) - t\right) \leq \exp\left(-C'_2 n t^2\right),
\end{equation}
where $C'_2$ is a positive constant depending on $j_1$ only.

Now, we study the term $\mathbb{P}\left(\widehat{c}_h(0) \geq \mu_2 + t\right)$. Combining the bound for the product of matrices given in Proposition~\ref{aux_res:eig_val}, the bound on the largest eigenvalue of the matrix of the degrees of a graph (see Proposition~\ref{prop:aux_res_d_max}) and the bound on the largest eigenvalue of a covariance matrix in Equation~\eqref{eq:bound_eig_cov_mat}, we have an upper bound on the largest eigenvalue of the matrix $D(n, h) \Sigma_n$: 
\begin{equation}\label{eq:rhomax}
    \rho_{\max}(h) \leq c_2 M.
\end{equation}
Thus, using the same argumentation as for the semi-variogram estimator
\begin{equation}
    \forall t > 0, \quad \mathbb{P}\left(\widehat{c}_h(0) \geq c(0) + t\right) \leq \exp\left(-C_2 n t\right),
\end{equation}
where $C_2$ is a positive constant depending on $j_1$ and $M$ only, which concludes the proof.

\subsection{Proof of Corollary \ref{cor}}
For the proof, we shall use both preliminary results in Appendix~\ref{subsec:aux_cor}: an upper bound on the total number of distinct observable distances and Corollary~\ref{cor:aux}, which proof, that simply follows from Proposition \ref{prop:CI_var_lag}, is given in Appendix~\ref{subsec:aux_cor}. In a first place, we study $\left\vert\widehat{c}(h) - c(h)\right\vert$ for all $h \geq 0$. Thanks to the definition of the covariance function estimation at unobserved lags by mean of the $1$-NN estimator (see subsection \ref{subsec:nonpar_cov_est}), for any distance $h$, let $h_o \in \mathcal{H}_n$ the observable distance that is the $1$-NN of $h$ and such that: $\widehat{c}(h) = \widehat{c}(h_o)$. Then,
\begin{equation*}
    \left\vert\widehat{c}(h) - c(h)\right\vert = \left\vert\widehat{c}(h) - c(h) + c(h_o) - c(h_o)\right\vert \leq \left\vert\widehat{c}(h_o) - c(h_o)\right\vert + \left\vert c(h_o) - c(h)\right\vert
\end{equation*}
Applying the mean value (or finite increment) inequality, combined with Assumption \ref{hyp:smooth}, we have
\begin{align*}
    \left\vert c(h_o) - c(h)\right\vert \leq D \vert\vert h - h_o \vert\vert \leq \frac{D}{\sqrt{n}-1},
\end{align*}
since $\forall h \geq 0, \, \vert\vert h - h_o \vert\vert \leq 1/(\sqrt{n}-1)$ (see subsection \ref{subsec:nonpar_cov_est}). From the link between the covariance and the semi-variogram functions and the link for their estimators in \eqref{eq:estvar2}, we have
$$\left\vert \widehat{c}(h_o) - c(h_o) \right\vert \leq \left\vert \widehat{c}_{h_o}(0) - c(0) \right\vert + \left\vert \widehat{\gamma}(h_o)- \gamma(h_o) \right\vert.$$
Then, we have:
$$\sup\limits_{h \geq 0} \left\vert\widehat{c}(h) - c(h)\right\vert \leq \max\limits_{h \in \mathcal{H}_n} \left\vert \widehat{c}_h(0) - c(0) \right\vert + \max\limits_{h \in \mathcal{H}_n} \left\vert \widehat{\gamma}(h)- \gamma(h) \right\vert + \frac{D}{\sqrt{n}-1}.$$
This yields: $\forall t > 0$,
\begin{align*}
    \mathbb{P} & \left(\sup\limits_{h \geq 0} \left\vert\widehat{c}(h) - c(h)\right\vert \geq t\right)\\ & \quad \leq \mathbb{P}\left(\max\limits_{h \in \mathcal{H}_n} \left\vert \widehat{c}_h(0) - c(0) \right\vert + \max\limits_{h \in \mathcal{H}_n} \left\vert \widehat{\gamma}(h)- \gamma(h) \right\vert + D/(\sqrt{n}-1) \geq t\right)\\ & \quad \leq \mathbb{P}\left(\max\limits_{h \in \mathcal{H}_n} \left\vert \widehat{c}_h(0) - c(0) \right\vert \geq \frac{1}{2}\left(t - D/(\sqrt{n}-1)\right)\right)\\ & \quad \quad + \mathbb{P}\left(\max\limits_{h \in \mathcal{H}_n} \left\vert \widehat{\gamma}(h)- \gamma(h) \right\vert \geq \frac{1}{2}\left(t - D/(\sqrt{n}-1)\right)\right).
\end{align*}
Then, we can apply the result in Corollary \ref{cor:aux} for both estimators with $k = \left(t - D/(\sqrt{n}-1)\right)/2$ and we obtain
\begin{align*}
    \forall t > 0, \, \mathbb{P}\left(\sup\limits_{h \geq 0} \left\vert\widehat{c}(h) - c(h)\right\vert \geq t\right) & \leq 2 n e^{- C'_2 n k^2} + 2 n e^{- C'_1 n k^2},
\end{align*}
as soon as $k \leq \min\left\{C_1/C'_1,\; C_2/C'_2\right\} = C'_{min}$. Furthermore, we have
\begin{align*}
    2 n e^{- C'_2 n k^2} + 2 n e^{- C'_1 n k^2} \leq 4 n \max\left\{e^{- C'_2 n k^2}, e^{- C'_1 n k^2}\right\} = 4 n e^{- C_{min} n k^2},
\end{align*}
where $C_{min} = \min\{C'_1, C'_2\}$. Finally, let $\delta \in (0, 1)$, such that
$\delta = 4 n e^{- C_{min} n k^2}$ with $k = \frac{1}{2}\left(t - D/(\sqrt{n}-1)\right)$. Thus, by a simple calculation, this implies that there exists a positive constant $C_3 = 2/\sqrt{C_{min}}$ depending on $j_1$, $m$ and $M$ solely such that $t = C_3\sqrt{\log\left(4n/\delta\right)/n} + D/(\sqrt{n}-1)$. Furthermore, going back to the condition on the variable $k$, by a straightforward computation, we have
$$k = \frac{1}{2}\left(t - \frac{D}{\sqrt{n}-1}\right) \leq C'_{min} \Longleftrightarrow n \geq C'_3 \log\left(\frac{4 n}{\delta}\right),$$
where $C'_3$ is a positive constant depending on $j_1$, $m$ and $M$ solely. Thus,
$$\mathbb{P}\left(\sup_{h\geq 0}\left\vert \widehat{c}(h)- c(h) \right\vert \leq C_3\sqrt{\log\left(4n/\delta\right)/n}+ D/(\sqrt{n}-1)\right) \geq 1 - \delta,$$
as soon as $n \geq C'_3 \log\left(\frac{4 n}{\delta}\right)$.

\subsection{Proof of Proposition \ref{prop:CI_Sigma_vec}}
\textbf{Proof of Assertion \textit{(i)}}

First, recall that the max norm and the operator norm are equivalent (since any norms in a given finite-dimensional vector space are equivalent and that the space of the squared matrices of size $d$ is a finite-dimensional vector space):
\begin{equation}
    \vert\vert\vert\widehat{\Sigma}(\bs_d)- \Sigma(\bs_d)\vert\vert\vert \leq \, d \, \vert\vert\widehat{\Sigma}(\bs_d)- \Sigma(\bs_d)\vert\vert_{\infty},
    \label{eq:Sigma_eq_op_max}
\end{equation}
where $\vert\vert A \vert\vert_{\infty} = \max\limits_{i, j \in \{1, \cdots, d\}} \vert A_{ij} \vert$ is the max norm for any squared matrix $A$ of size $d$. By the definition of the estimated and the true covariance matrices, notice that
\begin{align*}
    \vert\vert\widehat{\Sigma}(\bs_d)- \Sigma(\bs_d)\vert\vert_{\infty} & = \max\limits_{i, j \in \{1, \cdots, d\}} \left\vert \widehat{c}(\vert\vert s_i - s_j \vert\vert)- c(\vert\vert s_i - s_j \vert\vert) \right\vert\\ & \leq \sup\limits_{h \geq 0} \left\vert \widehat{c}(h) - c(h) \right\vert.
\end{align*}

Then, applying the non-asymptotic bound in Corollary \ref{cor}, we have the wanted result.

\vspace{0.3cm}

\noindent \textbf{Proof of Assertion \textit{(ii)}}

Thanks to the result in the previous assertion, where the operator norm of the difference $\widehat{\Sigma}(\bs_d)- \Sigma(\bs_d)$ is bounded with high probability, we can deduce that the eigenvalues of $\widehat{\Sigma}(\bs_d)$ have near values to the eigenvalues of $\Sigma(\bs_d)$. Recall that the eigenvalues of $\Sigma(\bs_d)$ are assumed to be bounded by $\underline{m}$ and $\overline{M}$, two positive constants (see Assumption \ref{hyp:cov_bound}). Thus, with high probability, the spectrum of $\widehat{\Sigma}(\bs_d)$ is also bounded by $\underline{m} > 0$ and $\overline{M} > 0$. Finally, we can deduce that $\widehat{\Sigma}(\bs_d)$ is invertible  with high probability. The first step of the proof is to apply the result from Theorem \ref{thm_aux:wedin} (refer to Appendix \ref{proofs:aux_res} for its presentation). Indeed
\begin{equation}
    \vert\vert\vert\widehat{\Sigma}(\bs_d)^{-1} - \Sigma(\bs_d)^{-1} \vert\vert\vert \leq \vert\vert\vert \Sigma(\bs_d)^{-1} \vert\vert\vert \; \vert\vert\vert \widehat{\Sigma}(\bs_d)^{-1} \vert\vert\vert \; \vert\vert\vert\widehat{\Sigma}(\bs_d)- \Sigma(\bs_d)\vert\vert\vert.
    \label{eq:ine_wedin}
\end{equation}

First notice that under Assumption \ref{hyp:gaussian}, $\Sigma(\bs_d)$ is always positive definite and invertible, so all its eigenvalues are strictly positive. Furthermore, we know that the operator norm of a symmetric positive definite matrix is equal to the largest eigenvalue of the matrix: $\vert\vert\vert \Sigma(\bs_d)^{-1} \vert\vert\vert = \max\limits_{i \in \{1, \cdots, d\}} \xi_i\left(\Sigma(\bs_d)^{-1}\right) = \xi_d\left(\Sigma(\bs_d)\right)^{-1}.$ Finally, one has:
\begin{equation}\label{bound_N1}
    \vert\vert\vert \Sigma(\bs_d)^{-1} \vert\vert\vert \leq \underline{m}^{-1},
\end{equation}
where $\underline{m} > 0$ is the lower bound of the spectrum of $\Sigma(\bs_d)$. As a consequence of Assertion \textit{(i)}, the eigenvalues of $\widehat{\Sigma}(\bs_d)$ are also bounded and bounded away from $0$, with high probability. Using the same argumentation as above, one has, with high probability:
\begin{equation}\label{bound_N2}
    \forall \delta \in (0, 1), \, \mathbb{P}\left(\vert\vert\vert \widehat{\Sigma}(\bs_d)^{-1} \vert\vert\vert \leq \underline{m}^{-1}\right) \geq 1 - \delta.
\end{equation}

Thus, going back to the inequality \eqref{eq:ine_wedin}
\begin{equation*}
    \forall \delta \in (0, 1), \, \mathbb{P}\left(\vert\vert\vert\widehat{\Sigma}(\bs_d)^{-1} - \Sigma(\bs_d)^{-1} \vert\vert\vert \leq \, (\underline{m}^{-1})^2 \, \vert\vert\vert\widehat{\Sigma}(\bs_d)- \Sigma(\bs_d)\vert\vert\vert\right) \geq 1 - \delta.
\end{equation*}

Combining all the previous results and the accuracy $\vert\vert\vert\widehat{\Sigma}(\bs_d)- \Sigma(\bs_d)\vert\vert\vert$ of the covariance matrix estimator (described in a non-asymptotic fashion by the bound given in Assertion \textit{(i)}), one can deduce the result in Assertion \textit{(ii)}.

\subsection{Proof of Theorem \ref{thm:EGRisk}}
\textbf{Sketch of Proof.} Shedding light onto the role of the technical assumptions made in subsection \ref{subsec:nonpar_cov_est}, we first give a brief idea of the proof's approach.

The proof of Assertion \textit{(i)} essentially relies on the following bound
\begin{align*}
    \sup\limits_{s \in S}\vert\vert \widehat{\Lambda}_d(s) - \Lambda_d^*(s) \vert\vert \, & \leq \, \underbrace{\vert\vert\vert \Sigma(\bs_d)^{-1} \vert\vert\vert}_{N_1} \; \overbrace{\sup\limits_{s \in S} \vert\vert \widehat{\bc}_d(s) - \bc_d(s) \vert\vert}^{N_2}\\ & + \underbrace{\vert\vert\vert \widehat{\Sigma}(\bs_d)^{-1} - \Sigma(\bs_d)^{-1} \vert\vert\vert}_{N_3} \; \overbrace{\sup\limits_{s \in S} \vert\vert \widehat{\bc}_d(s) \vert\vert}^{N_4},
\end{align*}

where

\begin{itemize}
    \item For \textit{term $N_1$:} Assumption \ref{hyp:gaussian} (all eigenvalues of $\Sigma(\bs_d)$ are strictly positives) and Assumption \ref{hyp:cov_bound} (spectrum of $\Sigma(\bs_d)$ is lower bounded by $\underline{m})$, imply that one has $\vert\vert\vert \Sigma(\bs_d)^{-1} \vert\vert\vert \leq \underline{m}^{-1}$.
    \item A bound for \textit{term $N_2$} can be deduced from the link between the max norm and the Euclidean norm, and the upper bound in Corollary \ref{cor}.
    \item Refer to Proposition \ref{prop:CI_Sigma_vec} Assertion \textit{(ii)} for a bound with high probability of \textit{term $N_3$}.
    \item \textit{Term $N_4$:} From Corollary \ref{cor} and Assumption \ref{hyp:smooth}, we deduce that $\sup\limits_{h \geq 0} \vert \widehat{c}(h) \vert < B$, with high probability.
\end{itemize}

Using Equation \eqref{eq:excess_IQR} given at the end of subsection \ref{subsec:bounds_ER}, and since the domain $S$ is bounded, the remaining terms to study are

\begin{itemize}
    \item As a consequence of Proposition \ref{prop:CI_Sigma_vec} Assertion \textit{(i)}, with probability at least $1 - \delta$, $\forall \delta \in (0, 1)$, the eigenvalues of $\widehat{\Sigma}(\bs_d)$ are close to the eigenvalues of $\Sigma(\bs_d)$. Thus, one can deduce the upper bound $\vert\vert\vert\widehat{\Sigma}(\bs_d)^{-1} \vert\vert\vert \leq \underline{m}^{-1}$, with high probability.
    \item A bound for $\sup\limits_{s \in S} \vert\vert \bc_d(s) \vert\vert$ is deduced from the link between the max norm and the Euclidean norm, together with Assumption \ref{hyp:smooth}: $\sup\limits_{s \in S} \vert\vert \bc_d(s) \vert\vert \leq \sqrt{d} B$.
\end{itemize}

\vspace{0.3cm}

\noindent \textbf{Proof of Assertion \textit{(i)}}

First, notice that
\begin{align*}
    \vert\vert \widehat{\Lambda}_d(s) - \Lambda_d^*(s) \vert\vert & = \vert\vert \widehat{\Sigma}(\bs_d)^{-1} \widehat{\bc}_d(s) - \Sigma(\bs_d)^{-1} \bc_d(s) \vert\vert\\ & = \vert\vert \Sigma(\bs_d)^{-1} \left(\widehat{\bc}_d(s) - \bc_d(s)\right) + \left(\widehat{\Sigma}(\bs_d)^{-1} - \Sigma(\bs_d)^{-1} \right) \widehat{\bc}_d(s) \vert\vert\\ & \leq \vert\vert\vert \Sigma(\bs_d)^{-1} \vert\vert\vert \; \vert\vert \widehat{\bc}_d(s) - \bc_d(s) \vert\vert + \vert\vert\vert \widehat{\Sigma}(\bs_d)^{-1} - \Sigma(\bs_d)^{-1} \vert\vert\vert \; \vert\vert \widehat{\bc}_d(s) \vert\vert,
\end{align*}
and taking the supremum over the domain $S$
\begin{align*}
    \sup\limits_{s \in S}\vert\vert \widehat{\Lambda}_d(s) - \Lambda_d^*(s) \vert\vert \, & \leq \, \vert\vert\vert \Sigma(\bs_d)^{-1} \vert\vert\vert \; \sup\limits_{s \in S} \vert\vert \widehat{\bc}_d(s) - \bc_d(s) \vert\vert\\ & + \vert\vert\vert \widehat{\Sigma}(\bs_d)^{-1} - \Sigma(\bs_d)^{-1} \vert\vert\vert \; \sup\limits_{s \in S} \vert\vert \widehat{\bc}_d(s) \vert\vert,
\end{align*}

Firstly, for the accuracy of the covariance vector estimator, since the max norm and the Euclidean norm are equivalent, one has
\begin{multline*}
    \sup\limits_{s \in S}\vert\vert \widehat{\bc}_d(s) - \bc_d(s) \vert\vert \leq \, \sqrt{d} \; \sup\limits_{s \in S}\vert\vert \widehat{\bc}_d(s) - \bc_d(s) \vert\vert_{\infty}\\ = \, \sqrt{d} \; \sup\limits_{s \in S}\max\limits_{i \in \{1, \cdots, d\}} \vert \widehat{c}(\vert\vert s - s_i \vert\vert) - c(\vert\vert s - s_i \vert\vert) \vert \leq \, \sqrt{d} \; \sup\limits_{h \geq 0} \vert \widehat{c}(h) - c(h) \vert,
\end{multline*}
which allows using Corollary \ref{cor} to control in a non-asymptotic fashion the supremum over all positive lags of the error estimation of the covariance function. Thus, one obtains the following bound for any $\delta \in (0, 1)$, with probability at least $1 - \delta$
\begin{equation}
    \sup\limits_{s \in S} \vert\vert \widehat{\bc}_d(s) - \bc_d(s) \vert\vert \leq C_3 \, \sqrt{d} \, \sqrt{\log(4 n / \delta)/n} + \sqrt{d} \, D/(\sqrt{n}-1),
\end{equation}
as soon as $n \geq C'_3 \log(4n/\delta)$. As above, from the link between the max norm and the Euclidean norm, one has
\begin{align*}
    \sup\limits_{s \in S}\vert\vert \widehat{\bc}_d(s) \vert\vert \, \leq \, \sqrt{d} \; \sup\limits_{s \in S}\vert\vert \widehat{\bc}_d(s) \vert\vert_{\infty} \leq \, \sqrt{d} \; \sup\limits_{h \geq 0} \vert \widehat{c}(h) \vert.
\end{align*}

Furthermore, as a consequence of the result in Corollary \ref{cor}, we can deduce that $\widehat{c}(h)$ is close to $c(h)$ for any lag $h \geq 0$, with high probability. Thus, under Assumption \ref{hyp:smooth}, one has the bound
\begin{equation}\label{bound_U2}
    \forall \delta \in (0, 1), \, \mathbb{P}\left(\sup\limits_{s \in S}\vert\vert \widehat{\bc}_d(s) \vert\vert \leq \sqrt{d} B \right) \geq 1 - \delta.
\end{equation}

Lastly, notice that the last two terms have been studied in previous proofs: $\vert\vert\vert \Sigma(\bs_d)^{-1} \vert\vert\vert$ is upper bounded by $\underline{m}^{-1}$ (see Equation \eqref{bound_N1} in the proof of Proposition \ref{prop:CI_Sigma_vec} Assertion \textit{(ii)}) ; and $\vert\vert\vert \widehat{\Sigma}(\bs_d)^{-1} - \Sigma(\bs_d)^{-1} \vert\vert\vert$ is bounded in a non-asymptotic fashion in Proposition \ref{prop:CI_Sigma_vec} Assertion \textit{(ii)}. Thus, combining all the previous results, one can deduce the wanted non-asymptotic bound.

\vspace{0.3cm}

\noindent \textbf{Proof of Assertion \textit{(ii)}}

As announced in subsection \ref{subsec:bounds_ER}, with probability one, the excess of integrated quadratic risk can be written as follows:

\begin{multline*}
     L_S(f_{\widehat{\Lambda}_d})-L^*_S=\\
     \int_{s\in S} \left( \widehat{\Lambda}_d(s)^{\top}\Sigma(\bs_d)\widehat{\Lambda}_d(s)- \Lambda_d^*(s)^{\top}\Sigma(\bs_d)\Lambda_d^*(s)  -2\; \bc_d(s)^{\top}\left(\widehat{\Lambda}_d(s) -\Lambda_d^*(s)\right)  \right) ds\\ =  \int_{s\in S} \, \left(\widehat{\Lambda}_d(s) - \Lambda_d^*(s)\right)^{\top}\Sigma(\bs_d)\widehat{\Lambda}_d(s) + \Lambda_d^*(s)^{\top}\Sigma(\bs_d)\left(\widehat{\Lambda}_d(s) - \Lambda_d^*(s)\right)\\  + 2\; \bc_d(s)^{\top}\left(\Lambda_d^*(s) - \widehat{\Lambda}_d(s)\right)   ds
\end{multline*}

Notice that
\begin{align*}
    \left(\widehat{\Lambda}_d(s) - \Lambda_d^*(s)\right)^{\top} & \Sigma(\bs_d)\widehat{\Lambda}_d(s) = \langle\widehat{\Lambda}_d(s) - \Lambda_d^*(s), \Sigma(\bs_d)\widehat{\Lambda}_d(s)\rangle\\ & \leq \vert\vert \widehat{\Lambda}_d(s) - \Lambda_d^*(s)\vert\vert \, \vert\vert\Sigma(\bs_d)\widehat{\Lambda}_d(s) \vert\vert\\ & \leq \vert\vert \widehat{\Lambda}_d(s) - \Lambda_d^*(s)\vert\vert \, \vert\vert\vert \Sigma(\bs_d) \vert\vert\vert \, \vert\vert\vert\widehat{\Sigma}(\bs_d)^{-1} \vert\vert\vert \, \vert\vert\widehat{\bc}_d(s) \vert\vert.
\end{align*}

Following the same idea, we have
\begin{align*}
    \Lambda_d^*(s)^{\top}\Sigma(\bs_d) & \left(\widehat{\Lambda}_d(s) - \Lambda_d^*(s)\right)\\ & \leq \vert\vert\vert \Sigma(\bs_d)^{-1} \vert\vert\vert \,  \vert\vert \bc_d(s) \vert\vert \, \vert\vert\vert \Sigma(\bs_d) \vert\vert\vert \,  \vert\vert \widehat{\Lambda}_d(s) - \Lambda_d^*(s) \vert\vert,
\end{align*}
and
\begin{align*}
    \bc_d(s)^{\top}\left(\Lambda_d^*(s) - \widehat{\Lambda}_d(s)\right) \leq \vert\vert \bc_d(s) \vert\vert \, \vert\vert \Lambda_d^*(s) - \widehat{\Lambda}_d(s)\vert\vert.
\end{align*}

Thus, taking the supremum over the domain $S$, the term under study has now become

\begin{multline*}
    \sup\limits_{s \in S} \vert\vert \widehat{\Lambda}_d(s) - \Lambda_d^*(s)\vert\vert \, \vert\vert\vert \Sigma(\bs_d) \vert\vert\vert \, \vert\vert\vert\widehat{\Sigma}(\bs_d)^{-1} \vert\vert\vert \, \sup\limits_{s \in S} \vert\vert\widehat{\bc}_d(s) \vert\vert\\
    + \sup\limits_{s \in S} \vert\vert \widehat{\Lambda}_d(s) - \Lambda_d^*(s)\vert\vert \, \vert\vert\vert \Sigma(\bs_d) \vert\vert\vert \, \vert\vert\vert \Sigma(\bs_d)^{-1} \vert\vert\vert \, \sup\limits_{s \in S} \vert\vert \bc_d(s) \vert\vert\\
    + 2 \sup\limits_{s \in S} \vert\vert \widehat{\Lambda}_d(s) - \Lambda_d^*(s)\vert\vert \, \sup\limits_{s \in S} \vert\vert \bc_d(s) \vert\vert.
\end{multline*}

Notice that some of the terms have already been studied in previous results and their proofs: a non-asymptotic bound for $\sup\limits_{s \in S} \vert\vert \widehat{\Lambda}_d(s) - \Lambda_d^*(s)\vert\vert$ is given by Theorem \ref{thm:EGRisk} Assertion \textit{(i)} ; the operator norm of the precision matrix estimator is upper bounded by $\underline{m}^{-1}$, with high probability, in Equation \eqref{bound_N2} (see the proof of Proposition \ref{prop:CI_Sigma_vec} Assertion \textit{(ii)}) ; the supremum over all domain $S$ of the Euclidean norm of the covariance vector estimator is upper bounded with high probability by $\sqrt{d} B$ (see Equation \eqref{bound_U2}, in the proof of Assertion \textit{(i)}) ; and $\vert\vert\vert \Sigma(\bs_d)^{-1} \vert\vert\vert$ is upper bounded by $\underline{m}^{-1}$ (see Equation \eqref{bound_N1} in the proof of Proposition \ref{prop:CI_Sigma_vec} Assertion \textit{(ii)}). Furthermore, from Assumption \ref{hyp:cov_bound}, $\vert\vert\vert \Sigma(\bs_d) \vert\vert\vert \leq \overline{M}$. Finally, for the last term defined as the supremum over the domain $S$ of the Euclidean norm of the covariance vector, using the link between the max norm and the Euclidean norm, and Assumption \ref{hyp:smooth}
\begin{align*}
    \sup\limits_{s \in S} \vert\vert \bc_d(s) \vert\vert \leq \sqrt{d} \; \sup\limits_{s \in S} \vert\vert \bc_d(s) \vert\vert_{\infty} \leq \sqrt{d} \; \sup\limits_{h \geq 0} \vert c(h) \vert \leq \sqrt{d} B.
\end{align*}

Since the domain $S$ is bounded, combining all these results allows us to conclude.

\section{Additional Experiments}\label{sec:add_num}
\subsection{Complete Maps on Real Data}

Firstly, we give further details on the Real Data application in subsection~\ref{subsec:exp_num_RD}, for the same setting: the observations are taken on a dyadic grid at scale $J = 4$. The maps of all mean squared errors for parametric and nonparametric Kriging predictors are depicted in Figure \ref{fig:RD_Tot}.

\begin{figure}[h]
    \centering
    \subfigure[Parametric]{\includegraphics[width=55mm]{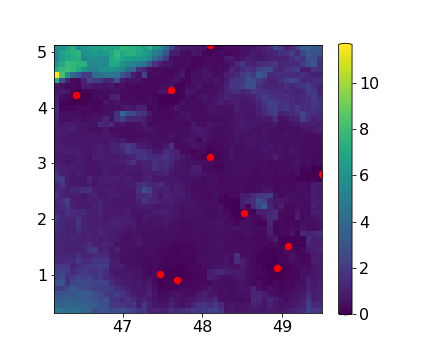}
        \label{fig:RD_J4_Theo}}
    \subfigure[Nonparametric]{\includegraphics[width=55mm]{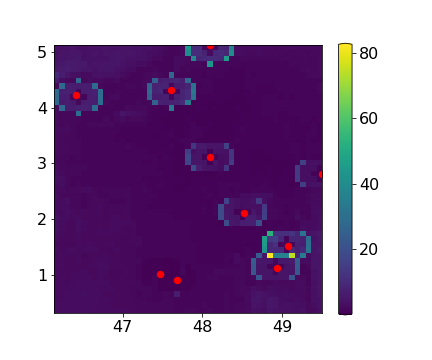}
        \label{fig:RD_J4_Emp}}
    \caption{Complete maps of all MSE on Real Data on a dyadic grid of observations at scale $J = 4$ (with $N = 2401$ and $d = 10$).}
    \label{fig:RD_Tot}
\end{figure}

As for Kriging applied to simulated data, the exact interpolator property is verified (see subsection \ref{sec:exp_num_SK}).
Notice that, for parametric Kriging, some border effects can be observed ; while, for nonparametric Kriging, there is a presence of local area with higher error in form of circles at a certain distance of the observed locations, where the mean error is higher and seems null everywhere else on the spatial domain.
Still, the results on real data are encouraging to extend our theoretical results to a more general framework. Indeed, recall that these real data are irregular and violate some of the made assumptions: Assumption \ref{hyp:simple} is not verified since the mean of the temperatures over several locations in France is not null; Assumption \ref{hyp:decorr} is not satisfied, as discussed in the choice of the value of the parameter $\theta$ (see subsection \ref{subsec:exp_num_RD}).

\subsection{Extension to Different Configurations of the Observations' Locations}

One may be interested in the influence of the configuration of the observation points $s_1, \cdots, s_d$ on the performance of the Kriging predictor. The results are presented here for two extreme situations: the first, called \textit{Corner} (C), happens when the major number of observations are taken randomly in a small sub-region of the spatial domain $S$, defined as one of the corners of the domain; for the second one, called \textit{Ring} (R), the major number of observations are sampled randomly in a circle of center equal to the middle of the spatial domain $S$.

\begin{table}[h]
    \caption{Mean and standard deviation (std) of the AMSE on $100$ independent simulations of a Gaussian process with truncated power law (left) and Gaussian (right) covariance functions for theoretical and empirical Kriging with different configurations of the observations' locations (where U: \textit{Uniform} ; C: \textit{Corner} ; R: \textit{Ring}) (with $J = 4$, $N = 1681$, $d = 60$ and $\theta = 5$).}
    
    \vskip 0.15in
    
    \begin{center}
    \centering
    \begin{tabular}{ | c || c | c || c | c | }
        \hline
        \small{\texttt{TPL}}
        &
        \multicolumn{2}{| c ||}{ \small{\texttt{Theoretical}} } & \multicolumn{2}{| c |}{ \small{\texttt{Empirical}} }
        \\ \hline \hline
         & mean & std & mean & std
        \\ \hline
        U & 0.747 & 0.113 & 0.797 & 0.120
        \\ \hline
        C & 0.850 & 0.126 & 1.596 & 6.899
        \\ \hline
        R & 0.866 & 0.126 & 0.917 & 0.140
        \\ \hline
    \end{tabular}
  \hspace{0.15cm}
    \centering
    \begin{tabular}{ | c || c | c || c | c | }
        \hline
        \small{\texttt{GAU}}
        &
        \multicolumn{2}{| c ||}{ \small{\texttt{Theoretical}} } & \multicolumn{2}{| c |}{ \small{\texttt{Empirical}} }
        \\ \hline \hline
         & mean & std & mean & std
        \\ \hline
        U & 0.186 & 0.107 & 67.814 & 476.812
        \\ \hline
        C & 0.598 & 0.270 & 3.676 & 18.225
        \\ \hline
        R & 0.517 & 0.182 & 84.840 & 815.206
        \\ \hline
    \end{tabular}
    \label{tab:all_Irregular}
    \end{center}
    \vskip 0.1in
\end{table}

\begin{figure}[h]
    \begin{center}
    \subfigure[Uniform]{\includegraphics[width=50mm]{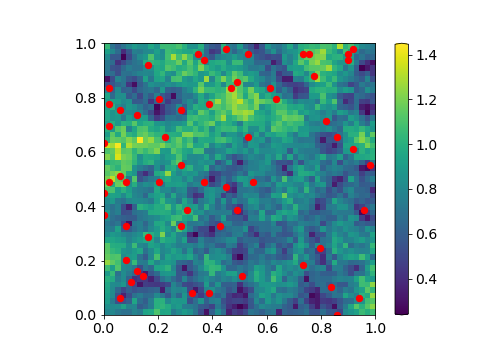}}
    \subfigure[Corner]{\includegraphics[width=50mm]{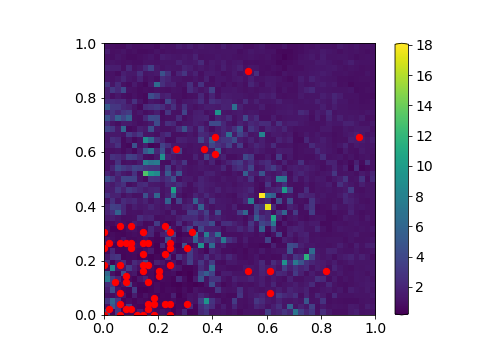}
        \label{fig:TPL_Irregular_NotReg_Emp}}
    \subfigure[Ring]{\includegraphics[width=50mm]{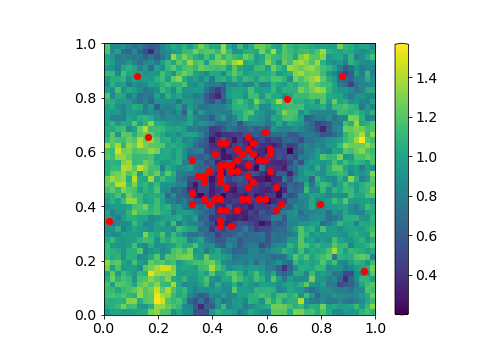}}\\
    \subfigure[Uniform]{\includegraphics[width=50mm]{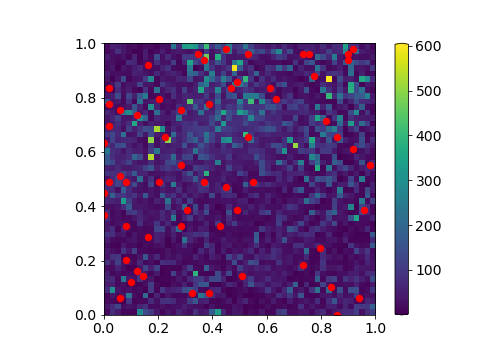}
        \label{fig:Gauss_Irregular_Random_Emp}}
    \subfigure[Corner]{\includegraphics[width=50mm]{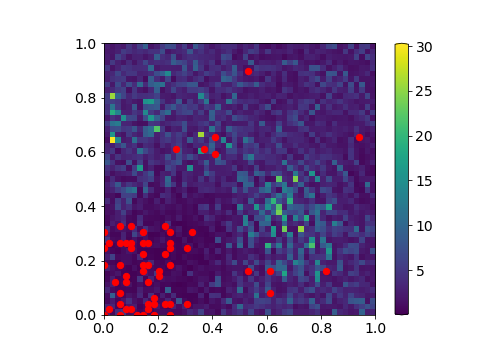}}
    \subfigure[Ring]{\includegraphics[width=50mm]{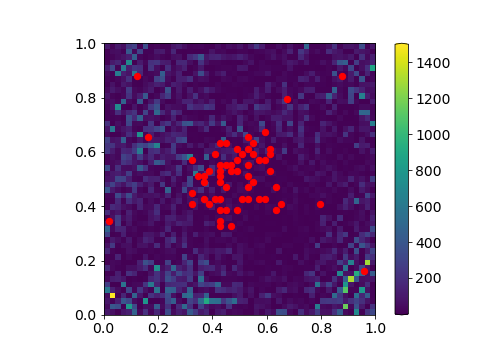}
        \label{fig:Gauss_Irregular_Ring_Emp}}
    \end{center}
    \caption{Complete maps of all MSE on $100$ realizations of a Gaussian process with truncated power law covariance (top) and Gaussian (bottom) covariance functions for the empirical Kriging predictor with different configurations of the observations' locations $s_1, \cdots, s_d$ (with $J = 4$, $N = 1681$, $d = 60$ and $\theta = 5$).}
    \label{fig:Irregular_Tot}
\end{figure}

The same procedure as in Section \ref{sec:num} is then applied for both covariance models, with the following setting: the number of training observations is equal to $n = 289$ (dyadic scale $J = 4$); the total number of sampled locations for the prediction is fixed to $d = 60$, where $50$ are taken in the sub-region of interest and the others $10$ in the remaining area of the domain; and the correlation length is fixed at $\theta = 5$. The mean and the standard deviation of all AMSE are presented in Table \ref{tab:all_Irregular} and the complete maps of all MSE are displayed in Figure \ref{fig:Irregular_Tot}, where \textit{Uniform} (U) stands for the selection of the observations' locations as before, using a random uniform procedure over the spatial domain $S$. For the truncated power law model, the results in Table \ref{tab:all_Irregular} (top) are similar when using the \textit{Uniform} procedure or the \textit{Ring} procedure and the errors are close for the theoretical and the empirical methods. In the \textit{Corner} case, the mean AMSE for the empirical Kriging is larger, and especially its standard deviation increases. Looking at the corresponding complete map in Figure \ref{fig:TPL_Irregular_NotReg_Emp}, it can be seen that the predictor seems to succeed for the point locations near the observations but fails at some locations far from any observed sample (with a large error, going up to $18$, as shown by the large error scale). This is a direct consequence of the fact that couples of point locations that are at a large distance from one another are under-represented in this setting.
When looking at the Gaussian model results in Table \ref{tab:all_Irregular} (bottom), it's obvious that this model is strongly linked to the configuration of the observations: the mean and the standard deviation for empirical Kriging are significantly larger than when the true covariance function is known. Indeed, let us observe that for the theoretical Kriging method, the mean and the standard deviation are more or less the same for the three configurations, whereas, for the empirical Kriging method, these two values rise abruptly when the observations are taken mainly in a circle. Thanks to Figure \ref{fig:Gauss_Irregular_Ring_Emp}, one can notice that the point locations that make the mean error explodes are located in the boundaries of the spatial domain, far from any observed sample (the maximum error value is more than $1400$). Other observations that can be made with these results are on the influence of the number $d$ of observed samples for the prediction step. When using the truncated power law model, the results for the \textit{Uniform} configuration in Table \ref{tab:all_Irregular} (top) are similar to the mean and standard deviation in Table \ref{tab:TPL_Gauss_J4_theta} (left) when $\theta = 5$. So, the size $d$ of observations does not seem to have an impact on the performance of the Kriging estimator. In contrast, for the Gaussian model, when the empirical version is used, the results significantly change between Table \ref{tab:all_Irregular} (bottom) when the configuration is \textit{Uniform} and Table \ref{tab:TPL_Gauss_J4_theta} (right) when $\theta = 5$, with a strong increase for both the mean and the standard deviation in the case where $d = 60$.

Therefore, it can be of interest to explore how the performance of Kriging can be affected by a variation in the observations' locations.

\subsection{Additional Covariance Models}

Based on several covariance models, fulfilling or not our assumptions, extra numerical experiments were performed to assess the validity of Theorem~\ref{thm:EGRisk}.

\begin{figure}[h]
    \centering
    \subfigure[Cubic]{\includegraphics[width=60mm]{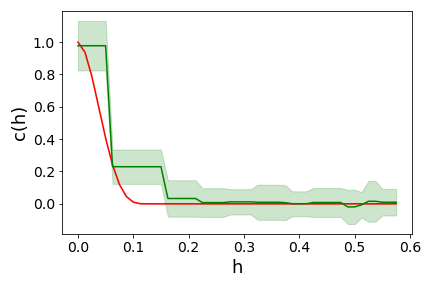}
        \label{fig:Cubic_cov_est}}
    \subfigure[Spherical]{\includegraphics[width=60mm]{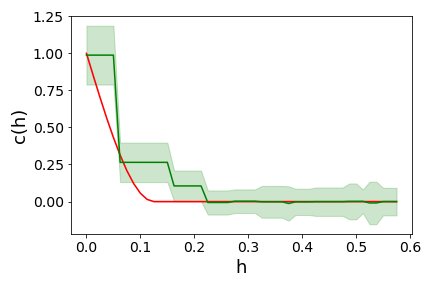}
        \label{fig:Spherical_cov_est}}
    \\
    \subfigure[Exponential]{\includegraphics[width=60mm]{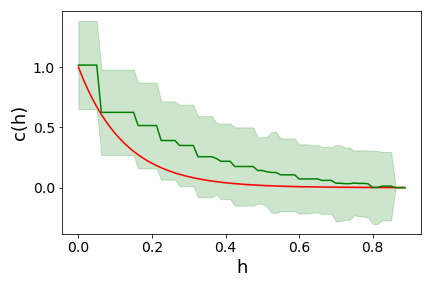}
        \label{fig:Exponential_cov_est}}
    \\
    \subfigure[Matern ($\nu = 3/2$)]{\includegraphics[width=60mm]{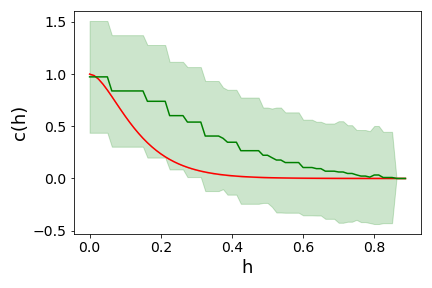}
        \label{fig:matern_1.5_cov_est}}
    \subfigure[Matern ($\nu = 5/2$)]{\includegraphics[width=60mm]{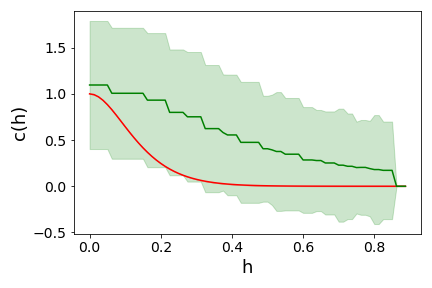}
        \label{fig:matern_2.5_cov_est}}
    \caption{Estimation of the cubic (top left), the spherical (top right), the exponential (center), the Matern with $\nu = 3/2$ (bottom left) and $\nu = 5/2$ (bottom right) covariance functions, on a dyadic grid at scale $J = 3$ ($n = 81$), with $\theta = 5$.}
    \label{fig:add_cov_est}
\end{figure}

Besides the covariance functions depicted in Section~\ref{sec:num}, the following covariance functions are considered (where $\theta \in \mathbb{R}_+^*$ is the correlation length):
\begin{itemize}
    \item the \textit{cubic} covariance function:
        \begin{equation}\label{eq:cubic}
            c : h \in [0,+\infty) \mapsto \left(1 - \left(7 \frac{h^2}{\theta^2} - \frac{35}{4}\frac{h^3}{\theta^3} + \frac{7}{2}\frac{h^5}{\theta^5} - \frac{3}{4}\frac{h^7}{\theta^7}\right) \right)\,\mathbb{I}\left\{h \leq \theta\right\}.
        \end{equation}
    \item the \textit{spherical} covariance function:
        \begin{equation}\label{eq:spher}
            c : h \in [0,+\infty) \mapsto \left(1 - \left(\frac{3}{2}\frac{h}{\theta} - \frac{1}{2}\frac{h^3}{\theta^3}\right) \right)\,\mathbb{I}\left\{h \leq \theta\right\}.
        \end{equation}
    \item the \textit{exponential} covariance function:
        \begin{equation}\label{eq:exp}
            c : h \in [0,+\infty) \mapsto \exp\left(-h/\theta\right).
        \end{equation}
    \item the \textit{Matern} covariance function with smoothness parameter $\nu$:
        \begin{equation}\label{eq:matern}
            c : h \in [0,+\infty) \mapsto \frac{2^{1-\nu}}{\Gamma(\nu)} \left(\sqrt{2 \nu} \frac{h}{\theta}\right)^{\nu} K_{\nu}\left(\sqrt{2 \nu} \frac{h}{\theta}\right),
        \end{equation}
        where $\Gamma$ is the gamma function and $K_{\nu}$ is the modified Bessel function of the second kind.
\end{itemize}

\begin{figure}[H]
    \centering
    \subfigure[$\theta = 2.5$]{\includegraphics[width=60mm, trim=0cm 0cm 1.6cm 0cm]{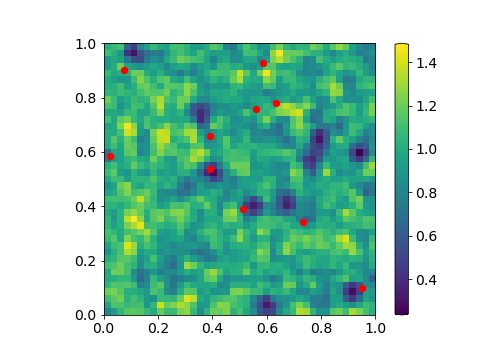}}
    \subfigure[$\theta = 5$]{\includegraphics[width=60mm, trim=0cm 0cm 1.6cm 0cm]{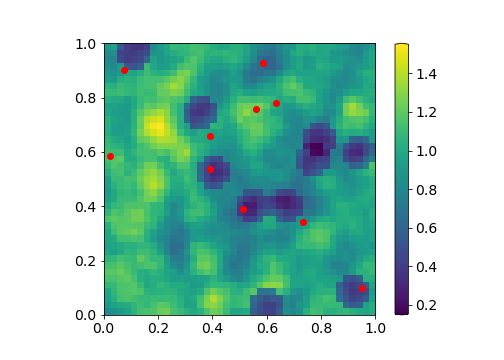}}
    \\
    \subfigure[$\theta = 7.5$]{\includegraphics[width=60mm, trim=0cm 0cm 1.6cm 0cm]{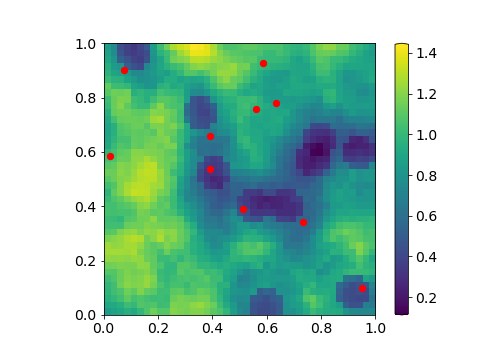}}
    \subfigure[$\theta = 10$]{\includegraphics[width=60mm, trim=0cm 0cm 1.6cm 0cm]{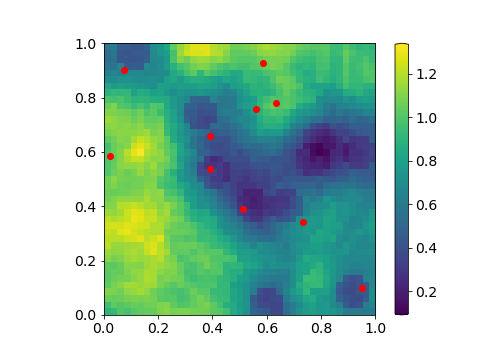}}
    \caption{MSE maps over $100$ realizations of a Gaussian process with cubic covariance function for the empirical Kriging predictor with different values of $\theta$ ($J = 3$ and $d = 10$).}
    \label{fig:Cubic_J3_Emp}
\end{figure}

\begin{remark}{\sc (Matern model)}
    When $\nu = p + \frac{1}{2}$ where $p$ is an integer, the Matern covariance function is a product of an exponential function and a polynomial function of order $p$.
\end{remark}

\begin{figure}[h]
    \centering
    \subfigure[$\theta = 2.5$]{\includegraphics[width=60mm, trim=0cm 0cm 1.6cm 0cm]{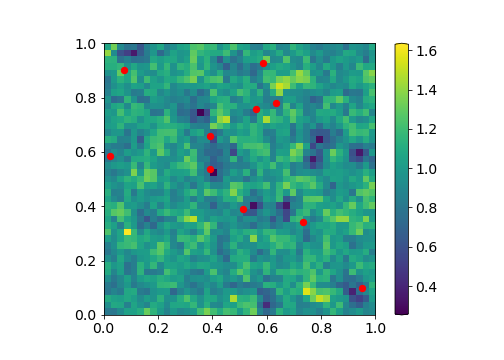}}
    \subfigure[$\theta = 5$]{\includegraphics[width=60mm, trim=0cm 0cm 1.6cm 0cm]{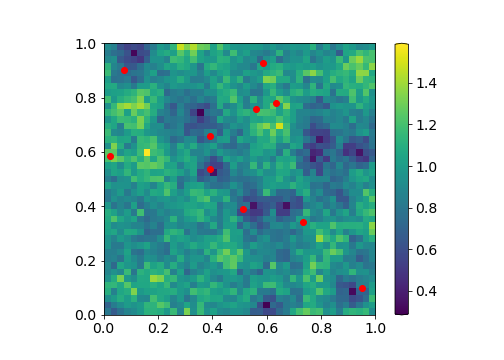}}
    \\
    \subfigure[$\theta = 7.5$]{\includegraphics[width=60mm, trim=0cm 0cm 1.6cm 0cm]{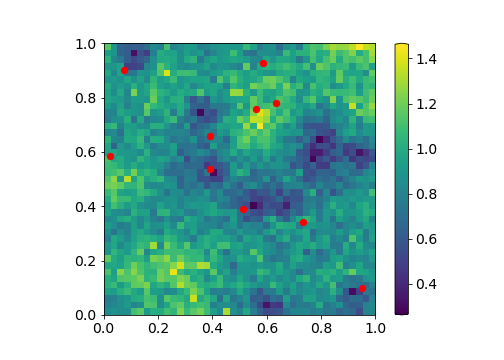}}
    \subfigure[$\theta = 10$]{\includegraphics[width=60mm, trim=0cm 0cm 1.6cm 0cm]{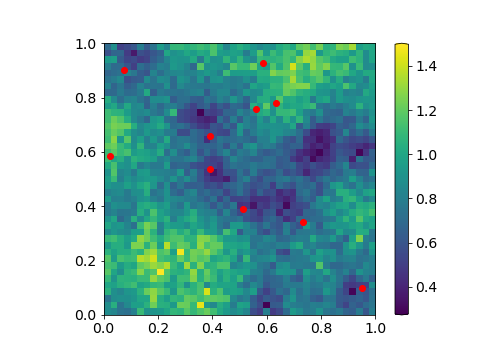}}
    \caption{MSE maps over $100$ realizations of a Gaussian process with spherical covariance function for the empirical Kriging predictor with different values of $\theta$ ($J = 3$ and $d = 10$).}
    \label{fig:Spherical_J3_Emp}
\end{figure}

Following this remark, we select the two following smoothness parameter values for the Matern model: when $\nu = \frac{3}{2}$, the covariance function becomes $c(h) = \left(1 + \sqrt{3}\frac{h}{\theta}\right) \exp\left(- \sqrt{3}\frac{h}{\theta}\right), \, \forall h \in [0,+\infty)$, and for $\nu = \frac{5}{2}$, $\forall h \in [0,+\infty), \, c(h) = \left(1 + \sqrt{5}\frac{h}{\theta} + \frac{5}{3} \left(\frac{h}{\theta}\right)^2\right) \exp\left(- \sqrt{5}\frac{h}{\theta}\right)$.

Note that the cubic and spherical covariance models satisfy all the assumptions involved in Theorem~\ref{thm:EGRisk}, whereas the exponential and Matern covariance functions do not verify Assumption~\ref{hyp:decorr}. We apply the same procedure as in Section~\ref{sec:num} for the five additional covariance models, with the same setting: the training dataset is composed of observations sampled on a dyadic grid at scale $J = 3$ ($n = 81$) and $d = 10$ observations for the prediction step. Firstly, as an illustration of the covariance estimation, Figure \ref{fig:add_cov_est} shows, for each model, the true covariance function in red and the mean of the estimated covariance function in green (with the corresponding mean standard deviation), on $100$ independent simulations. Observe that, for both the cubic (top left) and the spherical (top right) covariance functions, that satisfy Assumption~\ref{hyp:decorr}, the estimation is accurate and it successfully estimates the threshold after which the covariance is equal to zero. For the covariance functions that do not satisfy Assumption~\ref{hyp:decorr}, the exponential function is quite accurate but does not detect the correlation length, whereas for the Matern function, when $\nu = 5/2$, the estimation is less accurate than for the smaller value of the smoothness parameter. In particular, for these three covariance models, when the true covariance function tends to zero, the estimation is still considerably different from zero for an important number of lags and the mean standard deviation is large.

\begin{figure}[h]
    \centering
    \subfigure[$\theta = 2.5$]{\includegraphics[width=60mm, trim=0cm 0cm 1.6cm 0cm]{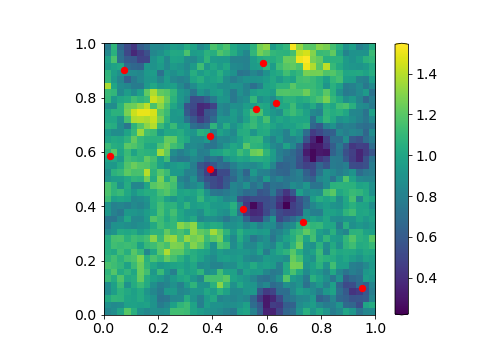}}
    \subfigure[$\theta = 5$]{\includegraphics[width=60mm, trim=0cm 0cm 1.6cm 0cm]{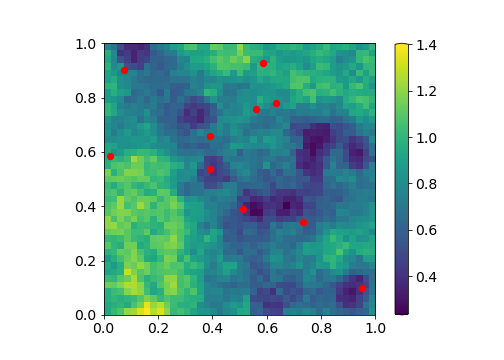}}
    \\
    \subfigure[$\theta = 7.5$]{\includegraphics[width=60mm, trim=0cm 0cm 1.6cm 0cm]{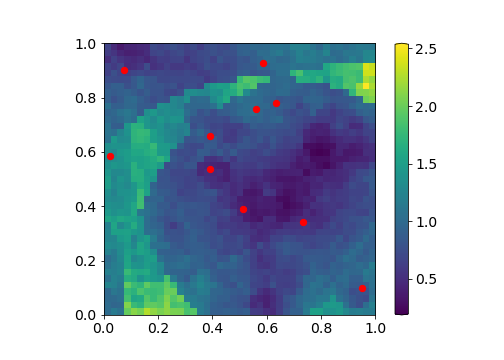}}
    \subfigure[$\theta = 10$]{\includegraphics[width=60mm, trim=0cm 0cm 1.6cm 0cm]{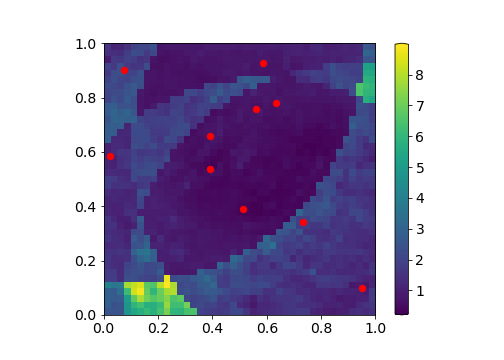}}
    \caption{MSE maps over $100$ realizations of a Gaussian process with exponential covariance function for the empirical Kriging predictor with different values of $\theta$ ($J = 3$ and $d = 10$).}
    \label{fig:Exponential_J3_Emp}
\end{figure}

As for the previous covariance models, the maps of all mean squared errors for the empirical Kriging predictor were computed over $100$ realizations of a Gaussian process, and the results are depicted in Figures~\ref{fig:Cubic_J3_Emp} to \ref{fig:Matern_2.5_J3_Emp}, with varying values for the correlation length $\theta \in \{2.5, 5, 7.5, 10\}$. For the cubic covariance model, it can be noticed in Figure~\ref{fig:Cubic_J3_Emp} that the MSE map seems to become smoother when $\theta$ increases. Still, as in the case of the spherical covariance model (see Figure~\ref{fig:Spherical_J3_Emp}), the complete maps are similar for all values of $\theta$ regarding the error scale and the allocation of the local area with small errors. For the three covariance functions that do not satisfy Assumption~\ref{hyp:decorr}, the same observations can be made: as $\theta$ grows, some border effects can be seen on the boundaries of the window of observation, and the error scale becomes larger. This lines up with the results obtained for the Gaussian covariance model.

\begin{figure}[H]
    \centering
    \subfigure[$\theta = 2.5$]{\includegraphics[width=60mm, trim=0cm 0cm 1.6cm 0cm]{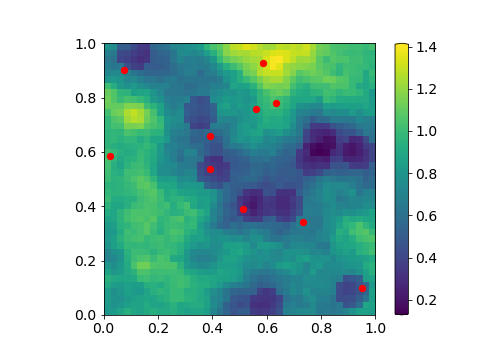}}
    \subfigure[$\theta = 5$]{\includegraphics[width=60mm, trim=0cm 0cm 1.6cm 0cm]{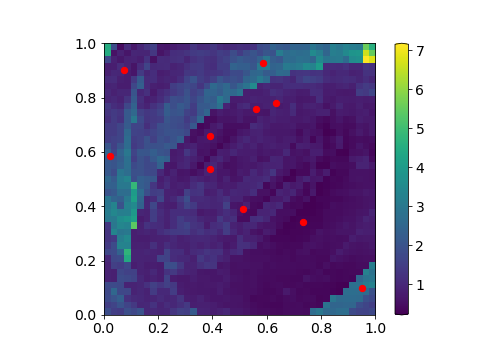}}
    \\
   \subfigure[$\theta = 7.5$]{\includegraphics[width=60mm, trim=0cm 0cm 1.6cm 0cm]{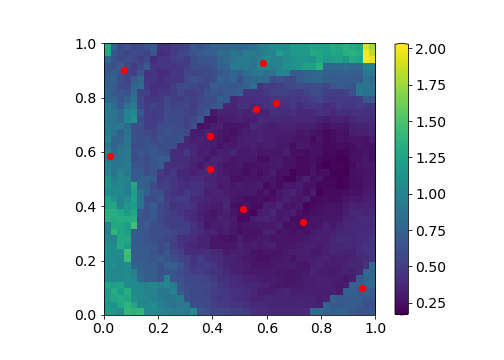}}
   \subfigure[$\theta = 10$]{\includegraphics[width=60mm, trim=0cm 0cm 1.6cm 0cm]{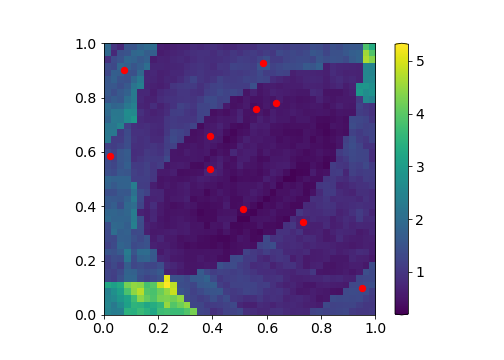}}
    \caption{MSE maps over $100$ realizations of a Gaussian process with Matern covariance ($\nu = 3/2$) for the empirical Kriging predictor with different values of $\theta$ ($J = 3$ and $d = 10$).}
    \label{fig:Matern_1.5_J3_Emp}
\end{figure}

\begin{figure}[H]
    \centering
   \subfigure[$\theta = 2.5$]{\includegraphics[width=60mm, trim=0cm 0cm 1.6cm 0cm]{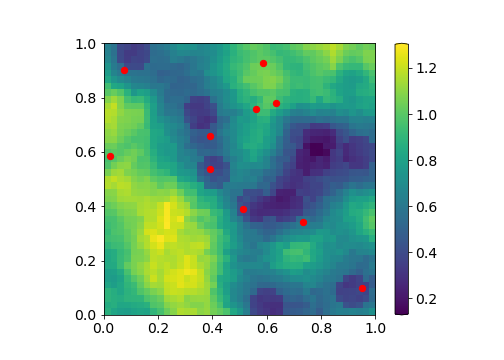}}
    \subfigure[$\theta = 5$]{\includegraphics[width=60mm, trim=0cm 0cm 1.6cm 0cm]{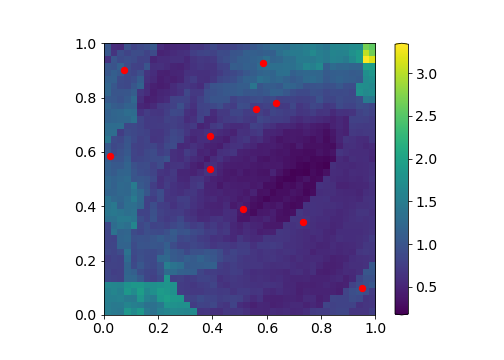}}
    \\
   \subfigure[$\theta = 7.5$]{\includegraphics[width=60mm, trim=0cm 0cm 1.6cm 0cm]{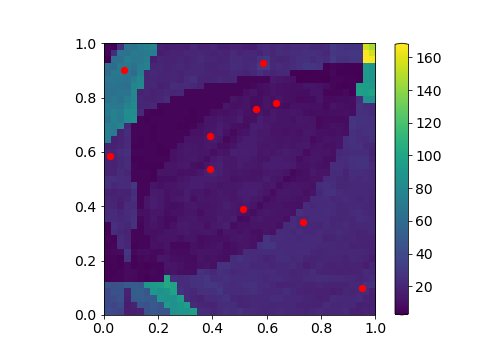}}
   \subfigure[$\theta = 10$]{\includegraphics[width=60mm, trim=0cm 0cm 1.6cm 0cm]{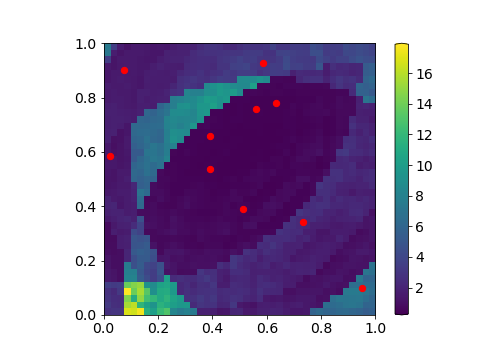}}
    \caption{MSE maps over $100$ realizations of a Gaussian process with Matern covariance ($\nu = 5/2$) for the empirical Kriging predictor with different values of $\theta$ ($J = 3$ and $d = 10$).}
    \label{fig:Matern_2.5_J3_Emp}
\end{figure}

Finally, the Average Mean Squared Error for both theoretical and empirical predictive mappings for all the covariance functions was calculated. The results (the mean and the standard deviation for the $100$ independent simulations for different values of $\theta$) are presented in Table \ref{tab:add_all_J3}. For the cubic and spherical models, we observe the same trend as for the truncated power law model: as the correlation length $\theta$ increases, for both the theoretical and empirical Kriging, the mean of the AMSE decreases while the standard deviation increases slowly, and the excess risk is small for all values of $\theta$. For the exponential covariance model, the mean and the standard deviation do not have a constant trend: first, the mean decreases while the standard deviation increases slightly, then they both increase more significantly, with a large standard deviation when $\theta = 10$. A significant increase of both the mean and standard deviation of the AMSE can also be observed for the Matern model, especially when $\nu = 5/2$: the mean and the standard deviation become large when $\theta = 7.5$ (see Table~\ref{tab:add_all_J3}).

\begin{table}[h]
  \centering
  \caption{Mean and standard deviation of the AMSE over $100$ independent simulations of a Gaussian process with cubic (top left), spherical (top right), exponential (center), Matern with $\nu = 3/2$ (bottom left) and $\nu = 5/2$ (bottom right) covariance functions for theoretical and empirical Kriging with different values of $\theta$ (with $J = 3$, $N = 1681$ and $d = 10$).}
  \vskip 0.15in
    \centering
    \begin{tabular}{ | c || c | c || c | c | }
        \hline
        \small{\texttt{CUB}}
        &
        \multicolumn{2}{| c ||}{ \small{\texttt{Theoretical}} } & \multicolumn{2}{| c |}{ \small{\texttt{Empirical}} }
        \\ \hline \hline
        $\theta$ & mean & std & mean & std
        \\ \hline
        $2.5$ & 0.979 & 0.091 & 0.983 & 0.088
        \\ \hline
        $5$ & 0.920 & 0.147 & 0.934 & 0.150
        \\ \hline
        $7.5$ & 0.847 & 0.232 & 0.869 & 0.241
        \\ \hline
        $10$ & 0.748 & 0.294 & 0.778 & 0.294
        \\ \hline
    \end{tabular}
    \hspace{0.3cm}
    \centering
    \begin{tabular}{ | c || c | c || c | c | }
        \hline
        \small{\texttt{SPH}}
        &
        \multicolumn{2}{| c ||}{ \small{\texttt{Theoretical}} } & \multicolumn{2}{| c |}{ \small{\texttt{Empirical}} }
        \\ \hline \hline
        $\theta$ & mean & std & mean & std
        \\ \hline
        $2.5$ & 0.984 & 0.079 & 0.997 & 0.081
        \\ \hline
        $5$ & 0.927 & 0.121 & 0.941 & 0.122
        \\ \hline
        $7.5$ & 0.870 & 0.188 & 0.890 & 0.195
        \\ \hline
        $10$ & 0.809 & 0.183 & 0.846 & 0.199
        \\ \hline
    \end{tabular}
        
    \vspace{0.3cm}
    
    \centering
    \begin{tabular}{ | c || c | c || c | c | }
        \hline
        \small{\texttt{EXP}}
        &
        \multicolumn{2}{| c ||}{ \small{\texttt{Theoretical}} } & \multicolumn{2}{| c |}{ \small{\texttt{Empirical}} }
        \\ \hline \hline
        $\theta$ & mean & std & mean & std
        \\ \hline
        $2.5$ & 0.890 & 0.178 & 0.923 & 0.184
        \\ \hline
        $5$ & 0.679 & 0.180 & 0.781 & 0.230
        \\ \hline
        $7.5$ & 0.548 & 0.158 & 0.922 & 1.976
        \\ \hline
        $10$ & 0.439 & 0.150 & 1.481 & 5.393
        \\ \hline
    \end{tabular}
  \\
  
  \vspace{0.3cm}
  
    \centering
    \begin{tabular}{ | c || c | c || c | c | }
        \hline
        \small{\texttt{$3/2$}}
        &
        \multicolumn{2}{| c ||}{ \small{\texttt{Theoretical}} } & \multicolumn{2}{| c |}{ \small{\texttt{Empirical}} }
        \\ \hline \hline
        $\theta$ & mean & std & mean & std
        \\ \hline
        $2.5$ & 0.791 & 0.308 & 0.769 & 0.281
        \\ \hline
        $5$ & 0.464 & 0.266 & 1.167 & 5.841
        \\ \hline
        $7.5$ & 0.260 & 0.174 & 0.548 & 0.793
        \\ \hline
        $10$ & 0.163 & 0.130 & 1.017 & 3.105
        \\ \hline
    \end{tabular}
    \hspace{0.3cm}
    \centering
    \begin{tabular}{ | c || c | c || c | c | }
        \hline
        \small{\texttt{$5/2$}}
        &
        \multicolumn{2}{| c ||}{ \small{\texttt{Theoretical}} } & \multicolumn{2}{| c |}{ \small{\texttt{Empirical}} }
        \\ \hline \hline
        $\theta$ & mean & std & mean & std
        \\ \hline
        $2.5$ & 0.791 & 0.269 & 0.756 & 0.249
        \\ \hline
        $5$ & 0.471 & 0.277 & 0.800 & 1.842
        \\ \hline
        $7.5$ & 0.236 & 0.191 & 20.945 & 179.206
        \\ \hline
        $10$ & 0.129 & 0.106 & 2.455 & 18.476
        \\ \hline
    \end{tabular}
  \label{tab:add_all_J3}
    \vskip 0.1in
\end{table}

The results on additional covariance functions lead to the same conclusions about the truncated power law and the Gaussian covariance functions: the predictive method may perform well, even if Assumption~\ref{hyp:decorr} is slightly violated.

\subsection{Anisotropic Covariance Function}

In this section, we study the role of the isotropy assumption for the covariance function (Assumption~\ref{hyp:stat_iso}). Based on the truncated power law and the Gaussian covariance models, we apply the same procedure as in Section~\ref{sec:num} to the case where anisotropic covariance functions are selected. The \texttt{gstools} library allows one to simulate an anisotropic covariance function with varying anisotropic ratios $\alpha$, where $\alpha = 1$ corresponds to the isotropy situation.

\begin{table}[h]
  \centering
  \caption{Mean and standard deviation of the AMSE over $100$ independent simulations of a Gaussian process with truncated power law (left) and Gaussian (right) covariance functions for theoretical and empirical Kriging with different values of the anisotropic ratio $\alpha$ (with $J = 3$, $N = 1681$, $d = 10$ and $\theta = 5$).}
  \vskip 0.15in
    \centering
    \begin{tabular}{ | c || c | c || c | c | }
        \hline
        \small{\texttt{TPL}}
        &
        \multicolumn{2}{| c ||}{ \small{\texttt{Theoretical}} } & \multicolumn{2}{| c |}{ \small{\texttt{Empirical}} }
        \\ \hline \hline
        $\alpha$ & mean & std & mean & std
        \\ \hline
        $0.9$ & 0.901 & 0.141 & 0.934 & 0.142
        \\ \hline
        $0.8$ & 0.925 & 0.139 & 0.951 & 0.135
        \\ \hline
        $0.7$ & 0.919 & 0.119 & 0.949 & 0.121
        \\ \hline
    \end{tabular}
  % \\
  % \vspace{0.3cm}
  \hspace{0.18cm}
    \centering
    \begin{tabular}{ | c || c | c || c | c | }
        \hline
        \small{\texttt{GAUSS}}
        &
        \multicolumn{2}{| c ||}{ \small{\texttt{Theoretical}} } & \multicolumn{2}{| c |}{ \small{\texttt{Empirical}} }
        \\ \hline \hline
        $\alpha$ & mean & std & mean & std
        \\ \hline
        $0.9$ & 0.428 & 0.173 & 0.722 & 0.354
        \\ \hline
        $0.8$ & 0.486 & 0.186 & 0.750 & 0.315
        \\ \hline
        $0.7$ & 0.579 & 0.216 & 0.784 & 0.294
        \\ \hline
    \end{tabular}
    \label{tab:J3_anis}
\end{table}

The same setting is used in order to compare the results with those in Section~\ref{sec:num}: the covariance estimation is done thanks to a training dataset, observed at $n = 81$ sites ($J = 3$) and the prediction over the whole spatial domain is computed based on $d = 10$ observations. We fix the value of the correlation length at $\theta = 5$, and repeat the experiments for different degrees of anisotropy $\alpha \in \{0.9, 0.8, 0.7\}$.

\begin{figure}[h]
    \centering
    \subfigure[$\alpha = 0.9$]{\includegraphics[width=50mm]{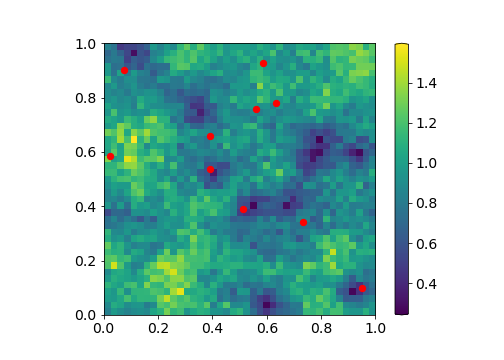}}
    \subfigure[$\alpha = 0.8$]{\includegraphics[width=50mm]{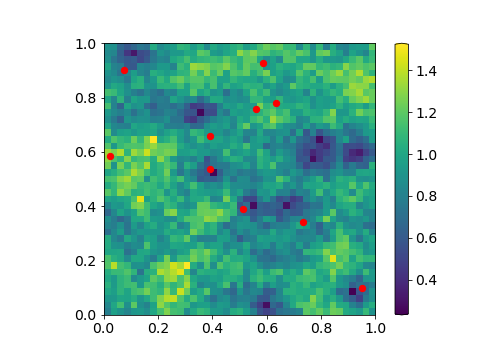}}
    \subfigure[$\alpha = 0.7$]{\includegraphics[width=50mm]{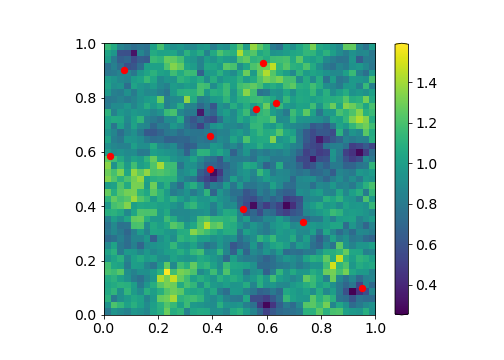}}
    \caption{MSE maps of over $100$ realizations of a Gaussian process with truncated power law covariance function for the empirical Kriging predictor with different values of the anisotropic ratio $\alpha$ ($J = 3$, $N = 1681$, $d = 10$ and $\theta = 5$).}
    \label{fig:TPL_J3_anis_Emp}
\end{figure}

The mean and standard deviation of the AMSE computed over the $100$ independent simulations of a Gaussian process, for both covariance models, are shown in Table~\ref{tab:J3_anis}. We observe that for both models, for the empirical Kriging prediction, the mean increases slightly when $\alpha$ decreases (so when the covariance function becomes more anisotropic), while the standard deviation decreases (see Table~\ref{tab:TPL_Gauss_J3_theta}, when $\theta = 5$ for a comparison with the isotropic case).

Figures~\ref{fig:TPL_J3_anis_Emp} and \ref{fig:Gauss_J3_anis_Emp} (for the truncated power law and the Gaussian model, respectively), show the complete MSE maps over $100$ realizations. It can be observed that the structure of the errors for both covariance models is similar to the maps obtained using an isotropic covariance function, with the same scale error and the same local area with small errors (see Figures~\ref{fig:TPL_J3_Emp_theta_5.0} and \ref{fig:Gauss_J3_Emp_theta_5.0}, respectively).

\begin{figure}[h]
    \centering
    \subfigure[$\alpha = 0.9$]{\includegraphics[width=50mm]{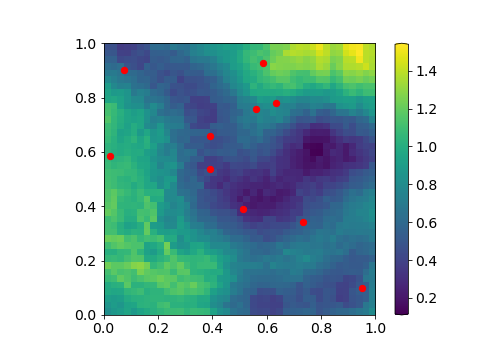}}
    \subfigure[$\alpha = 0.8$]{\includegraphics[width=50mm]{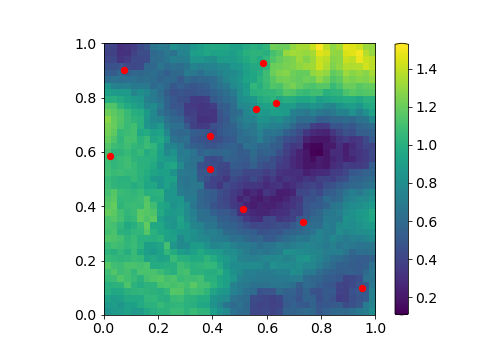}}
    \subfigure[$\alpha = 0.7$]{\includegraphics[width=50mm]{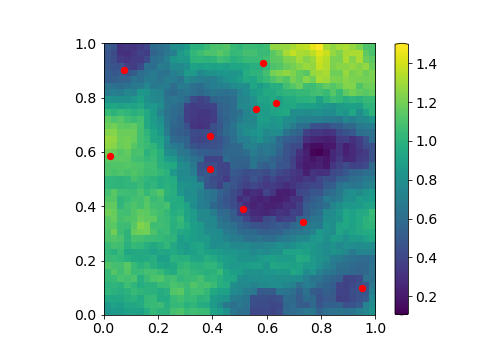}}
    \caption{MSE maps of over $100$ realizations of a Gaussian process with Gaussian covariance function for the empirical Kriging predictor  with different values of the anisotropic ratio $\alpha$ ($J = 3$, $N = 1681$, $d = 10$ and $\theta = 5$).}
    \label{fig:Gauss_J3_anis_Emp}
\end{figure}

These results, which show that the prediction methodology is robust with respect to slight departures from isotropy, encourage us to relax also the isotropic assumption, in some future work.

\subsection{Irregular Grid}\label{subsec:aux_irr_grid}

With the purpose of extending the theoretical results to a more general framework, we present the numerical results within a different setting: we consider the case where the simulations are done over an irregular grid. In this new setting, the realization of the random field used for the nonparametric covariance estimation (the training spatial dataset) is no longer observed at $n$ sites forming a dyadic grid. Instead of assuming that the training observations are made on a regular dyadic grid, we make the hypothesis that we have access only to a restricted number of these observations. The irregular grids are generated from regular grids using Bernoulli sampling, with varying probability $p$ ($\{0.8, 0.6, 0.4\}$) of observing a spatial site. The number of observed locations is $\{65, 50, 35\}$ (respectively).

\begin{table}[h]
  \centering
  \caption{Mean and standard deviation of the AMSE over $100$ independent simulations of a Gaussian process with truncated power law (left) and Gaussian (right) covariance functions for empirical Kriging with different probabilities $p$ for the Bernoulli sampling (with $N = 1681$, $d = 10$ and $\theta = 5$).}
  \vskip 0.15in
    \centering
    \begin{tabular}{ | c || c | c | }
        \hline
        \small{\texttt{TPL}} & \multicolumn{2}{| c |}{ \small{\texttt{Empirical}} }
        \\ \hline \hline
        $p$ & mean & std
        \\ \hline
        $0.8$ & 0.932 & 0.159
        \\ \hline
        $0.6$ & 0.934 & 0.156
        \\ \hline
        $0.4$ & 0.936 & 0.154
        \\ \hline
    \end{tabular}
  \hspace{0.18cm}
    \begin{tabular}{ | c || c | c | }
        \hline
        \small{\texttt{GAUSS}} & \multicolumn{2}{| c |}{ \small{\texttt{Empirical}} }
        \\ \hline \hline
        $p$ & mean & std
        \\ \hline
        $0.8$ & 0.686 & 0.320
        \\ \hline
        $0.6$ & 0.663 & 0.267
        \\ \hline
        $0.4$ & 0.687 & 0.370
        \\ \hline
    \end{tabular}
  \label{tab:all_irr_grid}
\end{table}

The estimation of the covariance function is computed as before, using Equation~\eqref{eq:emp_cov}. Since this framework could imply (likely but not surely) situations where for some $h$ previously observed on the complete dyadic grid, are not present anymore, we handle these cases by skipping the estimation of the covariance function for these values and simply applying the 1-NN estimator in the prediction step (as stated in subsection~\ref{subsec:nonpar_cov_est}). For the simple Kriging prediction, we use the same independent realization of $X$ observed at $d = 10$ sites. We present the results only for the truncated power law and the Gaussian covariance models, for $\theta = 5$: the mean and the standard deviation of the AMSE are displayed in Table~\ref{tab:all_irr_grid}. Note that, the sampled locations over the irregular grid are fixed for the $100$ replications of the experiment, as well as the $d$ locations for the prediction test. For both covariance models, we observe that, as fewer and fewer observations are selected, the AMSE and the standard deviation do not vary much from the regular grid situation. This shows that both models are robust against irregular sampling for the covariance estimation.

\begin{figure}[h]
    \centering
    \subfigure[$p = 0.8$]{\includegraphics[width=50mm]{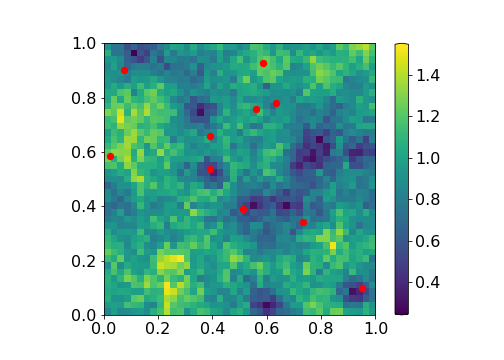}}
    \subfigure[$p = 0.6$]{\includegraphics[width=50mm]{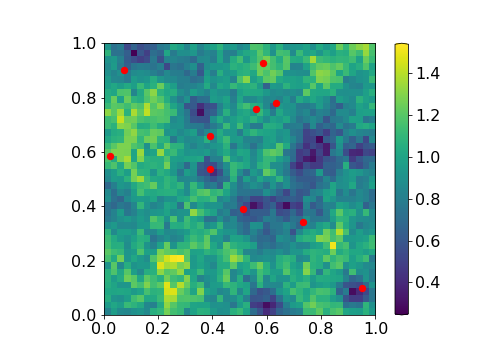}}
   \subfigure[$p = 0.4$]{\includegraphics[width=50mm]{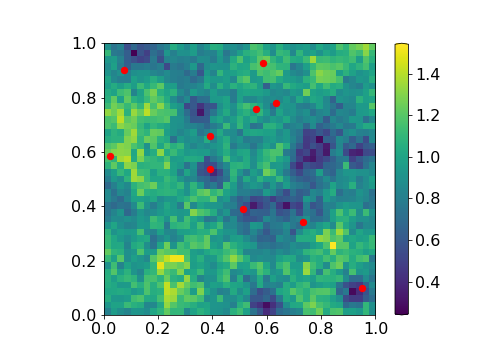}}
    \caption{MSE maps of over $100$ realizations of a Gaussian process with truncated power law covariance function for the empirical Kriging predictor with different probabilities $p$ for the Bernoulli sampling ($N = 1681$, $d = 10$ and $\theta = 5$).}
    \label{fig:TPL_J3_irr_grid_Emp}
\end{figure}

\begin{figure}[h]
    \centering
    \subfigure[$p = 0.8$]{\includegraphics[width=50mm]{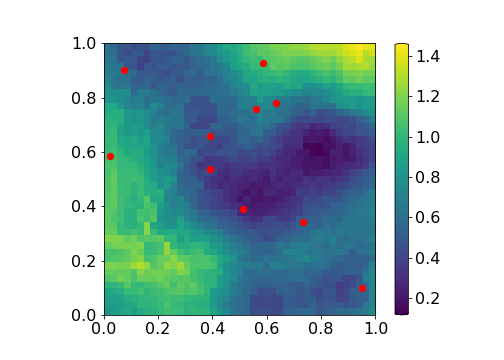}}
    \subfigure[$p = 0.6$]{\includegraphics[width=50mm]{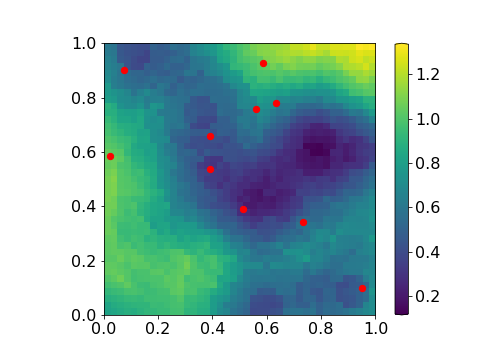}}
    \subfigure[$p = 0.4$]{\includegraphics[width=50mm]{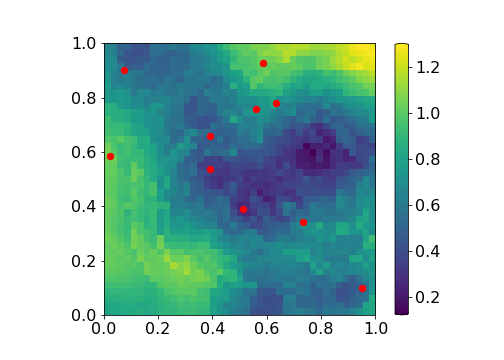}}
    \caption{MSE maps of over $100$ realizations of a Gaussian process with Gaussian covariance function for the empirical Kriging predictor with different probabilities $p$ for the Bernoulli sampling ($N = 1681$, $d = 10$ and $\theta = 5$).}
    \label{fig:Gauss_J3_irr_grid_Emp}
\end{figure}

Regarding the structure of these errors for the truncated power law covariance model, as it can be seen in the maps of the mean squared errors in Figure~\ref{fig:TPL_J3_irr_grid_Emp}, for all the values of $p$, the results are quite similar to the case where the observations are taken on a regular dyadic grid (see Figure~\ref{fig:TPL_J3_Emp_theta_5.0}). For the Gaussian model, for $p = 0.8$ and $p = 0.6$, we recognize the same errors' structure as in Figure~\ref{fig:Gauss_J3_Emp_theta_5.0} with some border effects. Still, when $p$ decreases, the errors seem to be expanded over the spatial domain.

As all previous results, the ones on the irregular sampling setup may be a motivation to extend the theoretical study to a more general framework, including new grid of observations.

\section{Additional Comments}
In this section, we first mention alternative statistical frameworks for the Kriging problem. Next, we underline the role of the technical assumptions involved in the analysis proposed in subsection \ref{subsec:nonpar_cov_est} and discuss possible avenues to relax some of these hypotheses in order to extend the main results of this paper to a more general framework.

\subsection{Alternative Frameworks}

As mentioned in the Introduction and in subsection \ref{subsec:kriging} (see Remark~\ref{rem:alt_framework}), other statistical frameworks for Kriging have already been studied. For example, existing studies have proposed to place oneself in the 'increasing domain' asymptotic (also referred to as the \textit{out-fill} setting), where the spatial domain $S$ under study becomes wider and wider as the number of observations $n$ grows and a minimum distance between neighbouring sampling locations is assumed, see \textit{e.g.} \citep{Mardia, Sherman}. Other researchers \citep{Hall, lahiri1999asymptotic, Lahiri_Cressie, putter2001} have also considered a hybrid setting, where a combination of in-fill and increasing domain asymptotics point of view is taken, often assuming that both the size of the spatial domain of the observations and the number of observations in each of its subsets grow with $n$.

Recall that, in \textit{simple Kriging}, one assumes that the mean of the spatial process is known. However, in most practical situations, the mean is unknown. \textit{Ordinary Kriging} is an interpolation method that does not require any knowledge of the mean and is suitable in these situations (see \textit{e.g.} \citep[Section 3]{Chiles.etal1999}). Furthermore, instead of the strong second-order stationarity assumption that is classically made, one can relax it to a weak stationarity assumption in such a way that only the variance is assumed to be constant over the spatial domain, while the mean can differ in a deterministic way and so present a spatial trend: this other framework is called \textit{Universal Kriging}. Another framework is \textit{Cokriging}, an interpolation method that uses additional observed variables, often correlated with each other and with the variable of interest, to improve the precision of the interpolation.

Note that in \citep{8645258}, non-asymptotic analysis has been carried out but in a much more restrictive statistical framework regarding practical assumptions: concentration analysis bounds have been derived when independent copies of the spatial process are assumed to be observed.

For asymptotic results, we refer interested readers to \citep[Section 3]{stein1999interpolation} for an overview.

Given the massive spatial datasets that are available, when seeking for flexible methods, we opt for a nonparametric approach rather than for a parametric one that requires the selection of a variogram model and the estimation of the unknown parameters (see \textit{e.g.} \citep{Zimmerman, zimmermancressie, Mller1999OptimalDF}). Moreover, one can notice that previous studies carried out nonparametric estimations of the covariance function (see \textit{e.g.} \citep{hall1994nonparametric, RePEc:spr:sistpr:v:11:y:2008:i:2:p:177-205, vershynin2010close, loukas17a}).

Regarding these last points, even if nonparametric approaches have already been documented, the non-asymptotic perspective embraced in the finite-sample analyses we have carried out in this paper, generates difficulties that had not been addressed before.

\subsection{Future Lines of Research}

The goal of this paper is to explain the main ideas to achieve generalization guarantees in the spatial context: to this purpose, we use several simplifying technical assumptions and discuss here possible extensions that will be studied in a future work.

As highlighted, first in the proofs of the main results and later in the Numerical Experiments (see Appendix~\ref{sec:add_num} and Section~\ref{sec:num}), even in the simplest framework, the analysis is far from straightforward. In the sketch of proof (and in the more detailed proofs in Appendix~\ref{sec:technical_proofs}), the importance of certain hypotheses is underlined by showing how some hypotheses are necessary for completing the steps of the proofs. Now our objective is to find which assumptions can be relaxed and with which consequences on the learning bounds.

For instance, Assumption~\ref{hyp:decorr} on the covariance function may seem restrictive but greatly simplifies the argument, making it more understandable. Indeed, relaxing it would require a much more technical analysis involving the decay rate of the covariance function. Furthermore, as highlighted in the numerical experiments, even if the Gaussian kernel fails to satisfy Assumption~\ref{hyp:decorr}, the empirical results encourage us to generalize our theoretical analysis to a more general framework. This remark can also be done for the two additional covariance models that do not satisfy Assumption~\ref{hyp:decorr}, and even tend to zero less quickly than the Gaussian model, namely the exponential and the Matern (when $\nu = 3/2$) models.

Another possible extension is to replace Assumption~\ref{hyp:smooth} by alternative smoothness hypotheses (\textit{e.g.} the mapping $h\in [0,\; 1-2^{-j_1}]\mapsto c(h)$ can be assumed of class $\mathcal{C}^{2}$), inducing a possibly different bias term in the excess risk bound.

Finally, other types of observation grids can be considered (like irregular ones), implying technical difficulties, for example when controlling the spectrum of the covariance matrix (see the proof of Proposition~\ref{prop:CI_var_lag}). This may also result in an estimation bias for the semi-variogram or may lead to defining different sets $N_{\varepsilon}(h)$ of neighbours like the set of pairs that are at distance more or less $h$ (with error $\varepsilon > 0$). For a first look at the results that one can obtain within this alternative framework, please refer to subsection~\ref{subsec:aux_irr_grid}, where irregular locations where considered under preferential sampling.

\end{document}